\documentclass[11pt]{article}
\usepackage[margin=1in]{geometry}
\usepackage{booktabs}
\usepackage{array}
\usepackage{microtype}
\usepackage{color}

\usepackage[sortcites,sorting=none,citestyle=numeric,natbib=true,backend=bibtex]{biblatex}
\let\citet\cite
\usepackage[hidelinks]{hyperref}

\title{Conformalized Adaptive Forecasting of Heterogeneous Trajectories}

\usepackage{subfigure}
\usepackage{placeins}
\usepackage{graphicx, color, url}
\usepackage{dsfont}
\usepackage{amssymb}
\usepackage{amsfonts}
\usepackage{url}
\usepackage{amsbsy}
\usepackage{setspace}
\usepackage{bm}
\usepackage{soul}
\usepackage{textcomp}
\usepackage{booktabs}
\usepackage{rotating}
\usepackage{multirow}
\usepackage{multicol}
\usepackage{makecell,booktabs}

\usepackage{algorithm, algorithmic}

\usepackage{amsmath,amsthm}

\usepackage{graphicx}
\usepackage{amsbsy}
\usepackage{mathtools}
\usepackage{bbm}

\newtheorem{theorem}{Theorem}

\makeatletter
\newcommand*\bigcdot{\mathpalette\bigcdot@{.5}}
\newcommand*\bigcdot@[2]{\mathbin{\vcenter{\hbox{\scalebox{#2}{$\m@th#1\bullet$}}}}}
\makeatother

\DeclareMathOperator*{\argmin}{arg\,min}

\usepackage[colorinlistoftodos]{todonotes}

\author{Yanfei Zhou\thanks{Department of Data Sciences and Operations, University of Southern California, Los Angeles, CA, USA.}, Lars Lindemann\thanks{Department of Computer Science, University of Southern California, Los Angeles, CA, USA.},
Matteo Sesia\footnotemark[1]}

\begin{document}

\maketitle

\begin{abstract}
This paper presents a new conformal method for generating {\em simultaneous} forecasting bands guaranteed to cover the {\em entire path} of a new random trajectory with sufficiently high probability. Prompted by the need for dependable uncertainty estimates in motion planning applications where the behavior of diverse objects may be more or less unpredictable, we blend different techniques from online conformal prediction of single and multiple time series, as well as ideas for addressing heteroscedasticity in regression. This solution is both principled, providing precise finite-sample guarantees, and effective, often leading to more informative predictions than prior methods. \end{abstract}

\section{Introduction}\label{introduction}

Time series forecasting is a crucial problem with numerous applications in science and engineering.
Many machine learning algorithms, including deep neural networks, have been developed to address this task, but they are typically designed to produce point predictions and struggle to quantify uncertainty.
This limitation is especially problematic in domains involving intrinsic unpredictability, such as human behavior, and in high-stakes situations like autonomous driving~\citep{lindemann2023safe,lekeufack2023conformal} or wildfire forecasting~\citep{xu2022wildfire, xu2023spatio}.

A popular framework for endowing any model with reliable uncertainty estimates is that of {\em conformal prediction} \citep{vovk2005algorithmic,lei2016RegressionPS}. 
The idea is to observe and quantify the model's predictive performance on a {\em calibration} data set, independent of the training sample.
If those data are sampled from the test population, the calibration performance is representative of the performance at test time. Thus, it becomes possible, with suitable algorithms, to convert any model's point predictions into {\em intervals} (or sets) with guaranteed coverage properties for future observations.

Conformal prediction typically hinges on {\em exchangeability}---an assumption less stringent than the requirement for calibration and test data to be independent and identically distributed.
Under data exchangeability, conformal prediction can provide reliable statistical safeguards for any predictive model.
Its flexibility enables applications across many tasks, including regression \citep{lei2014distribution,romano2019conformalized,sesia2021conformal}, classification \citep{lei2013distribution,sadinle2019least,podkopaev2021distribution}, outlier detection \citep{bates2021testing,marandon2022machine,liang2022integrative}, and time series forecasting \citep{xu2021conformal,stankeviciute2021conformal,xu2023sequential,ajroldi2023conformal}.
This paper focuses on the last topic.

Conformal methods for time series tend to fall into one of two categories: {\em multi-series} and {\em single-series}.
Methods in the former category aim to predict a new trajectory by leveraging other {\em jointly exchangeable} trajectories from the same population \citep{stankeviciute2021conformal,lindemann2023safe,lekeufack2023conformal}. 
In the single-series setting, the aim shifts to forecasting future values based on historical observations from a fixed series, typically avoiding strict exchangeability assumptions \citep{gibbs2021adaptive,gibbs2022conformal, angelopoulos2023conformal}. 
This paper draws inspiration from both areas and addresses a remaining limitation of current methods for multi-series forecasting.

The challenge addressed in this paper is that of {\em data heterogeneity}---distinct time series with different levels of unpredictability.
For instance, in motion planning, forecasting the paths of pedestrians may be complicated by the relatively erratic behavior of some individuals, such as small children or intoxicated adults. 
This variability aligns with the classical issue of {\em heteroscedasticity}.
The latter has recently gained some recognition within the conformal prediction literature, particularly for regression \citep{romano2019conformalized} and classification \citep{romano2020classification,einbinder2022training}. 
In this paper, we address heteroscedasticity within the more complex setting of trajectory forecasting.

\subsection*{Related Work}

The challenge of conformal inference for non-exchangeable data is receiving significant attention, both from more general perspectives \citep{tibshirani2019conformal, barber2022conformal, Qiu2023prediction} and in the context of time-series forecasting.
An important line of research has focused on forecasting a single series, including recent works inspired by \citet{gibbs2021adaptive} such as \citet{gibbs2022conformal, bastani2022practical, Zaffran2022AdaptiveCP, Feldman2022AchievingRC, Dixit2023adaptive, angelopoulos2023conformal,bhatnagar2023improved}.
Further, other approaches that combine conformal prediction with single-series forecasting include those of \citet{Chernozhukov2018exact, xu2021conformal, Xu2022ConformalPS, xu2023sequential, sousa2023general, auer2023conformal, xu2023uncertainty}.
The present paper builds on this extensive body of work, drawing particular inspiration from \citet{gibbs2021adaptive}. However, our approach is distinct in its pursuit of stronger simultaneous coverage guarantees, a goal justified by motion planning applications, for example, but not achievable within the constraints of single-series forecasting.

Conformal prediction in multi-series forecasting has so far received relatively less attention.
\citet{lin2022conformal} explored a somewhat related yet distinct problem. Their work focused on ensuring different types of ``longitudinal" and ``cross-sectional'' coverage, which is a different goal compared to our objective of simultaneously forecasting an {\em entire} new trajectory. 
We conduct direct comparisons between our method and those of \citet{stankeviciute2021conformal} and \citet{yu2023signal,cleaveland2023conformal}. 
These address problems akin to ours but adopt different approaches and do not focus on heteroscedasticity. Specifically, \citet{stankeviciute2021conformal} implemented a Bonferroni correction, which is often very conservative, while \citet{yu2023signal}, \citet{cleaveland2023conformal}, and \citet{sun2023copula} used a technique more aligned with ours but lacking in adaptability to heteroscedastic conditions.

\section{Background and Motivation}

\subsection{Problem Statement and Notation}\label{sec:notations}

We consider a data set comprising $n$ observations of arrays of length $(T+1)$, namely $\mathcal{D}:= \{\bm{Y}^{(1)},\hdots,\bm{Y}^{(n)}\}$. 
For $i \in [n] := \{1,\dots,n\}$, the array $\bm{Y}^{(i)} = (Y_0^{(i)}, Y_1^{(i)}, \dots,Y_T^{(i)})$ represents $T+1$ observations of some $d$-dimensional vector $Y_{t}^{(i)} = (Y_{t,1}^{(i)}, \ldots, Y_{t,d}^{(i)})\in \mathbb{R}^d$, measured at distinct time steps $t \in \{0,\dots, T\}$.
We will assume throughout the paper that the $n$ trajectories are sampled exchangeably from some arbitrary and unknown distribution $P$.
However, it is worth emphasizing that we make no assumptions about the potentially complex time dependence with each series $(Y_{0}^{(i)}$, $Y_{1}^{(i)}$, \ldots,$Y_{T}^{(i)}$).
Intuitively, our goal is to leverage the data in $\mathcal{D}$ to construct an informative {\em prediction band} for the trajectory of a new series $\bm{Y}^{(n+1)}$, which is assumed to be also sampled exchangeably from the same distribution.

For simplicity, we focus on {\em one-step-ahead} forecasting, which means that we want to construct a prediction band for $\bm{Y}^{(n+1)}$ one step at a time. That is, we imagine that the initial position $Y_{0}^{(n+1)}$ is given and then wait to observe $Y_{t-1}^{(n+1)}$ before predicting $Y_{t}^{(n+1)}$, for each $t \in [T]$.
This perspective is often useful, for example in motion planning applications, but it is of course not the only possible one. Fortunately, though, our solution for the one-step-ahead problem can easily be extended to {\em multiple-step-ahead} forecasting, as explained in Appendix~\ref{app:multi_step}, or even {\em one-shot} forecasting of an entire trajectory.

Let $\hat{C}(\bm{Y}^{(n+1)}) := (\hat{C}_1(\bm{Y}^{(n+1)}), \ldots, \hat{C}_T(\bm{Y}^{(n+1)}))$ represent the output prediction band, where each $\hat{C}_t(\bm{Y}^{(n+1)}) \subseteq \mathbb{R}^d$ is a prediction region for the vector $Y_{t}^{(n+1)}$ that may depend on past observations $Y_{s}^{(n+1)}$ for $s < t$, as well as on the data in $\mathcal{D}$.
As we develop a method to construct $\hat{C}(\bm{Y}^{(n+1)})$, one goal is to ensure the following notion of {\em simultaneous marginal coverage}:
\begin{equation}\label{eq:simu_coverage}
    \mathbb{P}\left[ Y_{t}^{(n+1)} \in \hat{C}_{t}(\bm{Y}^{(n+1)}), \; \forall t \in [T] \right] \geq 1-\alpha.
\end{equation}
Simply put, the entire trajectory should lie within the band with probability at least $1-\alpha$, for some chosen level $\alpha \in (0,1)$. 
This property is called {\em marginal} because it treats both $\bm{Y}^{(n+1)}$ and the data in $\mathcal{D}$ as random samples from $P$.

\subsection{Benefits and Limitations of Marginal Coverage}

Marginal coverage is not only convenient, since it is achievable under quite realistic assumptions, but also useful. 
For example, in motion planning, prediction bands with simultaneous marginal coverage can help autonomous vehicles decide on a path that is unlikely to collide with another vehicle or pedestrian at any point in time. 
However, the marginal nature of Equation~\eqref{eq:simu_coverage} is not always fully satisfactory, particularly because it may obscure the adverse impacts of heterogeneity across trajectories, as explained next.

Imagine forecasting the movement of pedestrians crossing a street at night.
Suppose that 90\% of them are sober, walking in highly predictable patterns, while the remaining 10\% are intoxicated. 
See Figure~\ref{fig:individual_path} for a visualization of this scenario.
It is clear that uncertainty estimation is of paramount concern while forecasting the harder-to-predict drunk trajectories.
Addressing this issue is crucial, for example, to ensure that autonomous vehicles navigate such environments with the necessary level of caution.
However, not all prediction bands with marginal coverage are equally useful in this context.
For example, 90\% marginal coverage could be easily attained even by a trivial algorithm that provides valid prediction bands only for trajectories of the ``easy'' type.
This thought experiment shows that despite their general theoretical guarantees, conformal prediction methods still require careful design to provide informative uncertainty estimates, particularly in the case of heterogeneous data.

\begin{figure}[!htb]
    \centering
    \makeatletter%
    \if@twocolumn%
    \includegraphics[width=\linewidth]{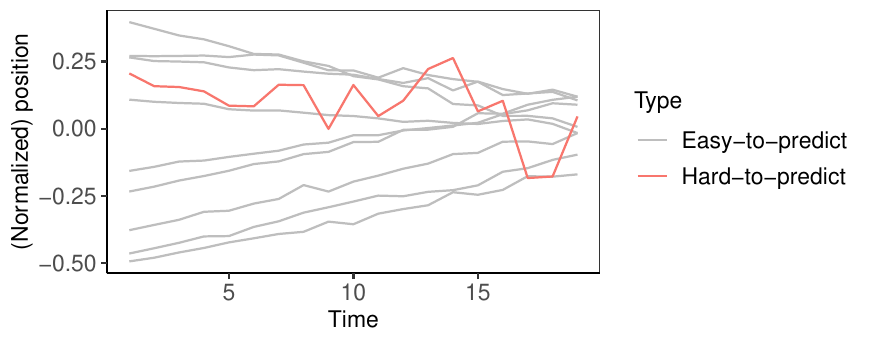}\vspace{-0.5cm}    
    \else%
    \includegraphics[width=0.7\linewidth]{figures/individual_path.pdf}\vspace{-0.5cm}    
    \fi
    \caption{One-dimensional representations of 10 pedestrian trajectories, one of which is intrinsically less predictable.} 
    \label{fig:individual_path}
\end{figure}

The aforementioned limitations of marginal coverage have been acknowledged before.
While achieving stronger theoretical guarantees in finite samples is generally unfeasible \citep{vovk2012conditional,barber2019limits}, some approaches practically tend to work better in this regard than others.
In particular, methods have been developed for regression \citep{romano2019conformalized,izbicki2019flexible}, classification \citep{romano2020classification,cauchois2021knowing,einbinder2022training}, and sketching \citep{sesia2022conformal} to seek {\em approximate conditional coverage} guarantees stronger than~\eqref{eq:simu_coverage}.

\subsection{Towards Approximate Conditional Coverage}

The goal in this paper is to construct prediction bands that are valid not only for a large fraction of all trajectories but also with high probability for distinct ``types" of trajectory.
In our street crossing example, this means we would like to have valid coverage not only marginally but also {\em conditional} on some relevant features of the pedestrian. For example, one may want $\hat{C}^{(n+1)}$ to approximately satisfy
\begin{equation}\label{eq:simu_coverage_cond}
    \mathbb{P}\left[ Y_{t}^{(n+1)} \in \hat{C}_{t}(\bm{Y}^{(n+1)}), \; \forall t  \mid \phi(\bm{Y}^{(n+1)}) \right] \geq 1-\alpha,
\end{equation}
where $\phi$ could represent the indicator of whether $\bm{Y}^{(n+1)}$ corresponds to an intoxicated pedestrian.

While there exist algorithms providing coverage conditional on a limited set of discrete features \citep{Romano2020With}, our challenge exceeds the capabilities of available approaches.
One issue is that the relevant features might not be directly observable.
For example, an autonomous vehicle might only detect a pedestrian's movements in real time, lacking broader contextual information about that person, such as knowing whether they are intoxicated or sober.
Therefore, our problem requires an innovative approach.

\subsection{Preview of Main Contributions}

We introduce a novel approach for constructing prediction bands for (multi-dimensional) trajectories, called Conformalized Adaptive Forecaster for Heterogeneous Trajectories (CAFHT). This method guarantees simultaneous marginal coverage as defined in \eqref{eq:simu_coverage} and is shown to achieve superior conditional coverage in practice compared to existing methods, as indicated by \eqref{eq:simu_coverage_cond}. 
A key feature of CAFHT is that it does not require pre-specified labels of intrinsic difficulty but rather it automatically adjusts the width of its prediction bands to each new trajectory in an online manner. 
This adaptability is derived from the capabilities of Adaptive Conformal Inference (ACI) \citep{gibbs2021adaptive},
which dynamically adjusts the prediction intervals to reflect the ease or challenge of predicting subsequent steps in a given trajectory.
Additionally, our method inherits from ACI the ability to produce prediction bands that are generally valid even for worst-case trajectories, provided these trajectories are of sufficient length \citep{gibbs2021adaptive}.

Figure~\ref{fig:individual_region} offers a glimpse into the effectiveness of CAFHT applied to the pedestrian trajectories from Figure~\ref{fig:individual_path}. Our method's advantage over state-of-the-art techniques \citep{stankeviciute2021conformal,yu2023signal} lies in its ability to automatically generate narrower bands for easier trajectories and wider ones for harder paths.
As shown through extensive experiments, this leads to more useful uncertainty estimates with higher conditional coverage.
In contrast, existing methods struggle to accommodate heterogeneity, often resulting in uniform prediction bands for all trajectories.

\begin{figure}[!htb]
    \centering
    \makeatletter%
    \if@twocolumn%
     \includegraphics[width=\linewidth]{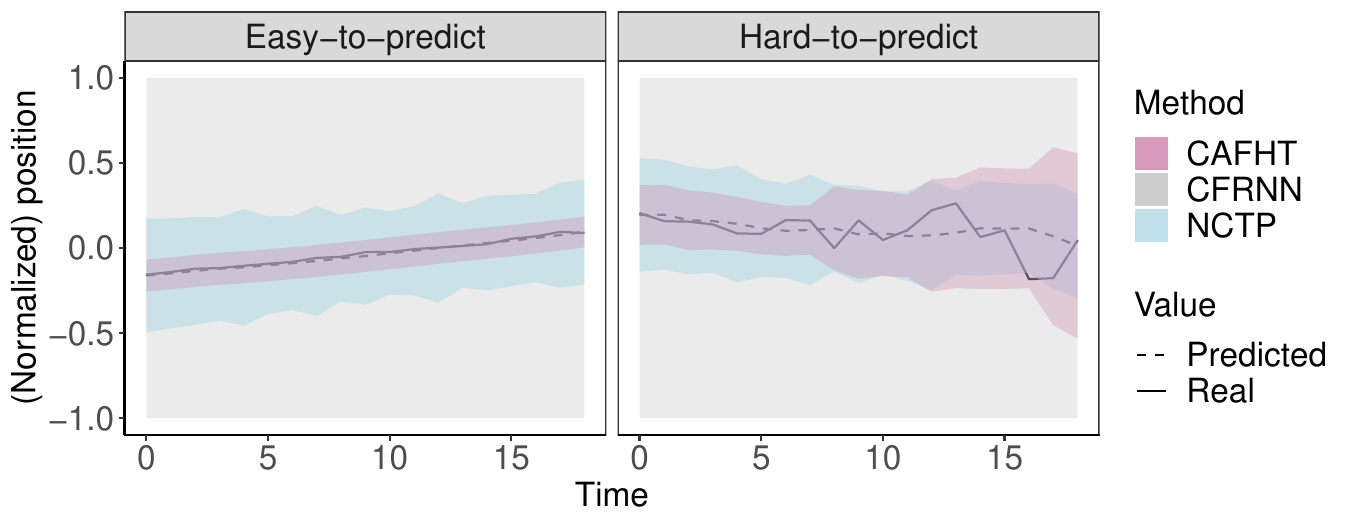}\vspace{-0.5cm}
    \else%
     \includegraphics[width=0.7\linewidth]{figures/individual_regions_simple_V2.pdf}\vspace{-0.5cm}
     \fi
    \caption{Conformal forecasting bands constructed using different methods, for the heterogeneous pedestrian trajectories from Figure~\ref{fig:individual_path}. All methods guarantee simultaneous marginal coverage at the 90\% level.
    Our method (CAFHT) can automatically adapt to the unpredictability of each trajectory. Here, the CFRNN bands so wide as to be uninformative, spanning from -1 to +1.}
    \label{fig:individual_region}
\end{figure}

In the next section, we explain how our approach integrates traditional split-conformal inference with online conformal prediction \citep{gibbs2021adaptive,gibbs2022conformal, angelopoulos2023conformal}.
Originally designed for single-series forecasting, these methods are adapted in our setting  to construct flexible prediction bands that automatically adjust to the varying unpredictability of each trajectory. 
For clarity, we begin by describing an implementation of our method based on ACI \citep{gibbs2021adaptive}, though other methods could also be accommodated, including the conformal PID approach from \citet{angelopoulos2023conformal} (discussed further in Section~\ref{sec:PID}). It is  crucial to note that all implementations of CAFHT are designed to provide the same guarantee of simultaneous marginal coverage and the same capability to accommodate heteroscedasticity.


\section{Methodology}

\subsection{Training a Black-Box Forecasting Model}

The preliminary step in our CAFHT method consists of randomly partitioning the data set $\mathcal{D}$ into two distinct subsets of trajectories, $\mathcal{D}_{\text{train}}$ and $\mathcal{D}_{\text{cal}}$.
The subset $\mathcal{D}_{\text{train}}$ is used to train a forecasting model $\hat{g}$.
This model could be almost anything, including a long short-term memory network (LSTM) \citep{hochreiter1997long,alahi2016social}, a transformer network \citep{nayakanti2022wayformer,zhou2023query}, or a traditional autoregressive moving average model \citep{wei2022}. 
Our only assumption regarding $\hat{g}$ is that it is able to generate point predictions for future steps based on partial observations from a new time series.

In this paper, we choose an LSTM model for demonstration and focus on one-step-ahead predictions.
While the ability of CAFHT to guarantee simultaneous marginal coverage does not depend on the forecasting accuracy of $\hat{g}$, more accurate models generally tend to yield more informative conformal predictions \citep{lei2018distribution}.

\subsection{Initializing the Adaptive Prediction Bands}

After training the forecaster $\hat{g}$ on the data in $\mathcal{D}_{\text{train}}$, our method will convert its one-step-ahead point predictions for any new trajectory $\bm{Y}$ into suitable {\em prediction bands}.
This is achieved by applying the ACI algorithm of \citet{gibbs2021adaptive}. 
For simplicity, we begin by focusing on the special case of one-dimensional trajectories ($d=1$). An extension of our solution to higher-dimensional trajectories is deferred to Section~\ref{sec:method-hd}.

ACI was designed to generate one-step-ahead forecasts for a single one-dimensional time series, without requiring a pre-trained forecaster $\hat{g}$.
In the single-series framework, \citet{gibbs2021adaptive} suggested training $\hat{g}$ in an online manner. 
In our setting, where we have access to multiple trajectories from the same population, it is logical to pre-train it.
In any case, pre-training does not exclude the potential for further online updates of $\hat{g}$ with each subsequent one-step-ahead prediction. 
However, to simplify the notation, our discussion now focuses on a static model.

A review of ACI \citep{gibbs2021adaptive} can be found in Appendix~\ref{app:review_ACI}. Here, we briefly highlight two critical aspects of that algorithm.
Note that the main ideas of our method can also be straightforwardly applied in combinations with other variations of the ACI method, as shown in Section~\ref{sec:PID}.

Firstly, the ACI algorithm involves a ``learning rate" parameter $\gamma >0$, controlling the adaptability of the prediction bands to the evolving time series. 
The adjustment mechanism operates as follows: at each time $t$, ACI modifies the width of the upcoming prediction interval for $Y^{(t+1)}$.
If the previous interval failed to encompass $Y^{(t)}$, the next interval is expanded; conversely, if it was sufficient, the next interval is narrowed.
Thus, larger values of $\gamma$ result in more substantial adjustments at each time step.
In contrast, lower values of $\gamma$ generally lead to ``smoother" prediction bands.

Secondly, the width of the ACI prediction band is also influenced by a parameter $\alpha \in (0,1)$, which represents the nominal level of the method. 
The design of the ACI algorithm aims to ensure that, over an extended period, the generated prediction bands will accurately contain the true value of $Y_t$ approximately a $1-\alpha$ fraction of the time. Consequently, a smaller $\alpha$ leads to broader bands.

Within our context, ACI is useful to transform the point predictions of $\hat{g}$ into {\em uncertainty-aware} prediction bands, but it is not satisfactory on its own.
Firstly, it is not always clear how to choose a good learning rate.
Secondly, the ACI prediction bands lack finite-sample guarantees. Specifically, they do not guarantee simultaneous marginal coverage~\eqref{eq:simu_coverage}.
Our method overcomes these limitations as follows. 

\subsection{Calibrating the Adaptive Prediction Bands} \label{sec:method-calibration}

We now discuss how to calibrate the ACI prediction bands discussed in the previous section to achieve simultaneous marginal coverage~\eqref{eq:simu_coverage}. For simplicity, we begin by taking the learning rate parameter $\gamma$ as fixed. We will then discuss later how to optimize the choice of $\gamma$ in a data-driven way.

Let $\hat{C}^{\text{ACI}}(\bm{Y}^{(i)}, \gamma) = [\hat{\ell}^{\text{ACI}}(\bm{Y}^{(i)}, \gamma) ,\hat{u}^{\text{ACI}}(\bm{Y}^{(i)}, \gamma)] $ denote the prediction band constructed by ACI, with learning rate $\gamma$ and level $\alpha_{\mathrm{ACI}} \in (0,1)$, for each {\em calibration} trajectory $i \in \mathcal{D}_{\text{cal}}$.
Note that this band is constructed one step at a time, based on the point predictions of $\hat{g}$ at each step $t \in [T]$ and past observations of $Y_s^{(i)}$ for all $s < t$; see Appendix~\ref{app:review_ACI} for further details on ACI.
We will refer to the cross-sectional prediction interval identified by this band at time $t \in [T]$ as $\hat{C}^{\text{ACI}}_t(\bm{Y}^{(i)}, \gamma)=[\hat{\ell}^{\text{ACI}}_t(\bm{Y}^{(i)}, \gamma), \hat{u}^{\text{ACI}}_t(\bm{Y}^{(i)}, \gamma)]$.

Our method will transform these ACI bands, which can only achieve a weaker notion of {\em asymptotic average coverage} because they do not leverage any exchangeability, into simultaneous prediction bands satisfying~\eqref {eq:simu_coverage_cond}.
For each $i \in \mathcal{D}_{\text{cal}}$, CAFHT evaluates a {\em conformity score} $\hat{\epsilon}_i (\gamma)$:
\makeatletter%
\if@twocolumn%
\begin{align}\label{eq:nonconf_scores}
\begin{split}
      \hat{\epsilon}_i (\gamma) := \max_{t \in [T]} \Biggr\{ \max \Biggr\{  
      & \left[\hat{\ell}^{\text{ACI}}_t(\bm{Y}^{(i)}, \gamma) - Y_t^{(i)}\right]_{+}  , \\
      & \quad \left[  Y_t^{(i)} - \hat{u}^{\text{ACI}}_t(\bm{Y}^{(i)}, \gamma) \right]_{+} \Biggr\} \Biggr\},
\end{split}
\end{align}
\else%
\begin{align}\label{eq:nonconf_scores}
\begin{split}
      \hat{\epsilon}_i (\gamma) := \max_{t \in [T]} \Biggr\{ \max \Biggr\{  
      & \left[\hat{\ell}^{\text{ACI}}_t(\bm{Y}^{(i)}, \gamma) - Y_t^{(i)}\right]_{+}  ,  \left[  Y_t^{(i)} - \hat{u}^{\text{ACI}}_t(\bm{Y}^{(i)}, \gamma) \right]_{+} \Biggr\} \Biggr\},
\end{split}
\end{align}
\fi
where $[x]_+ := \max(0,x)$ for any $x \in \mathbb{R}$. 
Intuitively, $\hat{\epsilon}_i (\gamma)$ measures the largest margin by which $\hat{C}^{\text{ACI}}(\bm{Y}^{(i)}, \gamma)$ should be expanded in both directions to simultaneously cover the entire trajectory $\bm{Y}^{(i)}$ from $t=1$ to $t=T$.
This is inspired by the method of \citet{romano2019conformalized} for quantile regression, although one difference is that their scores may be negative.
Other choices of conformity scores are also possible in our context, however, as discussed in Section~\ref{sec:method-adaptive-scores}.

Let $\hat{Q}(1-\alpha, \gamma)$ denote the $\lceil (1-\alpha)(1+|\mathcal{D}_{\text{cal}}|) \rceil$-th smallest value of $\hat{\epsilon}_i (\gamma)$ among $i \in \mathcal{D}_{\text{cal}}$.
CAFHT constructs a prediction band $\hat{C}(\bm{Y}^{(n+1)} ,\gamma )$ for $\bm{Y}^{(n+1)}$, one step at a time, as follows.
Let $\hat{C}^{\text{ACI}}_t(\bm{Y}^{(n+1)}, \gamma)$ denote the ACI prediction interval for $Y_t^{(n+1)}$ at time $t \in [T]$.
(Recall this depends on $\hat{g}$ and $\bm{Y}_{s}^{(n+1)}$ for all $s < t$.) 
Then, define the interval 
\makeatletter%
\if@twocolumn%
\begin{equation}\label{eq:predict_bands}
\begin{split}
    \hat{C}_t(\bm{Y}^{(n+1)} ,\gamma ) 
        & = \bigg[ \hat{\ell}^{\text{ACI}}_t(\bm{Y}^{(n+1)}, \gamma) - \hat{Q}(1-\alpha, \gamma), \\
        & \qquad \hat{u}^{\text{ACI}}_t(\bm{Y}^{(n+1)}, \gamma) + \hat{Q}(1-\alpha, \gamma) \bigg].
\end{split}
\end{equation}
    \else%
\begin{equation}\label{eq:predict_bands}
\begin{split}
    \hat{C}_t(\bm{Y}^{(n+1)} ,\gamma ) 
        & = \bigg[ \hat{\ell}^{\text{ACI}}_t(\bm{Y}^{(n+1)}, \gamma) - \hat{Q}(1-\alpha, \gamma), \hspace{0.5em}
        \hat{u}^{\text{ACI}}_t(\bm{Y}^{(n+1)}, \gamma) + \hat{Q}(1-\alpha, \gamma) \bigg].
\end{split}
\end{equation}
\fi
Our prediction band $\hat{C}(\bm{Y}^{(n+1)} ,\gamma)$ for one-step-ahead forecasting is then obtained by concatenating the intervals in~\eqref{eq:predict_bands} for all $t \in [T]$. 
More compactly, we can write $\hat{C}(\bm{Y}^{(n+1)} ,\gamma ) = \hat{C}^{\text{ACI}}(\bm{Y}^{(n+1)}, \gamma) \pm \hat{Q}(1-\alpha, \gamma)$.

The next result establishes finite-sample simultaneous coverage guarantees for this method.

\begin{theorem} \label{theorem:coverage}
Assume that the calibration trajectories in $\mathcal{D}_{\text{cal}}$ are exchangeable with $\bm{Y}^{(n+1)}$.
Then, for any $\alpha \in (0,1)$, the prediction band output by CAFHT, applied with fixed parameters $\alpha$, $\alpha_{\mathrm{ACI}}$, and $\gamma$, satisfies~\eqref{eq:simu_coverage}.
\end{theorem}

The proof of Theorem~\ref{theorem:coverage} is relatively simple and can be found in Appendix~\ref{app:proof}.
We remark that this guarantee holds at the desired level $\alpha$ regardless of the value of the ACI parameter $\alpha_{\mathrm{ACI}}$. However, it is typically intuitive to set $\alpha_{\mathrm{ACI}} = \alpha$. A notable advantage of this choice is that it leaves us with the challenge of tuning only one ACI parameter, $\gamma$.

Further, it is important to note that CAFHT can only expand the ACI prediction bands, since its conformity scores are non-negative.
Thus, our method retains the ACI guarantee of asymptotic average coverage at level $1-\alpha$ \citep{gibbs2021adaptive}, almost surely for {\em any} trajectory $\bm{Y}^{(n+1)}$: 
\begin{align} \label{eq:worst-case-coverage}
\underset{T\rightarrow \infty}{\lim} \frac{1}{T} \sum_{t=1}^T \mathbb{I}[Y_t \notin \hat{C}_t(\bm{Y}^{(n+1)} ,\gamma )] \overset{\mathrm{a.s.}}{ = }\alpha.
\end{align}
See Appendix~\ref{app:review_ACI} for details about how ACI achieves~\eqref{eq:worst-case-coverage}.

\subsection{Data-Driven Parameter Selection} \label{sec:method-selection}

The ability of the ACI algorithm to produce informative prediction bands can sometimes be sensitive to the choice of the learning rate $\gamma$ \citep{gibbs2021adaptive,angelopoulos2023conformal}.
This leads to a question: how can we select $\gamma$ in a data-driven manner? In our scenario, which involves multiple relevant trajectories from the same population, addressing this tuning challenge is somewhat simpler than in the original single-series context for which the ACI algorithm was designed.
Nonetheless, careful consideration is still required in the tuning process of $\gamma$, as we discuss next.

As a naive approach, one may feel tempted to apply the CAFHT method described above using different learning rates, with the idea of then cherry-picking the value of  $\gamma$ leading to the most appealing prediction bands. Unsurprisingly, however, such an unprincipled approach would invalidate the coverage guarantee because it breaks the exchangeability between the test trajectory and the calibration data.
This issue is closely related to problems of conformal prediction after model selection previously studied by \citet{yang2021finite} and \citet{liang2023ces}.
Therefore, we propose two alternative solutions inspired by their works.

The simplest approach to explain involves an additional data split. 
Let us randomly partition $\mathcal{D}_{\text{cal}}$ into two subsets of trajectories, $\mathcal{D}_{\text{cal}}^1$ and $\mathcal{D}_{\text{cal}}^2$.
The trajectories in $\mathcal{D}_{\text{cal}}^1$ can be utilized to select a good choice of $\gamma$ in a data-driven way.
In particular, we seek the value of $\gamma$ leading to the most informative prediction bands---a goal that can be quantified by minimizing the average width of our prediction bands produced for the trajectories in $\mathcal{D}_{\text{cal}}^1$.
Then, the calibration procedure described in Section~\ref{sec:method-calibration} will be applied using only the data in $\mathcal{D}_{\text{cal}}^2$ instead of the full $\mathcal{D}_{\text{cal}}$. 
The fact that the selection of $\gamma$ does not depend on the calibration trajectories in $\mathcal{D}_{\text{cal}}^2$ means that $\gamma$ can be essentially regarded as fixed, and therefore our output bands enjoy the marginal simultaneous coverage guarantee of Theorem~\ref{theorem:coverage}.
This version of our CAFHT method is outlined in Algorithm~\ref{alg:fixed_CAFHT_ds}.
The parameter tuning module of this procedure is summarized by Algorithm~\ref{alg:fixed_CAFHT_ds-model-selection} in Appendix~\ref{app:algorithms}.

\begin{algorithm}[!htb]
    \caption{CAFHT}
    \label{alg:fixed_CAFHT_ds}
    \begin{algorithmic} [1]
        \STATE \textbf{Input}: A pre-trained forecaster $\hat{g}$ producing one-step-ahead predictions; 
        calibration trajectories $\mathcal{D}_{\text{cal}}$; the initial position $Y_0^{(n+1)}$ of a test trajectory $\bm{Y}^{(n+1)}$;
        the desired nominal level $\alpha \in (0,1)$;
        a grid of candidate learning rates $\{\gamma_1, \dots, \gamma_L\}$.
        \STATE Randomly split $\mathcal{D}_{\text{cal}}$ into $\mathcal{D}_{\text{cal}}^1$ and $\mathcal{D}_{\text{cal}}^2$. 
        \STATE Select a learning rate $\hat{\gamma} \in \{\gamma_1, \ldots, \gamma_L\}$, applying Algorithm~\ref{alg:fixed_CAFHT_ds-model-selection} using the trajectory data in $\mathcal{D}_{\text{cal}}^1$.
        \STATE Construct $\hat{C}^{\text{ACI}}(\bm{Y}^{(i)}, \hat{\gamma})$ using ACI, for $i \in \mathcal{D}_{\text{cal}}^2$.
        \STATE Evaluate $\hat{\epsilon}_i(\hat{\gamma})$ using~\eqref{eq:nonconf_scores}, for $i \in \mathcal{D}_{\text{cal}}^2$. 
        \STATE Compute the empirical quantile $\hat{Q}(1-\alpha, \hat{\gamma})$.
        \FOR{$t \in [T]$}
        \STATE Compute $\hat{C}^{\text{ACI}}_t(\bm{Y}^{(n+1)}, \hat{\gamma})$ with ACI, using the past of the trajectory $(Y_{0}^{(n+1)},Y_{1}^{(n+1)},\ldots,Y_{t-1}^{(n+1)})$.
        \STATE Compute a prediction interval $\hat{C}_{t}(\bm{Y}^{(n+1)}, \hat{\gamma} )$ for the next step, using~\eqref{eq:predict_bands}.
        \STATE Observe the next step of the trajectory, $Y_t^{(n+1)}$.
        \ENDFOR
        \STATE \textbf{Output}: An online prediction band $\hat{C}(\bm{Y}^{(n+1)})$.
\end{algorithmic}
\end{algorithm}

Alternatively, it is also possible to carry out the selection of $\hat{\gamma}$ in a rigorous way without splitting $\mathcal{D}_{\text{cal}}$.
However, this would require replacing the empirical quantile $\hat{Q}(1-\alpha, \hat{\gamma})$ in the CAFHT method with a more conservative quantity $\hat{Q}(1-\alpha', \hat{\gamma})$, where the value of $\alpha' < \alpha$ depends on the number $L$ of candidate parameter values considered. We refer to Appendix~\ref{app:theory} for further details.

Our method employs a grid search to optimize the ACI hyper-parameters, a standard practice for hyper-parameter tuning. It is important to note that the more computationally demanding components of CAFHT, such as training the models and selecting $\gamma$ via grid search, are conducted offline and require completion only once. After these preliminary steps, the real-time component of CAFHT, which constructs prediction bands for new test trajectories, is fast and efficient.

\subsection{CAFHT with Multiplicative Scores} \label{sec:method-adaptive-scores}

A potential shortcoming of Algorithm~\ref{alg:fixed_CAFHT_ds} is that it can only add a constant margin of error to the prediction band constructed by the ACI algorithm.
While straightforward, this approach may not be always optimal.
In many cases, it would seem more natural to utilize a multiplicative error.
The rationale behind this is intuitive: trajectories that are inherently more unpredictable, resulting in wider ACI prediction bands, may necessitate larger margins of error to ensure valid simultaneous coverage. 
This concept can be seamlessly integrated into the CAFHT method by replacing the conformity scores initially outlined in~\eqref{eq:nonconf_scores} with these:
\makeatletter%
\if@twocolumn%
\begin{align}\label{eq:nonconf_scores_adap}
\begin{split}
      \tilde{\epsilon}_i (\gamma) = \max_{t \in T} \Biggr\{ \max \Biggr\{  &  \frac{\left[ \hat{\ell}^{\text{ACI}}_t(\bm{Y}^{(i)}, \gamma) - Y_t^{(i)}\right]_{+}}{|\hat{C}^{\text{ACI}}_t(\bm{Y}^{(i)}, \gamma)|}  ,  \\
      & \quad \frac{\left[  Y_t^{(i)} - \hat{u}^{\text{ACI}}_t(\bm{Y}^{(i)}, \gamma) \right]_{+}  }{|\hat{C}^{\text{ACI}}_t(\bm{Y}^{(i)}, \gamma)|}        \Biggr\}     \Biggr\}.
\end{split}
\end{align}
\else
\begin{align}\label{eq:nonconf_scores_adap}
\begin{split}
      \tilde{\epsilon}_i (\gamma) = \max_{t \in T} \Biggr\{ \max \Biggr\{  &  \frac{\left[ \hat{\ell}^{\text{ACI}}_t(\bm{Y}^{(i)}, \gamma) - Y_t^{(i)}\right]_{+}}{|\hat{C}^{\text{ACI}}_t(\bm{Y}^{(i)}, \gamma)|}  ,  \frac{\left[  Y_t^{(i)} - \hat{u}^{\text{ACI}}_t(\bm{Y}^{(i)}, \gamma) \right]_{+}  }{|\hat{C}^{\text{ACI}}_t(\bm{Y}^{(i)}, \gamma)|}        \Biggr\}     \Biggr\}.
\end{split}
\end{align}
\fi
Then, the counterpart of Equation~\eqref{eq:predict_bands} becomes 
\makeatletter%
\if@twocolumn%
\begin{align*}
\begin{split}
    \hat{C}(\bm{Y}^{(n+1)} ,\gamma ) & = \hat{C}^{\text{ACI}}(\bm{Y}^{(n+1)}, \gamma) \\
    & \qquad \pm \tilde{Q}(1-\alpha, \gamma)\cdot | \hat{C}^{\text{ACI}}(\bm{Y}^{(n+1)}, \gamma) |,
\end{split}
\end{align*}
\else
\begin{align*}
\begin{split}
    \hat{C}(\bm{Y}^{(n+1)} ,\gamma ) & = \hat{C}^{\text{ACI}}(\bm{Y}^{(n+1)}, \gamma) \pm \tilde{Q}(1-\alpha, \gamma)\cdot | \hat{C}^{\text{ACI}}(\bm{Y}^{(n+1)}, \gamma) |,
\end{split}
\end{align*}
\fi
where $\tilde{Q}(1-\alpha, \gamma)$ is the $\lceil (1-\alpha)(1+|\mathcal{D}_{\text{cal}}|) \rceil$-th smallest value in $\{ \tilde{\epsilon}_i (\gamma), i \in \mathcal{D}_{\text{cal}}\}$.

At this point, it is easy to prove that the prediction bands obtained produced by CAFHT with these multiplicative conformity scores still enjoy the same marginal simultaneous coverage guarantee established by Theorem~\ref{theorem:coverage}.

We refer to Figures~\ref{fig:multi_vs_fixed}--\ref{fig:multi_vs_fixed_full} in Appendix~\ref{app:more_experiments_multi_vs_add} for empirical illustrations and comparisons of prediction bands generated with multiplicative and additive scores; see also Table~\ref{tab:multi_vs_fixed_full} for a summary of their corresponding empirical quantiles $\hat{Q}(1-\alpha, \hat{\gamma})$.

\subsection{Extension to Multi-Dimensional Trajectories} \label{sec:method-hd}

The problem of forecasting trajectories with $d > 1$ (e.g., a two-dimensional walk), can be addressed with an intuitive extension of CAFHT.
In fact, ACI extends naturally to the multidimensional case and the first component of our method that requires some special care is the computation of the empirical quantile $\hat{Q}(1-\alpha, \hat{\gamma})$.
Yet, even this obstacle can be overcome quite easily. 
Consider evaluating a vector-valued version of the additive scores from~\eqref{eq:nonconf_scores}:
\makeatletter%
\if@twocolumn%
\begin{align}\label{eq:nonconf_scores_multid}
\begin{split}
      \hat{\epsilon}_{ij} (\gamma) := \max_{t \in [T]} \Biggr\{ \max \Biggr\{  
      & \left[\hat{\ell}^{\text{ACI}}_{t,j}(\bm{Y}^{(i)}, \gamma) - Y_{t,j}^{(i)}\right]_{+}  ,  \\
      & \quad \left[  Y_{t,j}^{(i)} - \hat{u}^{\text{ACI}}_{t,j}(\bm{Y}^{(i)}, \gamma) \right]_{+} \Biggr\} \Biggr\},
\end{split}
\end{align}
    \else%
\begin{align}\label{eq:nonconf_scores_multid}
\begin{split}
      \hat{\epsilon}_{ij} (\gamma) := \max_{t \in [T]} \Biggr\{ \max \Biggr\{  
      & \left[\hat{\ell}^{\text{ACI}}_{t,j}(\bm{Y}^{(i)}, \gamma) - Y_{t,j}^{(i)}\right]_{+}  ,  
      \left[  Y_{t,j}^{(i)} - \hat{u}^{\text{ACI}}_{t,j}(\bm{Y}^{(i)}, \gamma) \right]_{+} \Biggr\} \Biggr\},
\end{split}
\end{align}
\fi
for each dimension $j \in [d]$. 
Then, we can recover a one-dimensional problem prior to computing $\hat{Q}(1-\alpha, \hat{\gamma})$ by taking (for example) the maximum value of $\hat{\epsilon}_{ij} (\gamma)$; i.e., $\hat{\epsilon}_i^\infty(\gamma) = \max_{j \in [d]} \hat{\epsilon}_{ij} (\gamma)$.
Ultimately, each $\hat{C}_t(\bm{Y}^{(n+1)} ,\gamma )$ is obtained by applying~\eqref{eq:predict_bands} with $\hat{Q}(1-\alpha, \hat{\gamma})$ defined as the $\lceil (1-\alpha)(1+|\mathcal{D}_{\text{cal}}|) \rceil$-th smallest value of $\hat{\epsilon}_i^\infty(\gamma)$.

We conclude this section by remarking that this general idea could also be implemented using the multiplicative conformity scores described in Section~\ref{sec:method-adaptive-scores}, as well as by using different dimension reduction functions in~\eqref{eq:nonconf_scores}.
For example, one may consider replacing the infinity-norm in~\eqref{eq:nonconf_scores} with an $\ell^2$ norm, leading to a ``spherical" margin of error around the ACI prediction bands instead of a ``square" one.

\subsection{Leveraging Conformal PID Prediction Bands} \label{sec:PID}

CAFHT is not heavily reliant on the specific mechanics of ACI.
The crucial aspect of ACI is its capability to transform black-box point forecasts into prediction bands that approximately mirror the unpredictability of each trajectory.
Thus, our method can integrate with any variation of ACI.

Some of our demonstrations in Appendix~\ref{app:more_experiment} include an alternative implementation of CAFHT that employs the conformal PID algorithm of \citet{angelopoulos2023conformal} instead of ACI. 
To minimize computational demands, our demonstrations will primarily utilize the quantile tracking feature of the original conformal PID method. This simplified version of conformal PID is influenced only by a single hyper-parameter---a learning rate $\gamma$, similar to ACI.

\subsection{Direct Comparison to ACI}

CAFHT utilizes ACI as an internal component and is designed to leverage several exchangeable trajectories to construct prediction bands for a new trajectory sampled from the same population, while guaranteeing simultaneous marginal coverage as defined in~\eqref{eq:simu_coverage}. 
In contrast, ACI handles a single (arbitrary) trajectory and focuses on a different notion of asymptotic average coverage, which allows for temporary deviations of the true trajectory from the output prediction band.
This crucial distinction between CAFHT and ACI is highlighted by the numerical experiments detailed in Figures~\ref{fig:aci_vs_cafht}--\ref{fig:main_aci_vs_cafht} in Appendix~\ref{app:more_experiments_aci_vs_cafht}.

\section{Numerical Experiments}\label{sec:experiment}

\subsection{Setup and Benchmarks}

This section demonstrates the empirical performance of our method.
We focus on applying CAFHT with multiplicative scores, based on the ACI algorithm, and tuning the learning rate through data splitting.
Additional results pertaining to different implementations of CAFHT are in Appendix~\ref{app:more_experiment}. 
In all experiments, the candidate values for the ACI learning rate parameter $\gamma$ range from $0.001$ to $0.1$ at increments of $0.01$, and from $0.2$ to $0.9$ at increments of $0.1$. 

The CAFHT method is compared with two benchmark approaches that also provide simultaneous marginal coverage~\eqref{eq:simu_coverage}.
The first one is the Conformal Forecasting Recurrent Neural Network (CFRNN) approach of \citet{stankeviciute2021conformal}, which relies on a Bonferroni correction for multiple testing.
In particular, the CFRNN method produces a prediction band satisfying~\eqref{eq:simu_coverage} for a trajectory of length $T$ by separately computing $T$ conformal prediction intervals at level $\alpha/T$, one for each time step, each obtained using regression techniques typical to the regression setting.
An advantage of this approach is that it is conceptually intuitive, but it can become quite conservative if $T$ is large.

The second benchmark is the Normalized Conformal Trajectory Predictor (NCTP) of \citet{yu2023signal}.
This method is closer to ours but utilizes different scores and does not leverage ACI to adapt to heterogeneity. 
In short, NCTP directly takes as input a forecaster $\hat{g}$ providing one-step-ahead point predictions $\hat{Y}^{(i)}_t$ and evaluates the scores $\hat{\epsilon}_i = \max_{t\in [T]} \{(|\hat{Y}^{(i)}_t - Y^{(i)}_t|) / \sigma_t \}$ for each $i \in \mathcal{D}_{\mathrm{cal}}$, where $\sigma_t$ are suitable data-driven normalization constants. 
This approach is similar to that of \citet{cleaveland2023conformal}, which deviates only in the computation of the $\sigma_t$ constants, and it tends to work quite well if the trajectories are homogeneous.

While there exist other methods, such as CopulaCPTS \citep{sun2023copula}, which can achieve simultaneous marginal coverage as defined in \eqref{eq:simu_coverage}, they, like NCTP, lack adaptability to heteroscedastic conditions, and are thus expected to perform similarly under such conditions. 
For clarity and conciseness, we focus on CFRNN and NCTP as the benchmarks in our primary experiments. 
Additional experiments involving CopulaCPTS, detailed in Appendix~\ref{app:more_experiments_add_benchmark}, demonstrate performance comparable to NCTP, as anticipated.

For all methods, the underlying forecasting model is a recurrent neural network with 4 stacked LSTM layers followed by a linear layer. The learning rate is set equal to 0.001, for an AdamW optimizer with weight decay 1e-6. The models are trained for a total of 50 epochs, so that the mean squared error loss loss approximately converges. 

Prior to the beginning of our analyses, all trajectories will be pre-processed with a batch normalization step based on $\mathcal{D}_{\mathrm{train}}$, so that all values lie within the interval $[-1, 1]$.
This is useful to ensure a numerically stable learning process and more easily interpretable performance measures.

In all experiments, we evaluate the performance of the prediction bands in terms of their simultaneous marginal coverage (i.e., the proportion of test trajectories entirely contained within the prediction bands), the average width (over all times $t \in [T]$ and all test trajectories, which have a maximum value of $2$ after standardizing our data to fall within the range $[-1,1]$), and the simultaneous coverage conditional on a trajectory being ``hard-to-predict", as made more precise in the next subsection.

It is crucial to note that while we, as experiment designers, are aware of the ``difficulty label'' for each trajectory, the methods used in this study do not have access to this information. Therefore, achieving high simultaneous conditional coverage is inherently challenging. Although not theoretically guaranteed to exceed any specific threshold, higher values of this measure are preferable for practical purposes.

\subsection{Synthetic Trajectories}

We begin by considering univariate ($d=1)$ synthetic trajectories generated from an autoregressive (AR) model, $X_t = 0.9 X_{t-1} + 0.1 X_{t-2} - 0.2 X_{t-3} + \epsilon_t$, where $\epsilon_t \sim N(0, \sigma_t^2)$, for all $t \in [T]$ with $T=100$. 
Similar to \citet{stankeviciute2021conformal}, we consider two noise profiles: a dynamic profile in which $\sigma_t^2$ is increasing with time, and a static profile in which $\sigma_t^2$ is constant. 
The results based on the dynamic profile are presented here, while the others are discussed in Appendix~\ref{app:more_experiment}. 
To make the problem more interesting, we ensure that some trajectories are intrinsically more unpredictable than the others.
Specifically, in the dynamic noise setting, we set $\sigma_t^2 = t \cdot k$, with $k=10$, for a fraction $\delta = 0.1$ of the trajectories, while $\sigma_t^2 = t$ for the remaining ones.

\begin{figure*}[!t]
    \centering
    \includegraphics[width=\linewidth]{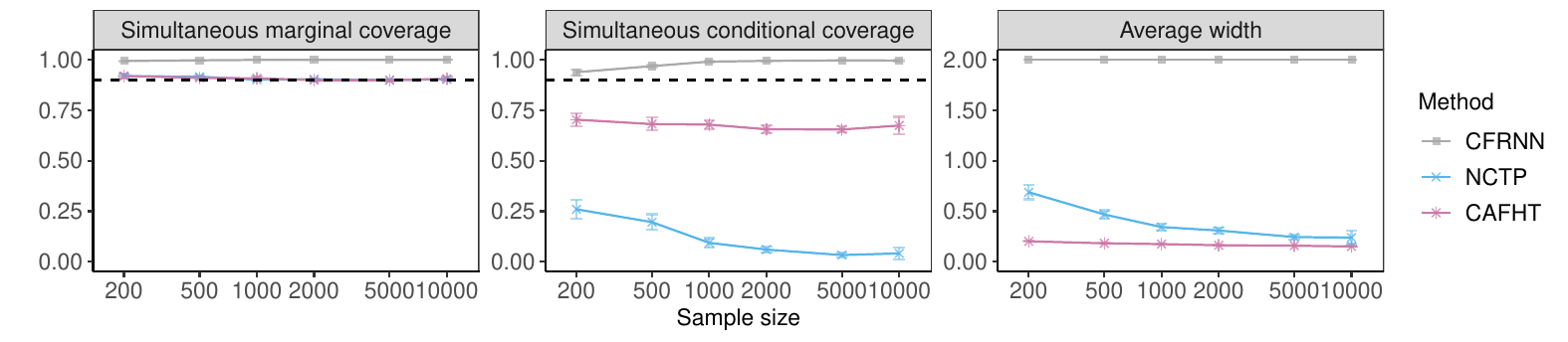}\vspace{-0.5cm}
    \caption{Performance on simulated heterogeneous trajectories of prediction bands constructed by different methods, as a function of the total number of training and calibration trajectories (of which 25\% are utilized for calibration). All methods achieve 90\% simultaneous marginal coverage. Our method (CAFHT) leads to more informative bands with lower average width and higher conditional coverage. The error bars indicate 2 standard errors. Note that the CFRNN bands here are so wide as to be uninformative.} 
    \label{fig:main_exp_sim_ndata}
\end{figure*}

Figure~\ref{fig:main_exp_sim_ndata} summarizes the performance of the three methods as a function of the number of trajectories in $\mathcal{D}$, which is varied between 200 and 10,000.
The results are averaged over 500 test trajectories and 100 independent experiments. 
See Table~\ref{tab:main_exp_sim_ndata} in Appendix~\ref{app:more_experiment} for standard errors.
In each case, $75\%$ of the trajectories are used for training and the remaining $25\%$ for calibration.
Our method utilizes 50\% of the calibration trajectories to select the ACI learning rate $\gamma$. 
All experiments target 90\% simultaneous marginal coverage, with additional results for higher coverage levels presented in Appendix~\ref{app:more_experiments_higher_cov}.

All methods attain 90\% simultaneous marginal coverage, aligning with theoretical predictions.
Notably, CAFHT yields the most informative bands, characterized by the narrowest average width and higher conditional coverage compared to NCTP. This can be explained by the fact that NCTP is not designed to account for the varying noise levels inherent in different trajectories. Consequently, NCTP generates less adaptive bands,  too wide for the easier trajectories and too narrow for the harder ones.
CAFHT also surpasses CFRNN; while CFRNN seems to attain the highest conditional coverage, it generates very wide bands that are practically uninformative for all trajectories. This is due to its rigid approach to handling time dependencies via a Bonferroni correction.

\begin{figure*}[!htb]
    \centering
    \includegraphics[width=\linewidth]{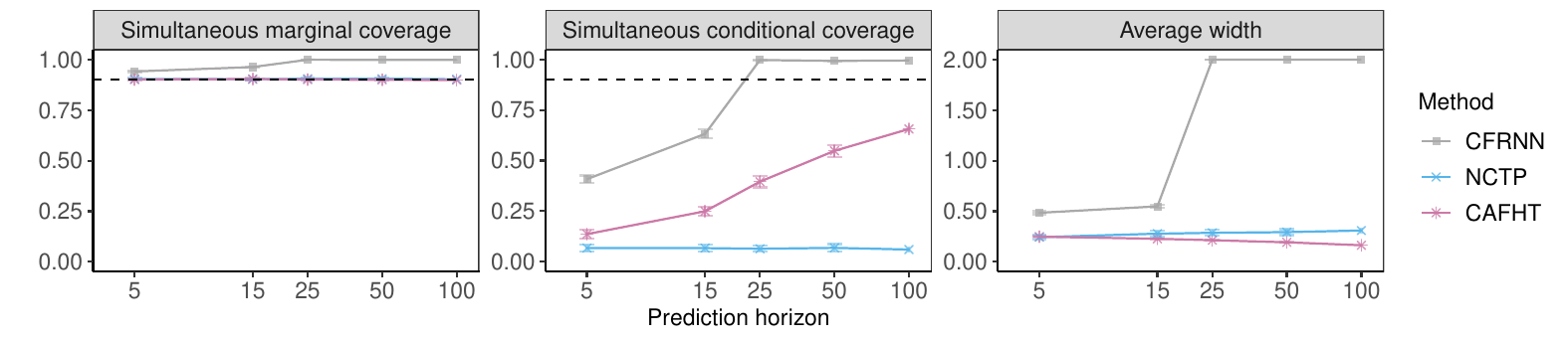}\vspace{-0.5cm}
    \caption{Performance on simulated heterogeneous trajectories of prediction bands constructed by different methods, as a function of the prediction horizon. Other details are as in Figure~\ref{fig:main_exp_sim_ndata}. For large prediction horizon, the CFRNN bands so wide as to be uninformative.} 
    \label{fig:main_exp_sim_horizon}
\end{figure*}

Figure~\ref{fig:main_exp_sim_horizon} summarizes the results of similar experiments investigating the performances of different methods as a function of the prediction horizon $T$, which is varied between 5 and 100; see Table~\ref{tab:main_exp_sim_horizon} in Appendix~\ref{app:more_experiment} for the corresponding standard errors. Here, the number of trajectories in $\mathcal{D}$ is fixed equal to 2000.
The results highlight how CFRNN becomes more conservative as $T$ increases. 
By contrast, NCTP produces relatively narrower bands but also achieves the lowest conditional coverage.
Meanwhile, our CAFHT method again yields the most informative prediction bands, with low average width and high conditional coverage. 

Appendix~\ref{app:more_experiment} describes additional experimental results that are qualitatively consistent with the main findings.
These experiments investigate the effects of the data dimensions (Figure~\ref{fig:main_exp_sim_ndim} and Table~\ref{tab:main_exp_sim_ndim}), of the proportion of hard trajectories (Figure~\ref{fig:main_exp_sim_delta} and Table~\ref{tab:main_exp_sim_delta}), and evaluate the robustness of different methods against distribution shifts (Figure~\ref{fig:main_exp_sim_delta_test} and Table~\ref{tab:main_exp_sim_delta_test}).
Additionally, these experiments are replicated using synthetic data from an AR model with a static noise profile; see Figures~\ref{fig:main_exp_sim_static_ndata}--\ref{fig:main_exp_sim_static_delta_test} and Tables~\ref{tab:main_exp_sim_static_ndata}--\ref{tab:main_exp_sim_static_delta_test}.

Furthermore, we conducted several experiments to investigate the performance of various implementations of our method.
Figures~\ref{fig:supp_exp_sim_static_ndata}--\ref{fig:supp_exp_sim_static_delta_test} and Tables~\ref{tab:supp_exp_sim_static_ndata}--\ref{tab:supp_exp_sim_static_delta_test} focus on comparing alternative model selection approaches while applying the multiplicative conformity scores defined in~\eqref{eq:nonconf_scores_adap}.
Figures~\ref{fig:supp_exp_sim_static_ndata_fixed}--\ref{fig:supp_exp_sim_static_delta_test_fixed} and Tables~\ref{tab:supp_exp_sim_static_ndata_fixed}--\ref{tab:supp_exp_sim_static_delta_test_fixed} summarize similar experiments based on the additive scores defined in~\eqref{eq:nonconf_scores}.

\subsection{Pedestrian Trajectories}

We now apply the three methods to forecast pedestrian trajectories generated from the ORCA simulator \citep{van2008reciprocal}, which follow nonlinear dynamics and are intrinsically harder to predict than the synthetic trajectories discussed before.
The data include 2-dimensional position measurements for 1,291 pedestrians, tracked over $T=20$ time steps.
To make the problem more challenging, we introduce dynamic noise to the trajectories of 10\% of randomly selected pedestrians, making their paths more unpredictable. Figure~\ref{fig:individual_path} plots ten representative trajectories.

All trajectories are normalized as in the previous section, and we train the same LSTM for 50 epochs.
In each experiment, the training and calibration sets use 1000 randomly chosen trajectories, and the test set consists of the remaining 291 trajectories. All results are averaged over 100 repetitions.

\begin{figure*}[!htb]
    \centering
    \includegraphics[width=\linewidth]{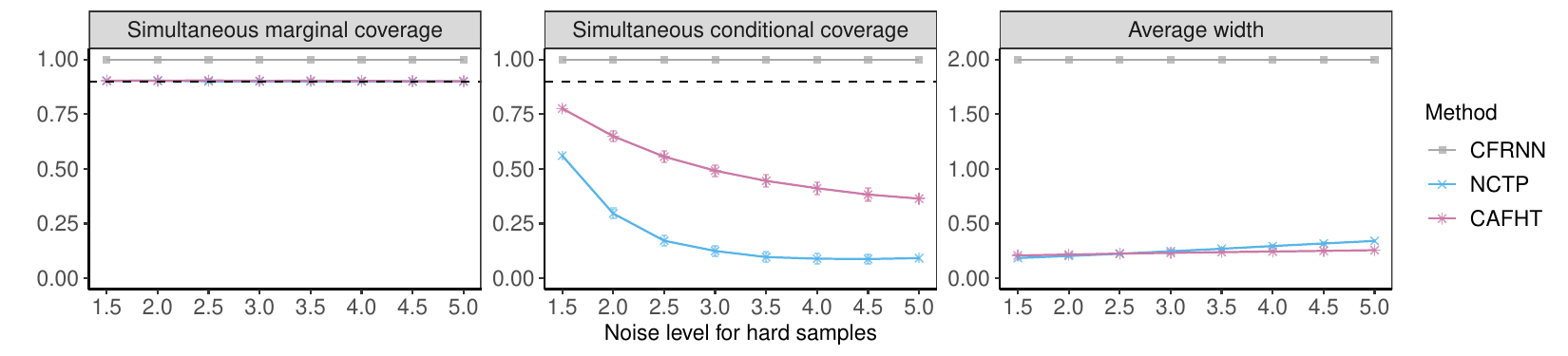}\vspace{-0.5cm}
    \caption{Performance on heterogeneous pedestrian trajectories of conformal prediction bands constructed by different methods, as a function of the noise level controlling the intrinsic unpredictability of the more difficult trajectories. Note that the CFRNN bands so wide as to be uninformative.} 
    \label{fig:main_exp_real_noiselevel_dt01}
\end{figure*}

Figure~\ref{fig:main_exp_real_noiselevel_dt01} investigates the effect of varying the noise level, setting $\sigma_t^2 \propto t \cdot \text{noise level}$ (varied from 1.5 to 5) for the hard trajectories and $\sigma_t^2 \propto t$ for the easy ones.
Again, all methods attain $90\%$ simultaneous marginal coverage, but CAFHT produces the most informative bands, with relatively narrow width and higher conditional coverage compared to NCTP.
Meanwhile, CFRNN leads to very conservative bands, as in the previous section.
See Table~\ref{tab:main_exp_real_noiselevel_dt01} in Appendix~\ref{app:more_experiment} for further details.

Additional numerical experiments are summarized in Appendix~\ref{app:more_experiment}.
Figure~\ref{fig:main_exp_real_noise_level} and Table~\ref{tab:main_exp_real_noise_level} investigate the effect of having a larger fraction of hard trajectories. 
Figure~\ref{fig:main_exp_real_dynamic_multi_ndata} and Table~\ref{tab:main_exp_real_dynamic_multi_ndata} compare the performances of different methods as a function of the sample size used for training and calibration. 
Figures~\ref{fig:supp_exp_real_ndata}--\ref{fig:supp_exp_real_noise_level_fixed} and Tables~\ref{tab:supp_exp_real_ndata}--\ref{tab:supp_exp_real_noise_level_fixed} perform a comparative analysis of different implementations of our methods under varying noise levels, using both multiplicative and additive conformity scores.


\section{Discussion}

This work opens several directions for future research.
On the theoretical side, one may want to understand the conditions under which our method can asymptotically achieve {\em optimal} prediction bands in the limit of large sample sizes, potentially drawing inspiration from \citet{lei2018distribution} and \citet{sesia2020comparison}.
Moreover, there are several potential ways to further enhance our method and address some of its remaining limitations. For example, it could be adapted to provide even stronger types of coverage guarantees beyond those considered in this paper by conditioning on the calibration data or on some other observable features. Another possible direction is to study how to best reduce the algorithmic randomness caused by data splitting \citep{vovk2015cross}, possibly using cross-conformal methods \citep{barber2019predictive} or E-value approaches \citep{bashari2023derandomized}.
Additionally, our method could be further improved by incorporating time dependency into the ACI learning rate or by relaxing the exchangeability assumption by leveraging weighted conformal inference ideas \citep{tibshirani2019conformal}.
Lastly, it would be especially interesting to apply this method in real-world motion planning scenarios.

Software implementing the algorithms and data experiments are available online at \url{https://github.com/FionaZ3696/CAFHT.git}.

\subsection*{Acknowledgements}

The authors thank anonymous referees for helpful comments, and the Center for Advanced Research Computing at the University of Southern California for providing computing resources.
M.~S.~and Y.~Z.~were partly supported by NSF grant DMS 2210637.
M.~S.~was also partly supported by an Amazon Research Award. 

\makeatletter%
\if@twocolumn%
\subsection*{Impact Statement}
This paper presents work whose goal is to advance the field of Machine Learning. There are many potential societal consequences of our work, none which we feel must be specifically highlighted here.
\fi

\microtypesetup{protrusion=false}       
\printbibliography                                                                                                
\newpage
\appendix

\renewcommand{\thesection}{A\arabic{section}}
\renewcommand{\theequation}{A\arabic{equation}}
\renewcommand{\thetheorem}{A\arabic{theorem}}
\renewcommand{\thecorollary}{A\arabic{corollary}}
\renewcommand{\theproposition}{A\arabic{proposition}}
\renewcommand{\thelemma}{A\arabic{lemma}}
\renewcommand{\thetable}{A\arabic{table}}
\renewcommand{\thefigure}{A\arabic{figure}}
\renewcommand{\thealgorithm}{A\arabic{algorithm}}
\setcounter{figure}{0}
\setcounter{table}{0}
\setcounter{proposition}{0}
\setcounter{theorem}{0}
\setcounter{lemma}{0}
\setcounter{algorithm}{0}

\section{Further Details on the ACI Algorithm}\label{app:review_ACI}

\subsection{Background on ACI}

In this section, we briefly review some relevant components of the \emph{adaptive conformal inference} (ACI) method introduced by \citet{gibbs2021adaptive} in the context of forecasting a single time series.
The goal of ACI is to construct prediction bands in an online setting, while accounting for possible changes in the data distribution across different times. Specifically, ACI is designed to create prediction bands with a long-term average coverage guarantee. Intuitively, this guarantee means that, for an indefinitely long time series, a sufficiently large proportion of the series should be contained within the output band.
This objective is notably distinct from the one investigated in our paper. However, since our method builds upon ACI, it can be useful to recall some relevant technical details of the latter method.

In the online learning setting considered by \citet{gibbs2021adaptive}, one observes covariate-response pairs $\{ (X_t,Y_t)\}_{t\in \mathbb{N}} \subset \mathbb{R}^d \times \mathbb{R}$ in a sequential fashion. At each time step $t \in \mathbb{N}$, the goal is to form a prediction set $\hat{C}_t$ for $Y_t$ using the previously observed data $\{ (X_r,Y_r)\}_{1\leq r \leq t-1}$ as well as the new covariates $X_t$. Given a target coverage level $\alpha \in (0,1)$, the constructed prediction set should guarantee that, over long time, at least $100(1-\alpha)\%$ of the time $Y_t$ lies within the set.

Recall that standard split-conformal prediction methods require a calibration dataset $\mathcal{D}_{\text{cal}}\subseteq \{ (X_r,Y_r)\}_{1\leq r \leq t-1}$ that is independent of the data used to fit the regression model.
The standard approach involves constructing a prediction set as $\hat{C}_t(\alpha) = \{y : S(X_t, y) \leq \hat{Q}(1-\alpha) \}$, where $S(X_t, y)$ is a score that measures how well $y$ conforms with the prediction of the fitted model. For example, if we denote the fitted model as $\hat{g}$, a classical example of scoring function would be $S(X_t, y) = |\hat{g}(X_t) - y|$. Then, in general, the score $S(X_t, y)$ is compared to a suitable empirical quantile, $\hat{Q}(1-\alpha)$, of the analogous scores evaluated on the calibration data: $\hat{Q}(1-\alpha) = \inf \{ s: (|\mathcal{D}_{\text{cal}}|^{-1} \sum_{(X_r, Y_r)\in \mathcal{D}_{\text{cal}}}\mathbbm{1}_{\{ S(X_r, Y_r)\leq s\} } ) \geq 1-\alpha \}$.
If the observations taken at different times are not exchangeable with one another, however, standard conformal prediction algorithms cannot achieve valid coverage. This is where ACI comes into play.

The core concept of ACI involves dynamically updating the functions $\hat{g}$, $S(\cdot)$, and $\hat{Q}(\cdot)$ at each time step, utilizing newly acquired data. Concurrently, ACI modifies the nominal miscoverage target level $\alpha_t$ of its conformal predictor for each time increment. The purpose of adjusting the $\alpha$ level at each time step is to calibrate future predictions to be more or less conservative depending on their empirical performance in covering past values of the time series. For instance, if a prediction band is found to be excessively broad, it will be narrowed in subsequent steps, and the opposite applies if it's too narrow. This strategy enables ACI to continuously adapt to potential dependencies and distribution changes within the time series, maintaining relevance and accuracy in an online context.
Specifically, ACI employs the following $\alpha$-update rule:
$$
    \alpha_{t+1} = \alpha_t + \gamma(\alpha - \text{err}_t),
$$
where
$$\text{err}_t = \begin{cases}
    1, & \text{ if } Y_t \notin \hat{C}^{\text{ACI}}_t(\alpha_t), \\
    0, & \text{ otherwise}.
\end{cases},$$
and $\hat{C}^{\text{ACI}}_t(\alpha_t) = \{y : S_t(X_t, y) \leq \hat{Q}_t(1-\alpha_t)\}$.
Equivalently,
$$\hat{C}^{\text{ACI}}_t(\alpha_t) = [\hat{\ell}^{\text{ACI}}_t, \hat{u}^{\text{ACI}}_t] = [\hat{g}(X_t) - \hat{Q}_t(1-\alpha_t), \hat{g}(X_t)+\hat{Q}_t(1-\alpha_t)].$$
The hyperparameter $\gamma > 0$ controls the magnitude of each update step. Intuitively, a larger $\gamma$ means that ACI can rapidly adjust to observed changes in the data distribution. However, this may come at the expense of increased instability in the prediction bands. Consequently, the ideal value of $\gamma$ tends to be specific to the application at hand, requiring careful consideration to balance responsiveness and stability. This is why our CAFHT method involves a data-driven parameter tuning component.

The main theoretical finding established by \citet{gibbs2021adaptive} is that ACI always attains valid long-term average coverage. Notably, this result is achieved without the necessity for any assumptions regarding the distribution of the unique time series in question.
More precisely, with probability one,
\begin{small}
\begin{align*}
    \left| \frac{1}{T}\sum_{t=1}^T \text{err}_t - \alpha \right| \leq \frac{\max(\alpha_1, 1-\alpha_1)+\gamma}{T\gamma},
\end{align*}
\end{small}
which implies
$$\underset{T\rightarrow \infty}{\lim}T^{-1}\sum_{t=1}^T \text{err}_t \overset{\mathrm{a.s.}}{ = }\alpha.$$
This result is not essential for proving our simultaneous marginal coverage guarantee, but it offers an intuitive rationale for our methodology. Indeed, the capacity of ACI to adaptively encompass the inherent variability in each time series is key to our method's enhanced conditional coverage compared to other conformal prediction approaches for multi-series forecasting.
Further, our method inherits the same long-term average coverage property of ACI because it can only expand the prediction bands of the latter.

In this paper, we implement ACI without re-training the forecasting model $\hat{g}$ at each step. This approach is viable due to our access to additional ``training" time series data from the same population, and it aids in diminishing the computational cost of our numerical experiments.
Nonetheless, our methodology is flexible enough to incorporate ACI with periodic re-training, aligning with the practices suggested by \citet{gibbs2021adaptive} and the very recent related conformal PID method of \citet{angelopoulos2023conformal}.

\subsection{Warm Starts} \label{sec:app-warmstart}

As originally designed, ACI primarily aimed at achieving asymptotic coverage in the limit of a very long trajectory, sometimes tolerating very narrow prediction intervals in the initial time steps.
However, we have observed that this behavior can negatively impact the performance of our method in finite-horizon scenarios.
To address this issue, we introduce in this paper a simple warm-start approach for ACI.
This involves incorporating artificial conformity scores at the start of each trajectory. These scores are generated as uniform random noise, with values falling within the range of observed residuals in the training dataset. Consequently, ACI typically begins with a wider interval for its first forecast.
Importantly, this modification does not affect the long-term asymptotic properties of ACI when applied to a single trajectory, nor does it impact our guarantee of finite-sample simultaneous marginal coverage. However, it often results in more informative (narrower) prediction bands.

The solution described above is applied in our experiments using 5 warm-start scores, denoted as $\hat{\epsilon}_{-4},\dots, \hat{\epsilon}_{0}$, and setting the initial value of $\alpha_{-4}$ equal to 0.1.
A similar warm-start approach is also utilized when we apply the conformal PID algorithm of \citet{angelopoulos2023conformal} instead of ACI.
However, for the algorithm the warm start simply consists of setting the initial quantile $q_0$ equal to the $(1-\alpha)$-th quantile evaluated on the empirical distribution of scores computed using the training set.

\section{Proof of Theorem~\ref{theorem:coverage}} \label{app:proof}
\begin{proof}[Proof of Theorem~\ref{theorem:coverage}]
    The proof follows directly from the exchangeability of the conformity scores, as it is often the case for split-conformal prediction methods.
    Denote $\hat{\epsilon}_{n+1}(\gamma)$ the conformity score of the test trajectory $\bm{Y}^{(t+1)}$ evaluated using the ACI prediction band constructed with step size $\gamma$. For any fixed $\alpha$ and $\gamma>0$, we have that  $Y_{t}^{(n+1)} \in \hat{C}^{(n+1)}_{t} \forall t \in [T]$
    if and only if $\hat{\epsilon}_{n+1}(\gamma) \leq \hat{Q}(1-\alpha, \gamma)$,
    where $\hat{Q}(1-\alpha, \gamma)$ is the $\lceil (1-\alpha)(1+|\mathcal{D}_{\text{cal}}|)\rceil$-th smallest value of $\hat{\epsilon}_i (\gamma)$ for all $i \in \mathcal{D}_{\text{cal}}$. Since the test trajectory is exchangeable with $\mathcal{D}_{\text{cal}}$, its score $\hat{\epsilon}_{n+1}(\gamma)$ is also exchangeable with $\{\hat{\epsilon}_{i}(\gamma), i\in\mathcal{D}_{\text{cal}}\}$. Then by Lemma 1 in \citet{romano2019conformalized}, it follows that $\mathbb{P}(Y_{t}^{(n+1)} \in \hat{C}^{(n+1)}_{t} \forall t \in [T] ) = \mathbb{P}(  \hat{\epsilon}_{n+1}(\gamma) \leq \hat{Q}(1-\alpha, \gamma)  )\geq 1-\alpha$.
\end{proof}

\clearpage

\section{Algorithms}\label{app:algorithms}

\begin{algorithm}[!htb]
    \caption{Model selection component of CAFHT}
    \label{alg:fixed_CAFHT_ds-model-selection}
    \begin{algorithmic} [1]
        \STATE \textbf{Input}: A pre-trained forecaster $\hat{g}$ producing one-step-ahead predictions;
        calibration trajectories $\mathcal{D}_{\text{cal}}^1$;
        a grid of candidate learning rates $\{\gamma_1, \dots, \gamma_L\}$.
        \FOR{$\ell \in [L]$}
            \STATE Construct $\hat{C}^{\text{ACI}}(\bm{Y}^{(i)}, \gamma_\ell)$ using ACI, for $i \in \mathcal{D}_{\text{cal}}^1$.
            \STATE Evaluate $\hat{\epsilon}_i (\gamma_\ell)$ using~\eqref{eq:nonconf_scores}, for $i \in \mathcal{D}_{\text{cal}}^1$.
            \STATE Compute $\hat{Q}(1 - \alpha, \gamma_\ell)$, the $(1-\alpha)(1+1/|\mathcal{D}_{\text{cal}}^1|)$-th quantile of $\{ \hat{\epsilon}_i (\gamma_\ell), i \in \mathcal{D}_{\text{cal}}^1\}$.
            \STATE Construct $\hat{C}(\bm{Y}^{(i)}, \gamma_\ell) = ( \hat{C}_1(\bm{Y}^{(i)}, \gamma_\ell), \dots, \hat{C}_T(\bm{Y}^{(i)}, \gamma_\ell) )$ using~\eqref{eq:predict_bands} for $i \in \mathcal{D}_{\text{cal}}^1$.
        \ENDFOR
        \STATE Pick $\hat{\gamma}$ such that,
        \begin{equation}
            \hat{\gamma} := \argmin_{\ell \in [L]} \text{AvgWidth}(C(\bm{Y}^{(i)}, \gamma_\ell)).
        \end{equation}
        \STATE \textbf{Output}: Selected learning rate parameter $\hat{\gamma}$.
\end{algorithmic}
\end{algorithm}

\begin{algorithm}[!htb]
    \caption{CAFHT - multiplicative scores}
    \label{alg:adaptive_CAFHT_ds}
    \begin{algorithmic} [1]
        \STATE \textbf{Input}: A pre-trained forecaster $\hat{g}$ producing one-step-ahead predictions;
        calibration trajectories $\mathcal{D}_{\text{cal}}$; the initial position $Y_0^{(n+1)}$ of a test trajectory $\bm{Y}^{(n+1)}$;
        the desired nominal level $\alpha \in (0,1)$;
        a grid of candidate learning rates $\{\gamma_1, \dots, \gamma_L\}$.
        \STATE Randomly split $\mathcal{D}_{\text{cal}}$ into $\mathcal{D}_{\text{cal}}^1$ and $\mathcal{D}_{\text{cal}}^2$.
        \STATE Select a learning rate $\hat{\gamma} \in \{\gamma_1, \ldots, \gamma_L\}$, applying Algorithm~\ref{alg:adaptive_CAFHT_ds-model-selection} using the trajectory data in $\mathcal{D}_{\text{cal}}^1$.
        \STATE Construct $\hat{C}^{\text{ACI}}(\bm{Y}^{(i)}, \hat{\gamma})$ using ACI, for $i \in \mathcal{D}_{\text{cal}}^2$.
        \STATE Evaluate $\hat{\epsilon}_i(\hat{\gamma})$ using~\eqref{eq:nonconf_scores_adap}, for $i \in \mathcal{D}_{\text{cal}}^2$.
        \STATE Compute the empirical quantile $\hat{Q}(1-\alpha, \hat{\gamma})$.
        \FOR{$t \in [T]$}
        \STATE Compute $\hat{C}^{\text{ACI}}_t(\bm{Y}^{(n+1)}, \hat{\gamma})$ with ACI, using the past of the test trajectory $(Y_{0}^{(n+1)},Y_{1}^{(n+1)},\ldots,Y_{t-1}^{(n+1)})$.
        \STATE Compute a prediction interval $\hat{C}_{t}(\bm{Y}^{(n+1)}, \hat{\gamma} )$ for the next step, using the multiplicative version of~\eqref{eq:predict_bands}.
        \STATE Observe the next step of the trajectory, $Y_t^{(n+1)}$.
        \ENDFOR
        \STATE \textbf{Output}: An online prediction band $\hat{C}(\bm{Y}^{(n+1)})$.
\end{algorithmic}
\end{algorithm}

\begin{algorithm}[!htb]
    \caption{Model selection component of CAFHT - multiplicative scores}
    \label{alg:adaptive_CAFHT_ds-model-selection}
    \begin{algorithmic} [1]
        \STATE \textbf{Input}: A pre-trained forecaster $\hat{g}$ producing one-step-ahead predictions;
        calibration trajectories $\mathcal{D}_{\text{cal}}^1$;
        a grid of candidate learning rates $\{\gamma_1, \dots, \gamma_L\}$.
        \FOR{$\ell \in [L]$}
            \STATE Construct $\hat{C}^{\text{ACI}}(\bm{Y}^{(i)}, \gamma_\ell)$ using ACI, for $i \in \mathcal{D}_{\text{cal}}^1$.
            \STATE Evaluate $\hat{\epsilon}_i (\gamma_\ell)$ using~\eqref{eq:nonconf_scores_adap}, for $i \in \mathcal{D}_{\text{cal}}^1$.
            \STATE Compute $\hat{Q}(1 - \alpha, \gamma_\ell)$, the $(1-\alpha)(1+1/|\mathcal{D}_{\text{cal}}^1|)$-th quantile of $\{ \hat{\epsilon}_i (\gamma_\ell), i \in \mathcal{D}_{\text{cal}}^1\}$.
            \STATE Construct $\hat{C}(\bm{Y}^{(i)}, \gamma_\ell) = ( \hat{C}_1(\bm{Y}^{(i)}, \gamma_\ell), \dots, \hat{C}_T(\bm{Y}^{(i)}, \gamma_\ell) )$ for $i \in \mathcal{D}_{\text{cal}}^1$,  using the multiplicative version of~\eqref{eq:predict_bands}.
        \ENDFOR
        \STATE Pick $\hat{\gamma}$ such that,
        \begin{equation}
            \hat{\gamma} := \argmin_{\ell \in [L]} \text{AvgWidth}(C(\bm{Y}^{(i)}, \gamma_\ell)).
        \end{equation}
        \STATE \textbf{Output}: Selected learning rate parameter $\hat{\gamma}$.
\end{algorithmic}
\end{algorithm}

\FloatBarrier

\section{Parameter Tuning for CAFHT Without Data Splitting}\label{app:theory}

Here, we outline an alternate implementation of CAFHT which, in contrast to the primary method described in Section~\ref{sec:method-selection}, obviates the need for additional subdivision of the calibration data in $\mathcal{D}_{\text{cal}}$ for selecting an optimal value of the ACI learning rate parameter $\gamma$.
In essence, this version of CAFHT employs the same calibration dataset $\mathcal{D}_{\text{cal}}$ for both choosing $\hat{\gamma}$ and calibrating the conformal margin of error via $\hat{Q}(1-\alpha', \hat{\gamma})$. It does so by using a judiciously selected $\alpha' < \alpha$ to compensate for the selection step. Enabled by the theoretical results of \citet{yang2021finite} and \citet{liang2023ces}, this method is outlined below by Algorithms~\ref{alg:fixed_CAFHT_theory}--\ref{alg:fixed_CAFHT_theory-model-selection} using additive conformity scores, and by Algorithms~\ref{alg:adaptive_CAFHT_theory}--\ref{alg:adaptive_CAFHT_theory-model-selection} using multiplicative conformity scores.

In the following, we will assume that the goal is for CAFHT to select a good $\hat{\gamma}$ from a list of $L$ candidate parameter values, $\gamma_1$, \ldots, $\gamma_L$, for some fixed integer $L \geq 1$.


Using the DKW inequality, \citet{yang2021finite} proves that, when calibrating at the nominal level $\alpha$, a conformal prediction set $\hat{C}^{(n+1)}$ constructed after using the same calibration set $\mathcal{D}_{\text{cal}}$ to select the best model among $L$ candidates may have an inflated coverage rate in the following form:
\begin{equation}\label{eq:dkw_bound}
    \mathbb{P}(Y^{(n+1)} \in \hat{C}^{(n+1)}) \geq \left( 1+\frac{1}{|\mathcal{D}_{\text{cal}}|}\right)(1-\alpha) - \frac{\sqrt{\log(2L)/2}+ c(L)}{\sqrt{|\mathcal{D}_{\text{cal}}|}},
\end{equation}
where $c(L)$ is a constant that is generally smaller than $1/3$ and can be computed explicitly,
$$c(L) = \frac{\sqrt{2}Le^{-\log(2L)}}{\sqrt{\log(2L)}+\sqrt{\log(2L)+4/\pi}}.$$
This justifies applying CAFHT, without data splitting, using $\hat{Q}(1-\alpha'_{\text{DKW}}, \hat{\gamma})$ instead of $\hat{Q}(1-\alpha, \hat{\gamma})$, where
\begin{align*}
    & \alpha'_{\text{DKW}} = 1- \frac{1-\alpha+ \text{err}}{1+1/|\mathcal{D}_{\text{cal}}|},
    & \text{err} = \frac{\sqrt{\log(2L) /2}+c(L)}{\sqrt{|\mathcal{D}_{\text{cal}}|}}.
\end{align*}

A further refinement of this approach was proposed by \citet{liang2023ces}, which suggested instead using
\begin{equation}\label{eq:theoretical_correction}
    \alpha' = \max\{ \alpha'_{\text{Markov}}, \alpha'_{\text{DKW}}  \},
\end{equation}
where $\alpha'_{\text{Markov}}$ is computed as follows.
By combining the results of \citet{vovk2012conditional} with Markov's inequality, \citet{liang2023ces} proved the following inequality in the same context of~\eqref{eq:dkw_bound}:
\begin{equation}\label{eq:markov_bound}
\mathbb{P}(Y^{(n+1)} \in \hat{C}^{(n+1)}) \geq I^{-1}\left(\frac{1}{bL}; |\mathcal{D}_{\text{cal}}|+1-l,l\right) \cdot (1-1/b),
\end{equation}
where $I^{-1}(x; |\mathcal{D}_{\text{cal}}|+1-l,l)$ is the inverse Beta cumulative distribution function with $l=\lfloor \alpha(|\mathcal{D}_{\text{cal}}|+1) \rfloor$, and $b>1$ is any fixed constant.
Therefore, the desired value of $\alpha'_{\text{Markov}}$ can be calculated by inverting~\eqref{eq:markov_bound} numerically, with the choice of $b=100$ recommended by \citet{liang2023ces}.
In particular, we generate a grid of $\hat{\alpha}$ candidates, evaluate the Markov lower bounds associated with each $\hat{\alpha}$, and then return the largest possible $\hat{\alpha}$ such that its Markov bound is greater than $1-\alpha$.

A potential advantage of the bound in~\eqref{eq:markov_bound} relative to~\eqref{eq:dkw_bound} is that the $[\sqrt{\log(2L)/2}+c(L)]/ \sqrt{|\mathcal{D}_{\text{cal}}|}$ term in the latter does not depend on $\alpha$. That makes~\eqref{eq:dkw_bound} sometimes too conservative when $\alpha$ is small; see Appendix A1.2 in \citet{liang2023ces}. However, neither bound always dominates the other, hence why we adaptively follow the tighter one using \eqref{eq:theoretical_correction}.

The performance of CAFHT applied without data splitting, relying instead on the theoretical correction for parameter tuning described above, is investigated empirically in Appendix \ref{app:more_experiment}.


\begin{algorithm}[!htb]
    \caption{CAFHT (theory)}
    \label{alg:fixed_CAFHT_theory}
    \begin{algorithmic} [1]
        \STATE \textbf{Input}: A pre-trained forecaster $\hat{g}$ producing one-step-ahead predictions;
        calibration trajectories $\mathcal{D}_{\text{cal}}$; the initial position $Y_0^{(n+1)}$ of a test trajectory $\bm{Y}^{(n+1)}$;
        the desired nominal level $\alpha \in (0,1)$;
        a grid of candidate learning rates $\{\gamma_1, \dots, \gamma_L\}$.
        \STATE Select a learning rate $\hat{\gamma} \in \{\gamma_1, \ldots, \gamma_L\}$, applying Algorithm~\ref{alg:fixed_CAFHT_theory-model-selection} using the trajectory data in $\mathcal{D}_{\text{cal}}$.
        \STATE Construct $\hat{C}^{\text{ACI}}(\bm{Y}^{(i)}, \hat{\gamma})$ using ACI, for $i \in \mathcal{D}_{\text{cal}}$.
        \STATE Evaluate $\hat{\epsilon}_i(\hat{\gamma})$ using~\eqref{eq:nonconf_scores_adap}, for $i \in \mathcal{D}_{\text{cal}}$.
        \STATE Compute the empirical quantile $\hat{Q}(1-\alpha', \hat{\gamma})$, where $\alpha'$ is defined in~\eqref{eq:theoretical_correction}.
        \FOR{$t \in [T]$}
        \STATE Compute $\hat{C}^{\text{ACI}}_t(\bm{Y}^{(n+1)}, \hat{\gamma})$ with ACI, using the past of the test trajectory $(Y_{0}^{(n+1)},Y_{1}^{(n+1)},\ldots,Y_{t-1}^{(n+1)})$.
        \STATE Compute a prediction interval $\hat{C}_{t}(\bm{Y}^{(n+1)}, \hat{\gamma} )$ for the next step, using~\eqref{eq:predict_bands} with $\hat{Q}(1-\alpha',\hat{\gamma})$.
        \STATE Observe the next step of the trajectory, $Y_t^{(n+1)}$.
        \ENDFOR
        \STATE \textbf{Output}: An online prediction band $\hat{C}(\bm{Y}^{(n+1)})$.
\end{algorithmic}
\end{algorithm}

\begin{algorithm}[!htb]
    \caption{Model selection component of CAFHT (theory)}
    \label{alg:fixed_CAFHT_theory-model-selection}
    \begin{algorithmic} [1]
        \STATE \textbf{Input}: A pre-trained forecaster $\hat{g}$ producing one-step-ahead predictions;
        calibration trajectories $\mathcal{D}_{\text{cal}}$;
        a grid of candidate learning rates $\{\gamma_1, \dots, \gamma_L\}$.
        \FOR{$\ell \in [L]$}
            \STATE Construct $\hat{C}^{\text{ACI}}(\bm{Y}^{(i)}, \gamma_\ell)$ using ACI, for $i \in \mathcal{D}_{\text{cal}}$.
            \STATE Evaluate $\hat{\epsilon}_i (\gamma_\ell)$ using~\eqref{eq:nonconf_scores}, for $i \in \mathcal{D}_{\text{cal}}$.
            \STATE Compute $\hat{Q}(1 - \alpha', \gamma_\ell)$, the $(1-\alpha')(1+1/|\mathcal{D}_{\text{cal}}|)$-th smallest value of $\{ \hat{\epsilon}_i (\gamma_\ell), i \in \mathcal{D}_{\text{cal}}\}$, where $\alpha'$ is defined in~\eqref{eq:theoretical_correction}.
            \STATE Construct $\hat{C}(\bm{Y}^{(i)}, \gamma_\ell) = ( \hat{C}_1(\bm{Y}^{(i)}, \gamma_\ell), \dots, \hat{C}_T(\bm{Y}^{(i)}, \gamma_\ell) )$ using~\eqref{eq:predict_bands} for $i \in \mathcal{D}_{\text{cal}}$.
        \ENDFOR
        \STATE Pick $\hat{\gamma}$ such that,
        \begin{equation}
            \hat{\gamma} := \argmin_{\ell \in [L]} \text{AvgWidth}(C(\bm{Y}^{(i)}, \gamma_\ell)).
        \end{equation}
        \STATE \textbf{Output}: Selected learning rate parameter $\hat{\gamma}$.
\end{algorithmic}
\end{algorithm}

\begin{algorithm}[!htb]
    \caption{CAFHT (theory) - multiplicative scores}
    \label{alg:adaptive_CAFHT_theory}
    \begin{algorithmic} [1]
        \STATE \textbf{Input}: A pre-trained forecaster $\hat{g}$ producing one-step-ahead predictions;
        calibration trajectories $\mathcal{D}_{\text{cal}}$; the initial position $Y_0^{(n+1)}$ of a test trajectory $\bm{Y}^{(n+1)}$;
        the desired nominal level $\alpha \in (0,1)$;
        a grid of candidate learning rates $\{\gamma_1, \dots, \gamma_L\}$.
        \STATE Select a learning rate $\hat{\gamma} \in \{\gamma_1, \ldots, \gamma_L\}$, applying Algorithm~\ref{alg:adaptive_CAFHT_theory-model-selection} using the trajectory data in $\mathcal{D}_{\text{cal}}$.
        \STATE Construct $\hat{C}^{\text{ACI}}(\bm{Y}^{(i)}, \hat{\gamma})$ using ACI, for $i \in \mathcal{D}_{\text{cal}}$.
        \STATE Evaluate $\hat{\epsilon}_i(\hat{\gamma})$ using the multiplicative version of~\eqref{eq:predict_bands}, for $i \in \mathcal{D}_{\text{cal}}$.
        \STATE Compute the empirical quantile $\hat{Q}(1-\alpha', \hat{\gamma})$, where $\alpha'$ is defined in~\eqref{eq:theoretical_correction}.
        \FOR{$t \in [T]$}
        \STATE Compute $\hat{C}^{\text{ACI}}_t(\bm{Y}^{(n+1)}, \hat{\gamma})$ with ACI, using the past of the test trajectory $(Y_{0}^{(n+1)},Y_{1}^{(n+1)},\ldots,Y_{t-1}^{(n+1)})$.
        \STATE Compute $\hat{C}_{t}(\bm{Y}^{(n+1)}, \hat{\gamma} )$ for the next step, using the multiplicative version of~\eqref{eq:predict_bands} with $\hat{Q}(1-\alpha',\hat{\gamma})$.
        \STATE Observe the next step of the trajectory, $Y_t^{(n+1)}$.
        \ENDFOR
        \STATE \textbf{Output}: An online prediction band $\hat{C}(\bm{Y}^{(n+1)})$.
\end{algorithmic}
\end{algorithm}

\begin{algorithm}[!htb]
    \caption{Model selection component of CAFHT (theory) - multiplicative scores}
    \label{alg:adaptive_CAFHT_theory-model-selection}
    \begin{algorithmic} [1]
        \STATE \textbf{Input}: A pre-trained forecaster $\hat{g}$ producing one-step-ahead predictions;
        calibration trajectories $\mathcal{D}_{\text{cal}}$;
        a grid of candidate learning rates $\{\gamma_1, \dots, \gamma_L\}$.
        \FOR{$\ell \in [L]$}
            \STATE Construct $\hat{C}^{\text{ACI}}(\bm{Y}^{(i)}, \gamma_\ell)$ using ACI, for $i \in \mathcal{D}_{\text{cal}}$.
            \STATE Evaluate $\hat{\epsilon}_i (\gamma_\ell)$ using the multiplicative version of~\eqref{eq:predict_bands}, for $i \in \mathcal{D}_{\text{cal}}$.
            \STATE Compute $\hat{Q}(1 - \alpha', \gamma_\ell)$, the $(1-\alpha')(1+1/|\mathcal{D}_{\text{cal}}|)$-th quantile of $\{ \hat{\epsilon}_i (\gamma_\ell), i \in \mathcal{D}_{\text{cal}}\}$, where $\alpha'$ is defined in~\eqref{eq:theoretical_correction}.
            \STATE Construct $\hat{C}(\bm{Y}^{(i)}, \gamma_\ell) = ( \hat{C}_1(\bm{Y}^{(i)}, \gamma_\ell), \dots, \hat{C}_T(\bm{Y}^{(i)}, \gamma_\ell) )$ for $i \in \mathcal{D}_{\text{cal}}$.
        \ENDFOR
        \STATE Pick $\hat{\gamma}$ such that,
        \begin{equation}
            \hat{\gamma} := \argmin_{\ell \in [L]} \text{AvgWidth}(C(\bm{Y}^{(i)}, \gamma_\ell)).
        \end{equation}
        \STATE \textbf{Output}: Selected learning rate parameter $\hat{\gamma}$.
\end{algorithmic}
\end{algorithm}

\FloatBarrier

\clearpage

\section{Additional Experimental Results}\label{app:more_experiment}

\subsection{Synthetic Data}

\subsubsection{Main Results | Comparing CAFHT to CFRNN and NCTP}\label{app:main_results}

\textbf{AR data with dynamic noise profile.}
Firstly, we investigate the performance of the three methods considered in this paper, namely CAFHT, CFRNN, and NCTP, using synthetic data from an AR model with dynamic noise profile.
The default settings of the experiments are as described in Section~\ref{sec:experiment}, but this appendix contains more detailed results.

Figure~\ref{fig:main_exp_sim_ndata} and Table~\ref{tab:main_exp_sim_ndata} report on the average performance on simulated heterogeneous trajectories of prediction bands constructed by different methods as a function of the total number of training and calibration trajectories.
The number of trajectories is varied between 200 and 10,000. All methods achieve 90\% simultaneous marginal coverage.
As discussed earlier in Section~\ref{sec:experiment}, these results show that our method (CAFHT) leads to more informative bands with lower average width and higher conditional coverage.

Figure~\ref{fig:main_exp_sim_horizon} and Table~\ref{tab:main_exp_sim_horizon} show the performance of prediction bands constructed by different methods, as a function of the prediction horizon, which is varied between 5 and 100. As the prediction horizon increases, the CFRNN method becomes more and more conservative, while the CAFHT method can consistently produce small predicting bands while maintaining relatively high conditional coverage.

Figure~\ref{fig:main_exp_sim_ndim} and Table~\ref{tab:main_exp_sim_ndim} report on the performance of all methods as a function of the dimensionality of the trajectories, which is varied between 1 and 10. Again, the results show that the CAFHT method leads to more informative bands with lower average width and higher conditional coverage.

Figure~\ref{fig:main_exp_sim_delta} and Table~\ref{tab:main_exp_sim_delta} report on the performances of these methods as a function of the proportion $\delta \in [0,1]$ of hard trajectories in the population.
We assess these results at $\delta$ values of 0.1, 0.2, and 0.5. It is observed that when the dataset contains a small number of hard-to-predict trajectories, the CAFHT method achieves superior conditional coverage and yields a narrower prediction band compared to the NCTP method. As the fraction of difficult-to-predict trajectories increases, the performance of NCTP improves (there would be no heterogeneity issue if all trajectories were ``hard to predict"). Nonetheless, the CAFHT method consistently produces the narrowest, and thereby the most informative, prediction bands across the range of $\delta$ values considered.

Finally, Figure~\ref{fig:main_exp_sim_delta_test} and Table~\ref{tab:main_exp_sim_delta_test} investigate the robustness of all methods to distribution shifts.
To this end, we kept the proportion of difficult-to-predict trajectories at $0.1$ in both the training and calibration datasets, but varied this proportion in the test dataset, altering $\delta$ from $0.2$ to $0.9$ in the test set. Under these circumstances, as the calibration set and test set are not exchangeable, no method can ensure marginal coverage at the intended 90\% level.
However, as shown in Figure~\ref{fig:main_exp_sim_delta_test} and Table~\ref{tab:main_exp_sim_delta_test}, CAFHT, in practice, tends to achieve higher marginal coverage compared to NCTP. This is consistent with the fact that CAFHT typically leads to higher conditional coverage in the absence of distribution shifts \citep{einbinder2022training}.
Additionally, the increasing width of the CAFHT prediction bands as the strength of the distribution shift grows demonstrates its enhanced ability to accurately measure predictive uncertainty.

\begin{table}[!htb]
\centering
    \caption{Performance on simulated heterogeneous trajectories of prediction bands constructed by different methods, as a function of the total number of training and calibration trajectories. The red numbers indicate smaller prediction bands or higher conditional coverage. See the corresponding plot in Figure~\ref{fig:main_exp_sim_ndata}. }
  \label{tab:main_exp_sim_ndata}
  \centering
{\scriptsize
\begin{tabular}[t]{rlllll}
\toprule
\multicolumn{3}{c}{ } & \multicolumn{3}{c}{Simultaneous coverage} \\
\cmidrule(l{3pt}r{3pt}){4-6}
Sample size & Method & Average width & Conditional-hard & Conditional-easy & Marginal\\
\midrule
\addlinespace[0.3em]
\multicolumn{6}{l}{\textbf{200}}\\
\hspace{1em}200 & CFRNN & 2.000 (0.000) & \textcolor{red}{0.939 (0.007)} & 1.000 (0.000) & 0.994 (0.001)\\
\hspace{1em}200 & NCTP & 0.687 (0.035) & 0.260 (0.024) & 0.992 (0.002) & 0.919 (0.004)\\
\hspace{1em}200 & CAFHT & \textcolor{red}{0.202 (0.003)} & \textcolor{red}{0.704 (0.017)} & 0.944 (0.005) & 0.920 (0.005)\\
\addlinespace[0.3em]
\multicolumn{6}{l}{\textbf{500}}\\
\hspace{1em}500 & CFRNN & 2.000 (0.000) & \textcolor{red}{0.969 (0.005)} & 1.000 (0.000) & 0.997 (0.000)\\
\bottomrule
\end{tabular}}
\end{table}

\begin{table}[!htb]
\centering
    \caption{Performance on simulated heterogeneous trajectories of prediction bands constructed by different methods, as a function of the prediction horizon. The red numbers indicate smaller prediction bands or higher conditional coverage. See corresponding plot in Figure~\ref{fig:main_exp_sim_horizon}.}
  \label{tab:main_exp_sim_horizon}
  \centering
{\scriptsize
\begin{tabular}[t]{rlllll}
\toprule
\multicolumn{3}{c}{ } & \multicolumn{3}{c}{Simultaneous coverage} \\
\cmidrule(l{3pt}r{3pt}){4-6}
Prediction horizon & Method & Average width & Conditional-hard & Conditional-easy & Marginal\\
\midrule
\addlinespace[0.3em]
\multicolumn{6}{l}{\textbf{5}}\\
\hspace{1em}5 & CFRNN & 0.484 (0.007) & \textcolor{red}{0.408 (0.009)} & 1.000 (0.000) & 0.942 (0.001)\\
\hspace{1em}5 & NCTP & \textcolor{red}{0.242 (0.011)} & 0.067 (0.009) & 0.996 (0.001) & 0.904 (0.002)\\
\hspace{1em}5 & CAFHT & 0.247 (0.005) & \textcolor{red}{0.136 (0.011)} & 0.985 (0.001) & 0.902 (0.003)\\
\addlinespace[0.3em]
\multicolumn{6}{l}{\textbf{15}}\\
\hspace{1em}15 & CFRNN & 0.548 (0.007) & \textcolor{red}{0.632 (0.011)} & 1.000 (0.000) & 0.964 (0.001)\\
\bottomrule
\end{tabular}}
\end{table}

\begin{figure}[!htb]
    \centering
    \includegraphics[width=\linewidth]{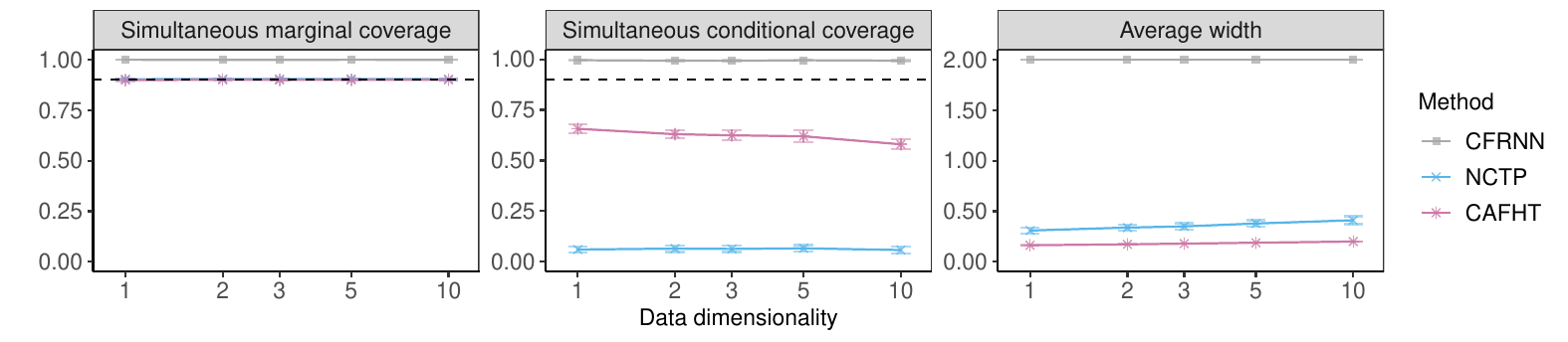}\vspace{-0.5cm}
    \caption{Performance on simulated heterogeneous trajectories of prediction bands constructed by different methods, as a function of the data dimensionality. See Table~\ref{tab:main_exp_sim_ndim} for more detailed results and standard errors.}
    \label{fig:main_exp_sim_ndim}
\end{figure}

\begin{table}[!htb]
\centering
    \caption{Performance on simulated heterogeneous trajectories of prediction bands constructed by different methods, as a function of the data dimensionality. The red numbers indicate smaller prediction bands or higher conditional coverage. See the corresponding plot in Figure~\ref{fig:main_exp_sim_ndim}.}
  \label{tab:main_exp_sim_ndim}
  \centering
{\scriptsize
\begin{tabular}[t]{rlllll}
\toprule
\multicolumn{3}{c}{ } & \multicolumn{3}{c}{Simultaneous coverage} \\
\cmidrule(l{3pt}r{3pt}){4-6}
Data dimensionality & Method & Average width & Conditional-hard & Conditional-easy & Marginal\\
\midrule
\addlinespace[0.3em]
\multicolumn{6}{l}{\textbf{1}}\\
\hspace{1em}1 & CFRNN & 2.000 (0.000) & \textcolor{red}{0.995 (0.001)} & 1.000 (0.000) & 1.000 (0.000)\\
\hspace{1em}1 & NCTP & 0.308 (0.014) & 0.060 (0.008) & 0.996 (0.001) & 0.903 (0.002)\\
\hspace{1em}1 & CAFHT & \textcolor{red}{0.163 (0.001)} & \textcolor{red}{0.656 (0.010)} & 0.926 (0.002) & 0.899 (0.003)\\
\addlinespace[0.3em]
\multicolumn{6}{l}{\textbf{2}}\\
\hspace{1em}2 & CFRNN & 2.000 (0.000) & \textcolor{red}{0.993 (0.001)} & 1.000 (0.000) & 0.999 (0.000)\\
\bottomrule
\end{tabular}}
\end{table}

\begin{figure}[!htb]
    \centering
    \includegraphics[width=\linewidth]{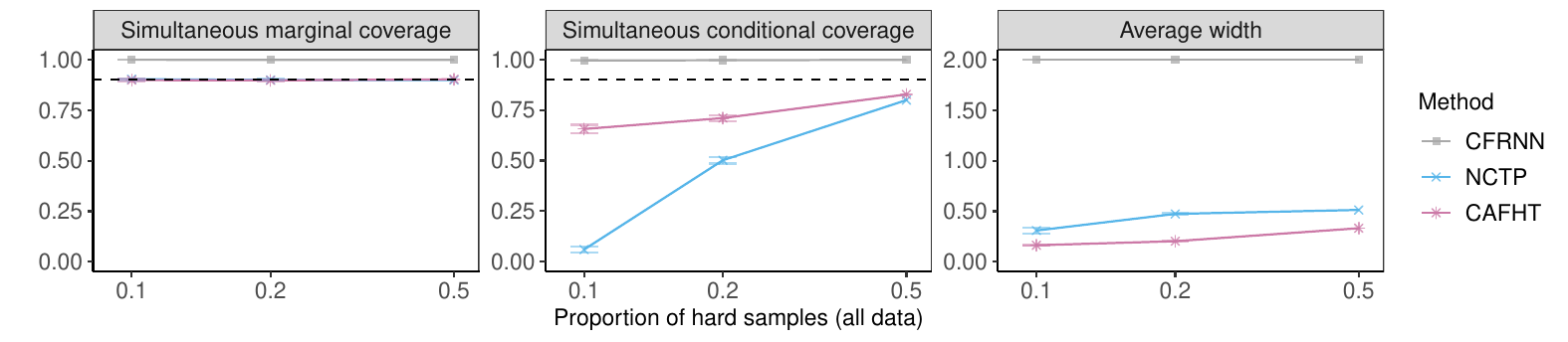}\vspace{-0.5cm}
    \caption{Performance on simulated heterogeneous trajectories of prediction bands constructed by different methods, as a function of the proportion of hard-to-predict trajectories. See Table~\ref{tab:main_exp_sim_delta} for more detailed results and standard errors.}
    \label{fig:main_exp_sim_delta}
\end{figure}

\begin{table}[!htb]
\centering
    \caption{Performance on simulated heterogeneous trajectories of prediction bands constructed by different methods, as a function of the proportion of hard-to-predict trajectories. The red numbers indicate smaller prediction bands or higher conditional coverage. See the corresponding plot in Figure~\ref{fig:main_exp_sim_delta}.}
  \label{tab:main_exp_sim_delta}
  \centering
{\scriptsize
\begin{tabular}[t]{rlllll}
\toprule
\multicolumn{3}{c}{ } & \multicolumn{3}{c}{Simultaneous coverage} \\
\cmidrule(l{3pt}r{3pt}){4-6}
\makecell{Proportion of \\ hard samples \\(all data)} & Method & Average width & Conditional-hard & Conditional-easy & Marginal\\
\midrule
\addlinespace[0.3em]
\multicolumn{6}{l}{\textbf{0.1}}\\
\hspace{1em}0.1 & CFRNN & 2.000 (0.000) & \textcolor{red}{0.995 (0.001)} & 1.000 (0.000) & 1.000 (0.000)\\
\hspace{1em}0.1 & NCTP & 0.308 (0.014) & 0.060 (0.008) & 0.996 (0.001) & 0.903 (0.002)\\
\hspace{1em}0.1 & CAFHT & \textcolor{red}{0.163 (0.001)} & \textcolor{red}{0.656 (0.010)} & 0.926 (0.002) & 0.899 (0.003)\\
\addlinespace[0.3em]
\multicolumn{6}{l}{\textbf{0.2}}\\
\hspace{1em}0.2 & CFRNN & 2.000 (0.000) & \textcolor{red}{0.997 (0.001)} & 1.000 (0.000) & 0.999 (0.000)\\
\hspace{1em}0.2 & NCTP & 0.473 (0.004) & 0.502 (0.008) & 1.000 (0.000) & 0.900 (0.002)\\
\hspace{1em}0.2 & CAFHT & \textcolor{red}{0.203 (0.001)} & \textcolor{red}{0.710 (0.007)} & 0.944 (0.002) & 0.897 (0.003)\\
\addlinespace[0.3em]
\multicolumn{6}{l}{\textbf{0.5}}\\
\hspace{1em}0.5 & CFRNN & 2.000 (0.000) & \textcolor{red}{0.999 (0.000)} & 1.000 (0.000) & 0.999 (0.000)\\
\hspace{1em}0.5 & NCTP & 0.512 (0.004) & 0.799 (0.004) & 1.000 (0.000) & 0.899 (0.002)\\
\hspace{1em}0.5 & CAFHT & \textcolor{red}{0.331 (0.002)} & \textcolor{red}{0.827 (0.005)} & 0.978 (0.001) & 0.903 (0.003)\\
\bottomrule
\end{tabular}}
\end{table}

\begin{figure}[!htb]
    \centering
    \includegraphics[width=0.75\linewidth]{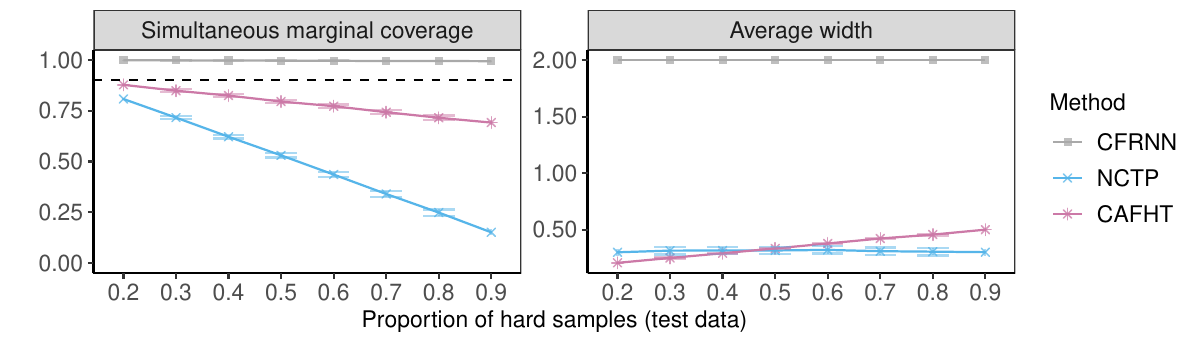}\vspace{-0.5cm}
    \caption{Performance on simulated heterogeneous trajectories of prediction bands constructed by different methods under distributional shift. The results are shown as a function of the proportion of hard-to-predict trajectories in the test data. See Table~\ref{tab:main_exp_sim_delta_test} for more detailed results and standard errors.}
    \label{fig:main_exp_sim_delta_test}
\end{figure}

\begin{table}[!htb]
\centering
    \caption{Performance on simulated heterogeneous trajectories of prediction bands constructed by different methods under distributional shift. The results are shown as a function of the proportion of hard-to-predict trajectories in the test data.  The red numbers indicate higher marginal coverage. See the corresponding plot in Figure~\ref{fig:main_exp_sim_delta_test}.}
  \label{tab:main_exp_sim_delta_test}
  \centering
{\scriptsize
\begin{tabular}[t]{rlll}
\toprule
Proportion of hard samples (test data)  & Method & Length & Marginal coverage\\
\midrule
\addlinespace[0.3em]
\multicolumn{4}{l}{\textbf{0.2}}\\
\hspace{1em}0.2 & CFRNN & 2.000 (0.000) & \textcolor{red}{0.999 (0.000)}\\
\hspace{1em}0.2 & NCTP & 0.300 (0.016) & 0.808 (0.003)\\
\hspace{1em}0.2 & CAFHT & 0.207 (0.002) & \textcolor{red}{0.877 (0.003)}\\
\addlinespace[0.3em]
\multicolumn{4}{l}{\textbf{0.3}}\\
\hspace{1em}0.3 & CFRNN & 2.000 (0.000) & \textcolor{red}{0.998 (0.000)}\\
\hspace{1em}0.3 & NCTP & 0.315 (0.017) & 0.716 (0.004)\\
\bottomrule
\end{tabular}}
\end{table}

\FloatBarrier

\textbf{AR data with static noise profile.} Next, we present the results based on data generated from the AR model with the static noise profile.

\begin{figure}[!htb]
    \centering
    \includegraphics[width=\linewidth]{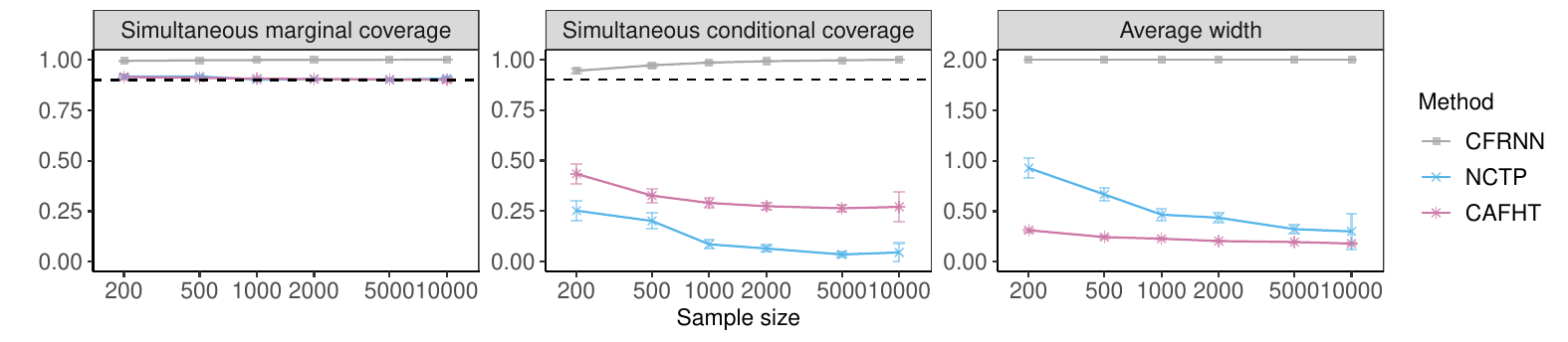}\vspace{-0.5cm}
    \caption{Performance on simulated heterogeneous trajectories of prediction bands constructed by different methods, as a function of the total number of training and calibration trajectories. See Table~\ref{tab:main_exp_sim_static_ndata} for more detailed results and standard errors.}
    \label{fig:main_exp_sim_static_ndata}
\end{figure}

\begin{table}[!htb]
\centering
    \caption{Performance on simulated heterogeneous trajectories of prediction bands constructed by different methods, as a function of the total number of training and calibration trajectories. The red numbers indicate smaller prediction bands or higher conditional coverage. See the corresponding plot in Figure~\ref{fig:main_exp_sim_static_ndata}.}
  \label{tab:main_exp_sim_static_ndata}
  \centering
{\scriptsize
\begin{tabular}[t]{rlllll}
\toprule
\multicolumn{3}{c}{ } & \multicolumn{3}{c}{Simultaneous coverage} \\
\cmidrule(l{3pt}r{3pt}){4-6}
Sample size & Method & Average width & Conditional-hard & Conditional-easy & Marginal\\
\midrule
\addlinespace[0.3em]
\multicolumn{6}{l}{\textbf{200}}\\
\hspace{1em}200 & CFRNN & 2.000 (0.000) & \textcolor{red}{0.945 (0.006)} & 1.000 (0.000) & 0.994 (0.001)\\
\hspace{1em}200 & NCTP & 0.927 (0.048) & 0.251 (0.024) & 0.990 (0.002) & 0.916 (0.004)\\
\hspace{1em}200 & CAFHT & \textcolor{red}{0.311 (0.010)} & \textcolor{red}{0.434 (0.025)} & 0.968 (0.005) & 0.914 (0.006)\\
\addlinespace[0.3em]
\multicolumn{6}{l}{\textbf{500}}\\
\hspace{1em}500 & CFRNN & 2.000 (0.000) & \textcolor{red}{0.972 (0.004)} & 1.000 (0.000) & 0.997 (0.000)\\
\bottomrule
\end{tabular}}
\end{table}

\begin{figure}[!htb]
    \centering
    \includegraphics[width=\linewidth]{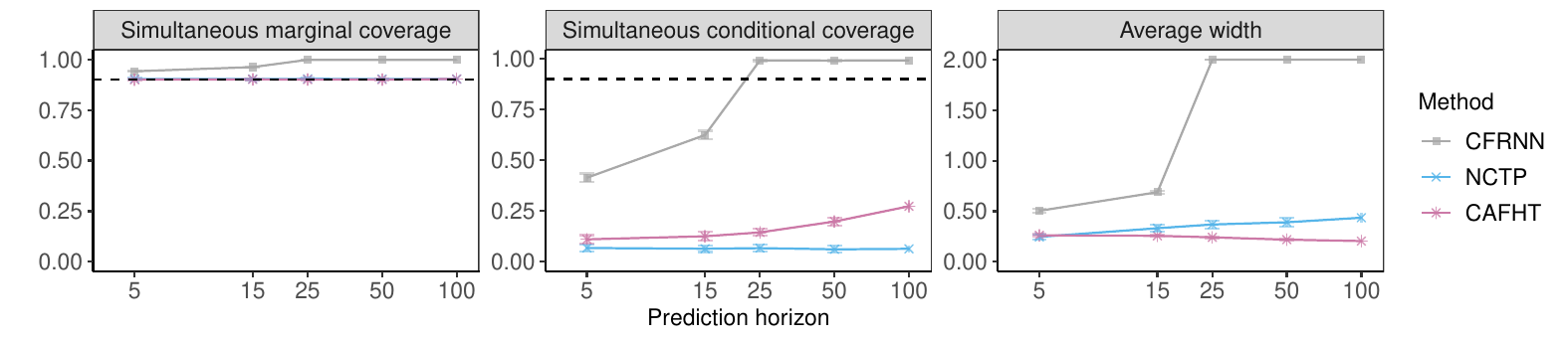}\vspace{-0.5cm}
    \caption{Performance on simulated heterogeneous trajectories of prediction bands constructed by different methods, as a function of the prediction horizon. See Table~\ref{tab:main_exp_sim_static_horizon} for more detailed results and standard errors.}
    \label{fig:main_exp_sim_static_horizon}
\end{figure}

\begin{table}[!htb]
\centering
    \caption{Performance on simulated heterogeneous trajectories of prediction bands constructed by different methods, as a function of the prediction horizon. The red numbers indicate smaller prediction bands or higher conditional coverage. See corresponding plot in Figure~\ref{fig:main_exp_sim_static_horizon}.}
  \label{tab:main_exp_sim_static_horizon}
  \centering
{\scriptsize
\begin{tabular}[t]{rlllll}
\toprule
\multicolumn{3}{c}{ } & \multicolumn{3}{c}{Simultaneous coverage} \\
\cmidrule(l{3pt}r{3pt}){4-6}
Prediction horizon & Method & Average width & Conditional-hard & Conditional-easy & Marginal\\
\midrule
\addlinespace[0.3em]
\multicolumn{6}{l}{\textbf{5}}\\
\hspace{1em}5 & CFRNN & 0.504 (0.009) & \textcolor{red}{0.414 (0.010)} & 1.000 (0.000) & 0.942 (0.001)\\
\hspace{1em}5 & NCTP & \textcolor{red}{0.246 (0.012)} & 0.067 (0.009) & 0.996 (0.001) & 0.904 (0.002)\\
\hspace{1em}5 & CAFHT & 0.260 (0.006) & \textcolor{red}{0.110 (0.011)} & 0.987 (0.001) & 0.901 (0.002)\\
\addlinespace[0.3em]
\multicolumn{6}{l}{\textbf{15}}\\
\hspace{1em}15 & CFRNN & 0.688 (0.008) & \textcolor{red}{0.624 (0.010)} & 1.000 (0.000) & 0.963 (0.001)\\
\bottomrule
\end{tabular}}
\end{table}

\begin{figure}[!htb]
    \centering
    \includegraphics[width=\linewidth]{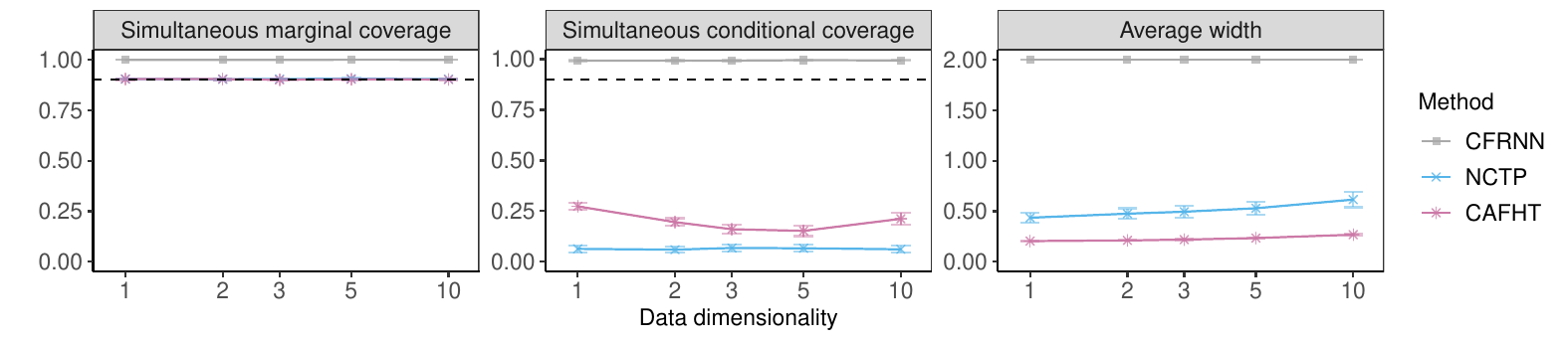}\vspace{-0.5cm}
    \caption{Performance on simulated heterogeneous trajectories of prediction bands constructed by different methods, as a function of the data dimensionality. See Table~\ref{tab:main_exp_sim_static_ndim} for more detailed results and standard errors.}
    \label{fig:main_exp_sim_static_ndim}
\end{figure}

\begin{table}[!htb]
\centering
    \caption{Performance on simulated heterogeneous trajectories of prediction bands constructed by different methods, as a function of the data dimensionality. The red numbers indicate smaller prediction bands or higher conditional coverage. See the corresponding plot in Figure~\ref{fig:main_exp_sim_static_ndim}.}
  \label{tab:main_exp_sim_static_ndim}
  \centering
{\scriptsize
\begin{tabular}[t]{rlllll}
\toprule
\multicolumn{3}{c}{ } & \multicolumn{3}{c}{Simultaneous coverage} \\
\cmidrule(l{3pt}r{3pt}){4-6}
Data dimensionality & Method & Average width & Conditional-hard & Conditional-easy & Marginal\\
\midrule
\addlinespace[0.3em]
\multicolumn{6}{l}{\textbf{1}}\\
\hspace{1em}1 & CFRNN & 2.000 (0.000) & \textcolor{red}{0.993 (0.001)} & 1.000 (0.000) & 0.999 (0.000)\\
\hspace{1em}1 & NCTP & 0.435 (0.024) & 0.063 (0.008) & 0.996 (0.001) & 0.904 (0.002)\\
\hspace{1em}1 & CAFHT & \textcolor{red}{0.204 (0.002)} & \textcolor{red}{0.273 (0.008)} & 0.973 (0.002) & 0.904 (0.002)\\
\addlinespace[0.3em]
\multicolumn{6}{l}{\textbf{2}}\\
\hspace{1em}2 & CFRNN & 2.000 (0.000) & \textcolor{red}{0.994 (0.001)} & 1.000 (0.000) & 0.999 (0.000)\\
\bottomrule
\end{tabular}}
\end{table}

\begin{figure}[!htb]
    \centering
    \includegraphics[width=\linewidth]{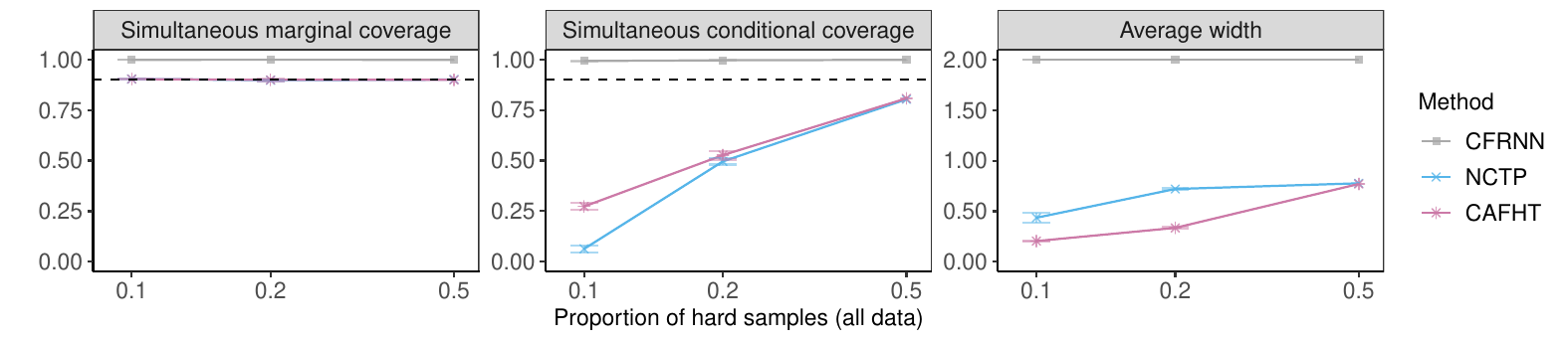}\vspace{-0.5cm}
    \caption{Performance on simulated heterogeneous trajectories of prediction bands constructed by different methods, as a function of the proportion of hard-to-predict trajectories. See Table~\ref{tab:main_exp_sim_static_delta} for more detailed results and standard errors.}
    \label{fig:main_exp_sim_static_delta}
\end{figure}

\begin{table}[!htb]
\centering
    \caption{Performance on simulated heterogeneous trajectories of prediction bands constructed by different methods, as a function of the proportion of hard-to-predict trajectories. The red numbers indicate smaller prediction bands or higher conditional coverage. See the corresponding plot in Figure~\ref{fig:main_exp_sim_static_delta}.}
  \label{tab:main_exp_sim_static_delta}
  \centering
{\scriptsize
\begin{tabular}[t]{rlllll}
\toprule
\multicolumn{3}{c}{ } & \multicolumn{3}{c}{Simultaneous coverage} \\
\cmidrule(l{3pt}r{3pt}){4-6}
\makecell{Proportion of \\ hard samples \\(all data)} & Method & Average width & Conditional-hard & Conditional-easy & Marginal\\
\midrule
\addlinespace[0.3em]
\multicolumn{6}{l}{\textbf{0.1}}\\
\hspace{1em}0.1 & CFRNN & 2.000 (0.000) & \textcolor{red}{0.993 (0.001)} & 1.000 (0.000) & 0.999 (0.000)\\
\hspace{1em}0.1 & NCTP & 0.435 (0.024) & 0.063 (0.008) & 0.996 (0.001) & 0.904 (0.002)\\
\hspace{1em}0.1 & CAFHT & \textcolor{red}{0.204 (0.002)} & \textcolor{red}{0.273 (0.008)} & 0.973 (0.002) & 0.904 (0.002)\\
\addlinespace[0.3em]
\multicolumn{6}{l}{\textbf{0.2}}\\
\hspace{1em}0.2 & CFRNN & 2.000 (0.000) & \textcolor{red}{0.996 (0.001)} & 1.000 (0.000) & 0.999 (0.000)\\
\hspace{1em}0.2 & NCTP & 0.720 (0.006) & 0.496 (0.008) & 1.000 (0.000) & 0.898 (0.002)\\
\hspace{1em}0.2 & CAFHT & \textcolor{red}{0.335 (0.005)} & \textcolor{red}{0.527 (0.011)} & 0.993 (0.001) & 0.900 (0.003)\\
\addlinespace[0.3em]
\multicolumn{6}{l}{\textbf{0.5}}\\
\hspace{1em}0.5 & CFRNN & 2.000 (0.000) & \textcolor{red}{0.999 (0.000)} & 1.000 (0.000) & 0.999 (0.000)\\
\hspace{1em}0.5 & NCTP & 0.776 (0.005) & 0.803 (0.003) & 1.000 (0.000) & 0.901 (0.002)\\
\hspace{1em}0.5 & CAFHT & \textcolor{red}{0.770 (0.007)} & \textcolor{red}{0.808 (0.004)} & 0.992 (0.001) & 0.900 (0.002)\\
\bottomrule
\end{tabular}}
\end{table}

\begin{figure}[!htb]
    \centering
    \includegraphics[width=0.75\linewidth]{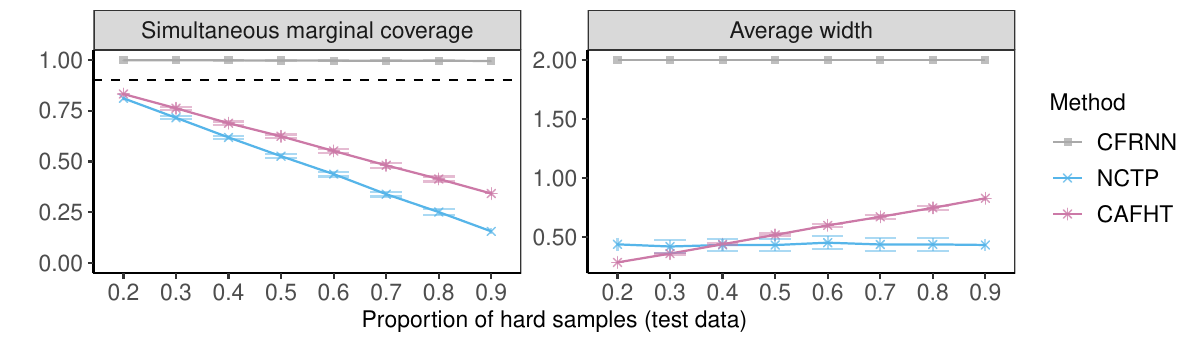}\vspace{-0.5cm}
    \caption{Performance on simulated heterogeneous trajectories of prediction bands constructed by different methods under distributional shift. The results are shown as a function of the proportion of hard-to-predict trajectories in the test data. See Table~\ref{tab:main_exp_sim_static_delta_test} for detailed results and standard errors.}
    \label{fig:main_exp_sim_static_delta_test}
\end{figure}

\begin{table}[!htb]
\centering
    \caption{Performance on simulated heterogeneous trajectories of prediction bands constructed by different methods under distributional shift. The results are shown as a function of the proportion of hard-to-predict trajectories in the test data. The red numbers indicate higher marginal coverage. See the corresponding plot in Figure~\ref{fig:main_exp_sim_static_delta_test}.}
  \label{tab:main_exp_sim_static_delta_test}
  \centering
{\scriptsize
\begin{tabular}[t]{rlll}
\toprule
Proportion of hard samples (test data) & Method & Length & Marginal coverage\\
\midrule
\addlinespace[0.3em]
\multicolumn{4}{l}{\textbf{0.2}}\\
\hspace{1em}0.2 & CFRNN & 2.000 (0.000) & \textcolor{red}{0.999 (0.000)}\\
\hspace{1em}0.2 & NCTP & 0.437 (0.028) & 0.811 (0.003)\\
\hspace{1em}0.2 & CAFHT & 0.284 (0.003) & \textcolor{red}{0.832 (0.003)}\\
\addlinespace[0.3em]
\multicolumn{4}{l}{\textbf{0.3}}\\
\hspace{1em}0.3 & CFRNN & 2.000 (0.000) & \textcolor{red}{0.998 (0.000)}\\
\hspace{1em}0.3 & NCTP & 0.419 (0.027) & 0.715 (0.004)\\
\bottomrule
\end{tabular}}
\end{table}

\FloatBarrier

\subsubsection{Supplementary Results | Comparing Different Versions of CAFHT}
In this subsection, we add different versions of CAFHT into comparison. We will separately analyze the CAFHT prediction bands constructed using multiplicative conformity scores~\eqref{eq:nonconf_scores_adap} and those constructed using additive conformity scores~\eqref{eq:nonconf_scores}. The conclusions from the results evaluated using synthetic data with the dynamic profile and with the static profile are very similar. To save space, we only demonstrate the results using data with the static profile.

We consider the following implementations of CAFHT:
\begin{itemize}
    \item CAFHT: the main method. It is the CAFHT method based on the ACI prediction band using the data splitting strategy; see Algorithm~\ref{alg:adaptive_CAFHT_ds}.
    \item CAFHT - PID: the CAFHT method based on the conformal PID prediction band using the data splitting strategy. It can be implemented simply by substituting $\hat{C}^{\text{ACI}}$ to $\hat{C}^{\text{PID}}$ in Algorithm~\ref{alg:adaptive_CAFHT_ds} wherever possible.
    \item CAFHT (theory): the CAFHT method based on the ACI prediction band after correcting the theoretical coverage; see Appendix~\ref{app:theory} and Algorithm~\ref{alg:adaptive_CAFHT_theory}.
    \item CAFHT (theory) - PID: the CAFHT method based on the conformal PID prediction band after correcting the theoretical coverage. It can be implemented simply by substituting $\hat{C}^{\text{ACI}}$ to $\hat{C}^{\text{PID}}$ in Algorithm~\ref{alg:adaptive_CAFHT_theory} wherever possible.
\end{itemize}

\FloatBarrier

\subsubsection*{CAFHT | Multiplicative Scores}
The results of CAFHT with multiplicative conformity scores \eqref{eq:nonconf_scores_adap} are first presented.

Similar to what we have observed from the results in subsection \ref{app:main_results}, CAFHT outperforms the benchmark methods (CFRNN and NCTP) across all configurations we considered. Generally, CAFHT produces narrower, more informative bands with higher conditional coverage. Among the different versions of CAFHT, the prediction bands generated using the theoretical correction approach (outlined in \ref{app:theory}) tend to be more conservative compared to those from the data-splitting approach. Additionally, in our experiments, the performance of prediction bands constructed by CAFHT with ACI is empirically similar to those created using PID.

\begin{figure}[!htb]
    \centering
    \includegraphics[width=\linewidth]{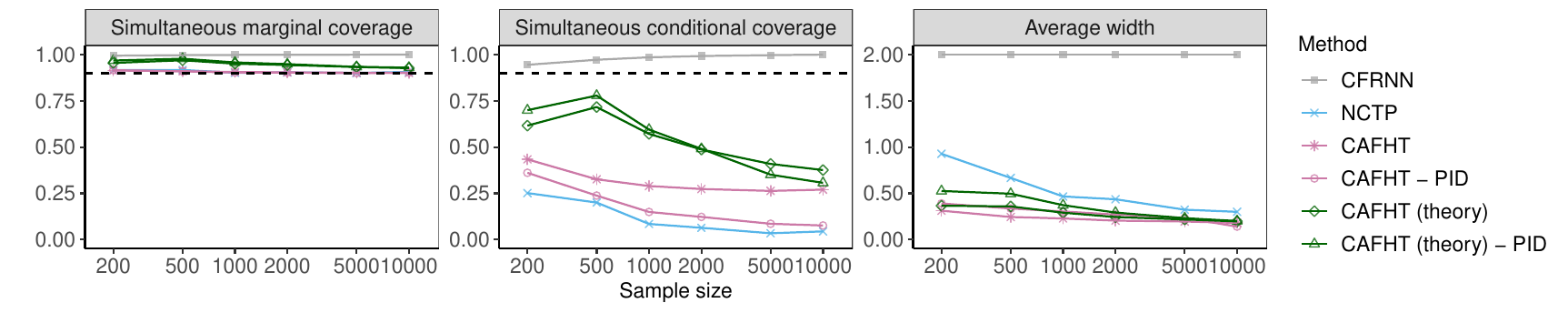}\vspace{-0.5cm}
    \caption{Performance on simulated heterogeneous trajectories of prediction bands constructed by different methods, as a function of the total number of training and calibration trajectories. See Table~\ref{tab:supp_exp_sim_static_ndata} for detailed results and standard errors.}
    \label{fig:supp_exp_sim_static_ndata}
\end{figure}

\begin{table}[!htb]
\centering
    \caption{Performance on simulated heterogeneous trajectories of prediction bands constructed by different methods, as a function of the total number of training and calibration trajectories. The red numbers indicate smaller prediction bands or higher conditional coverage. See corresponding plot in Figure~\ref{fig:supp_exp_sim_static_ndata}.}
  \label{tab:supp_exp_sim_static_ndata}
  \centering
{\scriptsize
\begin{tabular}[t]{rlllll}
\toprule
\multicolumn{3}{c}{ } & \multicolumn{3}{c}{Simultaneous coverage} \\
\cmidrule(l{3pt}r{3pt}){4-6}
Sample size & Method & Average width & Conditional-hard & Conditional-easy & Marginal\\
\midrule
\addlinespace[0.3em]
\multicolumn{6}{l}{\textbf{200}}\\
\hspace{1em}200 & CFRNN & 2.000 (0.000) & \textcolor{red}{0.945 (0.006)} & 1.000 (0.000) & 0.994 (0.001)\\
\hspace{1em}200 & NCTP & 0.927 (0.048) & 0.251 (0.024) & 0.990 (0.002) & 0.916 (0.004)\\
\hspace{1em}200 & CAFHT & \textcolor{red}{0.311 (0.010)} & \textcolor{red}{0.434 (0.025)} & 0.968 (0.005) & 0.914 (0.006)\\
\hspace{1em}200 & CAFHT - PID & 0.390 (0.022) & \textcolor{red}{0.361 (0.027)} & 0.978 (0.004) & 0.916 (0.006)\\
\hspace{1em}200 & CAFHT (theory) & 0.365 (0.009) & \textcolor{red}{0.616 (0.016)} & 0.993 (0.001) & 0.955 (0.002)\\
\bottomrule
\end{tabular}}
\end{table}

\begin{figure}[!htb]
    \centering
    \includegraphics[width=\linewidth]{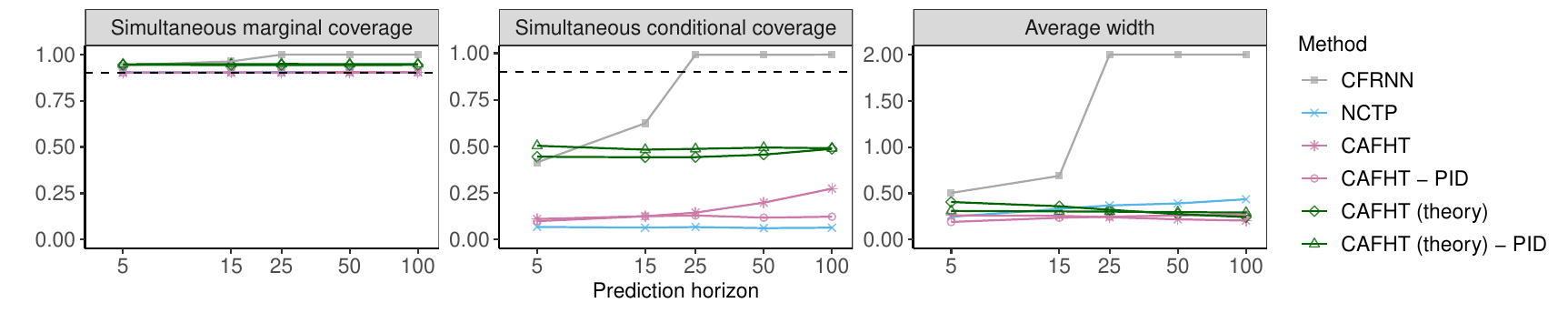}\vspace{-0.5cm}
    \caption{Performance on simulated heterogeneous trajectories of prediction bands constructed by different methods, as a function of the prediction horizon. See Table~\ref{tab:supp_exp_sim_static_horizon} for detailed results and standard errors.}
    \label{fig:supp_exp_sim_static_horizon}
\end{figure}

\begin{table}[!htb]
\centering
    \caption{Performance on simulated heterogeneous trajectories of prediction bands constructed by different methods, as a function of the prediction horizon. The red numbers indicate smaller prediction bands or higher conditional coverage. See the corresponding plot in Figure~\ref{fig:supp_exp_sim_static_horizon}.}
  \label{tab:supp_exp_sim_static_horizon}
  \centering
{\scriptsize
\begin{tabular}[t]{rlllll}
\toprule
\multicolumn{3}{c}{ } & \multicolumn{3}{c}{Simultaneous coverage} \\
\cmidrule(l{3pt}r{3pt}){4-6}
Prediction horizon & Method & Average width & Conditional-hard & Conditional-easy & Marginal\\
\midrule
\addlinespace[0.3em]
\multicolumn{6}{l}{\textbf{5}}\\
\hspace{1em}5 & CFRNN & 0.504 (0.009) & \textcolor{red}{0.414 (0.010)} & 1.000 (0.000) & 0.942 (0.001)\\
\hspace{1em}5 & NCTP & 0.246 (0.012) & 0.067 (0.009) & 0.996 (0.001) & 0.904 (0.002)\\
\hspace{1em}5 & CAFHT & 0.260 (0.006) & \textcolor{red}{0.110 (0.011)} & 0.987 (0.001) & 0.901 (0.002)\\
\hspace{1em}5 & CAFHT - PID & \textcolor{red}{0.190 (0.008)} & \textcolor{red}{0.098 (0.011)} & 0.991 (0.001) & 0.903 (0.002)\\
\hspace{1em}5 & CAFHT (theory) & 0.406 (0.007) & \textcolor{red}{0.444 (0.012)} & 0.998 (0.000) & 0.944 (0.002)\\
\bottomrule
\end{tabular}}
\end{table}

\begin{figure}[!htb]
    \centering
    \includegraphics[width=\linewidth]{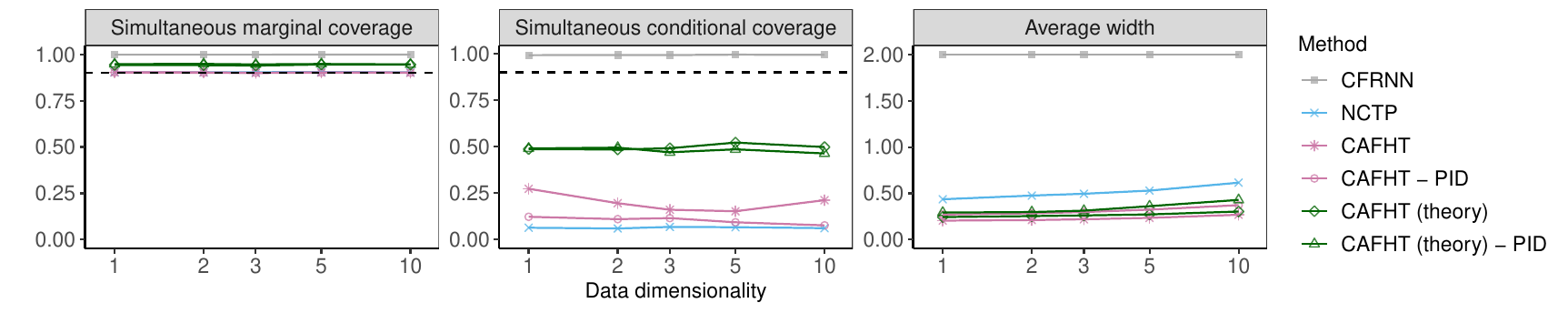}\vspace{-0.5cm}
    \caption{Performance on simulated heterogeneous trajectories of prediction bands constructed by different methods, as a function of the data dimensionality. See Table~\ref{tab:supp_exp_sim_static_ndim} for detailed results and standard errors.}
    \label{fig:supp_exp_sim_static_ndim}
\end{figure}

\begin{table}[!htb]
\centering
    \caption{Performance on simulated heterogeneous trajectories of prediction bands constructed by different methods, as a function of the data dimensionality. The red numbers indicate smaller prediction bands or higher conditional coverage. See the corresponding plot in Figure~\ref{fig:supp_exp_sim_static_ndim}.}
  \label{tab:supp_exp_sim_static_ndim}
  \centering
{\scriptsize
\begin{tabular}[t]{rlllll}
\toprule
\multicolumn{3}{c}{ } & \multicolumn{3}{c}{Simultaneous coverage} \\
\cmidrule(l{3pt}r{3pt}){4-6}
Data dimensionality & Method & Average width & Conditional-hard & Conditional-easy & Marginal\\
\midrule
\addlinespace[0.3em]
\multicolumn{6}{l}{\textbf{1}}\\
\hspace{1em}1 & CFRNN & 2.000 (0.000) & \textcolor{red}{0.993 (0.001)} & 1.000 (0.000) & 0.999 (0.000)\\
\hspace{1em}1 & NCTP & 0.435 (0.024) & 0.063 (0.008) & 0.996 (0.001) & 0.904 (0.002)\\
\hspace{1em}1 & CAFHT & \textcolor{red}{0.204 (0.002)} & \textcolor{red}{0.273 (0.008)} & 0.973 (0.002) & 0.904 (0.002)\\
\hspace{1em}1 & CAFHT - PID & 0.270 (0.016) & \textcolor{red}{0.122 (0.012)} & 0.989 (0.001) & 0.904 (0.002)\\
\hspace{1em}1 & CAFHT (theory) & 0.243 (0.003) & \textcolor{red}{0.487 (0.010)} & 0.994 (0.001) & 0.944 (0.001)\\
\bottomrule
\end{tabular}}
\end{table}

\begin{figure}[!htb]
    \centering
    \includegraphics[width=\linewidth]{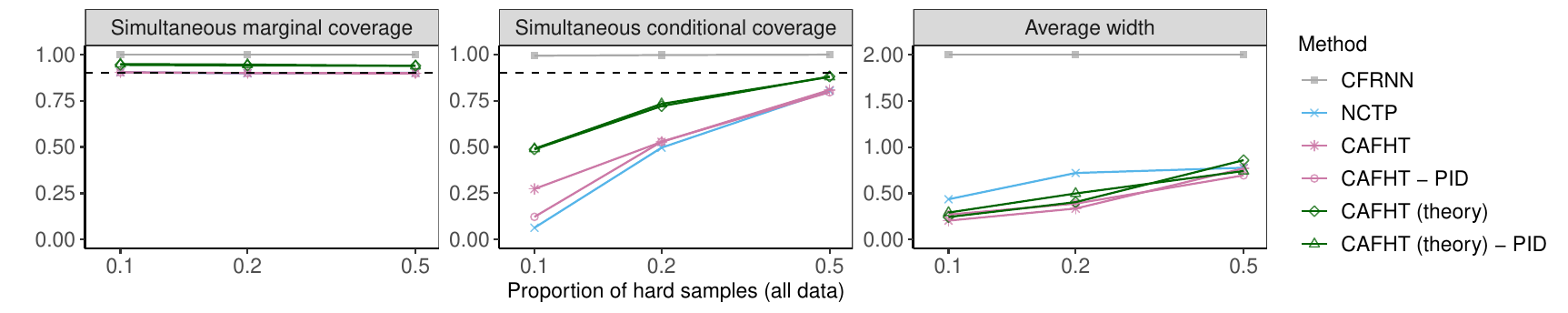}\vspace{-0.5cm}
    \caption{Performance on simulated heterogeneous trajectories of prediction bands constructed by different methods, as a function of the proportion of hard-to-predict trajectories. See Table~\ref{tab:supp_exp_sim_static_delta} for detailed results and standard errors.}
    \label{fig:supp_exp_sim_static_delta}
\end{figure}

\begin{table}[!htb]
\centering
    \caption{Performance on simulated heterogeneous trajectories of prediction bands constructed by different methods, as a function of the proportion of hard-to-predict trajectories. The red numbers indicate smaller prediction bands or higher conditional coverage. See the corresponding plot in Figure~\ref{fig:supp_exp_sim_static_delta}.}
  \label{tab:supp_exp_sim_static_delta}
  \centering
{\scriptsize
\begin{tabular}[t]{rlllll}
\toprule
\multicolumn{3}{c}{ } & \multicolumn{3}{c}{Simultaneous coverage} \\
\cmidrule(l{3pt}r{3pt}){4-6}
\makecell{Proportion of \\ hard samples \\(all data)} & Method & Average width & Conditional-hard & Conditional-easy & Marginal\\
\midrule
\addlinespace[0.3em]
\multicolumn{6}{l}{\textbf{0.1}}\\
\hspace{1em}0.1 & CFRNN & 2.000 (0.000) & \textcolor{red}{0.993 (0.001)} & 1.000 (0.000) & 0.999 (0.000)\\
\hspace{1em}0.1 & NCTP & 0.435 (0.024) & 0.063 (0.008) & 0.996 (0.001) & 0.904 (0.002)\\
\hspace{1em}0.1 & CAFHT & \textcolor{red}{0.204 (0.002)} & \textcolor{red}{0.273 (0.008)} & 0.973 (0.002) & 0.904 (0.002)\\
\hspace{1em}0.1 & CAFHT - PID & 0.270 (0.016) & \textcolor{red}{0.122 (0.012)} & 0.989 (0.001) & 0.904 (0.002)\\
\hspace{1em}0.1 & CAFHT (theory) & 0.243 (0.003) & \textcolor{red}{0.487 (0.010)} & 0.994 (0.001) & 0.944 (0.001)\\
\bottomrule
\end{tabular}}
\end{table}

\begin{figure}[!htb]
    \centering
    \includegraphics[width=0.75\linewidth]{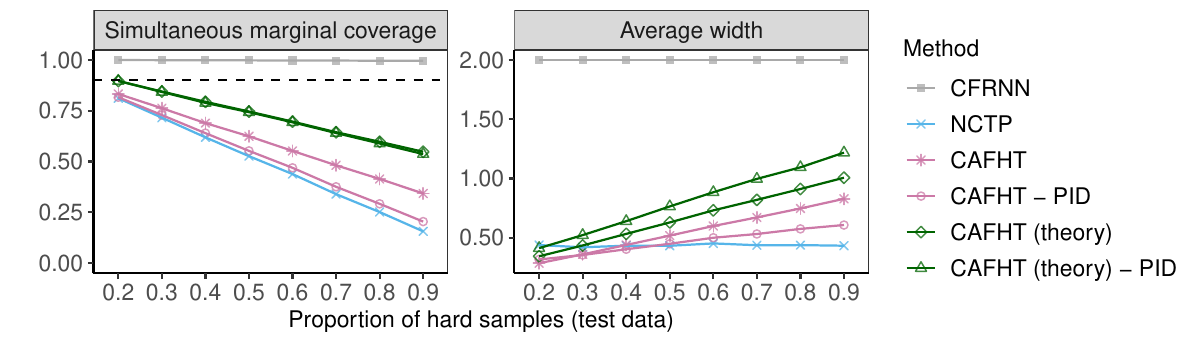}\vspace{-0.5cm}
    \caption{Performance on simulated heterogeneous trajectories of prediction bands constructed by different methods under distributional shift. The results are shown as a function of the proportion of hard-to-predict trajectories in the test data. See Table~\ref{tab:supp_exp_sim_static_delta_test} for detailed results and standard errors.}
    \label{fig:supp_exp_sim_static_delta_test}
\end{figure}

\begin{table}[!tbh]
\centering
    \caption{Performance on simulated heterogeneous trajectories of prediction bands constructed by different methods under distributional shift. The results are shown as a function of the proportion of hard-to-predict trajectories in the test data. The red numbers indicate higher marginal coverage. See the corresponding plot in Figure~\ref{fig:supp_exp_sim_static_delta_test}.}
  \label{tab:supp_exp_sim_static_delta_test}
  \centering
{\scriptsize
\begin{tabular}[t]{rlll}
\toprule
Proportion of hard samples (test data) & Method & Length & Marginal coverage\\
\midrule
\addlinespace[0.3em]
\multicolumn{4}{l}{\textbf{0.2}}\\
\hspace{1em}0.2 & CFRNN & 2.000 (0.000) & \textcolor{red}{0.999 (0.000)}\\
\hspace{1em}0.2 & NCTP & 0.437 (0.028) & 0.811 (0.003)\\
\hspace{1em}0.2 & CAFHT & 0.284 (0.003) & \textcolor{red}{0.832 (0.003)}\\
\hspace{1em}0.2 & CAFHT - PID & 0.316 (0.018) & \textcolor{red}{0.816 (0.003)}\\
\hspace{1em}0.2 & CAFHT (theory) & 0.342 (0.005) & \textcolor{red}{0.896 (0.002)}\\
\hspace{1em}0.2 & CAFHT (theory) - PID & 0.411 (0.008) & \textcolor{red}{0.896 (0.003)}\\
\bottomrule
\end{tabular}}
\end{table}

\FloatBarrier

\subsubsection*{CAFHT | Additive Scores}
Finally, the results of CAFHT with additive conformity scores \eqref{eq:nonconf_scores} are presented.

\begin{figure}[!htb]
    \centering
    \includegraphics[width=\linewidth]{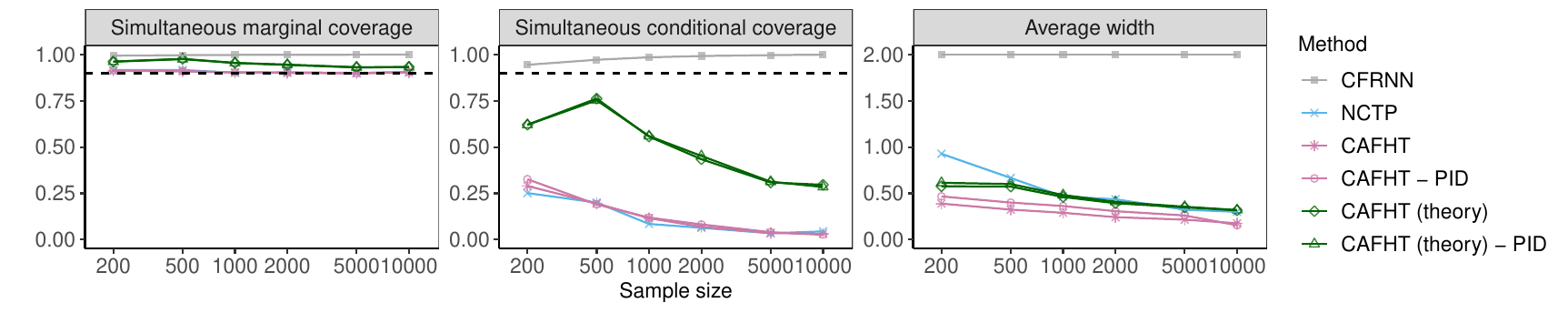}\vspace{-0.5cm}
    \caption{Performance on simulated heterogeneous trajectories of prediction bands constructed by different methods, as a function of the total number of training and calibration trajectories. See Table~\ref{tab:supp_exp_sim_static_ndata_fixed} for detailed results and standard errors.}
    \label{fig:supp_exp_sim_static_ndata_fixed}
\end{figure}

\begin{table}[!htb]
\centering
    \caption{Performance on simulated heterogeneous trajectories of prediction bands constructed by different methods, as a function of the total number of training and calibration trajectories. The red numbers indicate smaller prediction bands or higher conditional coverage. See the corresponding plot in Figure~\ref{fig:supp_exp_sim_static_ndata_fixed}.}
  \label{tab:supp_exp_sim_static_ndata_fixed}
  \centering
{\scriptsize
\begin{tabular}[t]{rlllll}
\toprule
\multicolumn{3}{c}{ } & \multicolumn{3}{c}{Simultaneous coverage} \\
\cmidrule(l{3pt}r{3pt}){4-6}
Sample size & Method & Average width & Conditional-hard & Conditional-easy & Marginal\\
\midrule
\addlinespace[0.3em]
\multicolumn{6}{l}{\textbf{200}}\\
\hspace{1em}200 & CFRNN & 2.000 (0.000) & \textcolor{red}{0.945 (0.006)} & 1.000 (0.000) & 0.994 (0.001)\\
\hspace{1em}200 & NCTP & 0.927 (0.048) & 0.251 (0.024) & 0.990 (0.002) & 0.916 (0.004)\\
\hspace{1em}200 & CAFHT & \textcolor{red}{0.388 (0.018)} & \textcolor{red}{0.290 (0.030)} & 0.982 (0.004) & 0.913 (0.006)\\
\hspace{1em}200 & CAFHT - PID & 0.466 (0.026) & \textcolor{red}{0.326 (0.031)} & 0.980 (0.004) & 0.915 (0.006)\\
\hspace{1em}200 & CAFHT (theory) & 0.576 (0.015) & \textcolor{red}{0.621 (0.023)} & 1.000 (0.000) & 0.963 (0.002)\\
\bottomrule
\end{tabular}}
\end{table}

\begin{figure}[!htb]
    \centering
    \includegraphics[width=\linewidth]{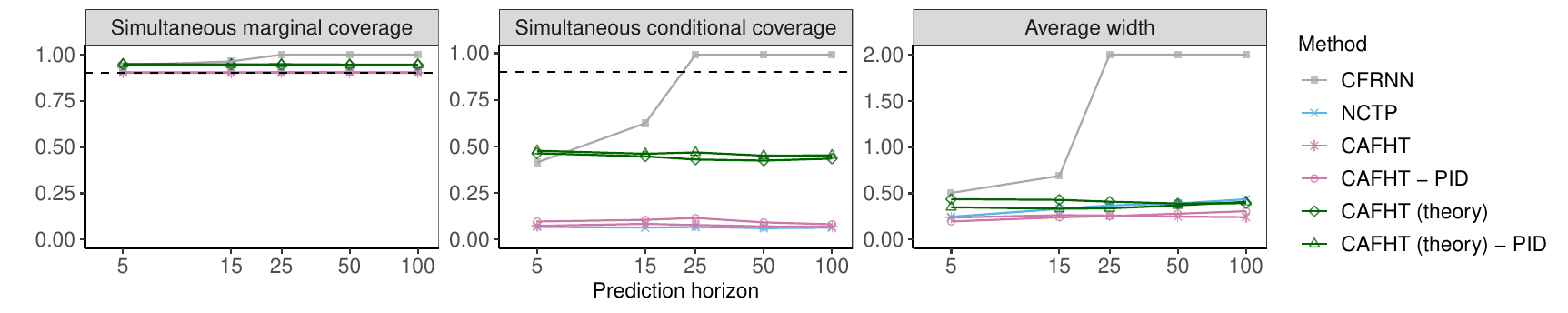}\vspace{-0.5cm}
    \caption{Performance on simulated heterogeneous trajectories of prediction bands constructed by different methods, as a function of the prediction horizon. See Table~\ref{tab:supp_exp_sim_static_horizon_fixed} for detailed results and standard errors.}
    \label{fig:supp_exp_sim_static_horizon_fixed}
\end{figure}

\begin{table}[!htb]
\centering
    \caption{Performance on simulated heterogeneous trajectories of prediction bands constructed by different methods, as a function of the prediction horizon. The red numbers indicate smaller prediction bands or higher conditional coverage. See the corresponding plot in Figure~\ref{fig:supp_exp_sim_static_horizon_fixed}.}
  \label{tab:supp_exp_sim_static_horizon_fixed}
  \centering
{\scriptsize
\begin{tabular}[t]{rlllll}
\toprule
\multicolumn{3}{c}{ } & \multicolumn{3}{c}{Simultaneous coverage} \\
\cmidrule(l{3pt}r{3pt}){4-6}
Prediction horizon & Method & Average width & Conditional-hard & Conditional-easy & Marginal\\
\midrule
\addlinespace[0.3em]
\multicolumn{6}{l}{\textbf{5}}\\
\hspace{1em}5 & CFRNN & 0.504 (0.009) & \textcolor{red}{0.414 (0.010)} & 1.000 (0.000) & 0.942 (0.001)\\
\hspace{1em}5 & NCTP & 0.246 (0.012) & 0.067 (0.009) & 0.996 (0.001) & 0.904 (0.002)\\
\hspace{1em}5 & CAFHT & 0.237 (0.008) & 0.073 (0.009) & 0.993 (0.001) & 0.902 (0.002)\\
\hspace{1em}5 & CAFHT - PID & \textcolor{red}{0.195 (0.008)} & \textcolor{red}{0.096 (0.010)} & 0.990 (0.001) & 0.902 (0.002)\\
\hspace{1em}5 & CAFHT (theory) & 0.436 (0.008) & \textcolor{red}{0.463 (0.012)} & 1.000 (0.000) & 0.947 (0.002)\\
\bottomrule
\end{tabular}}
\end{table}

\begin{figure}[!htb]
    \centering
    \includegraphics[width=\linewidth]{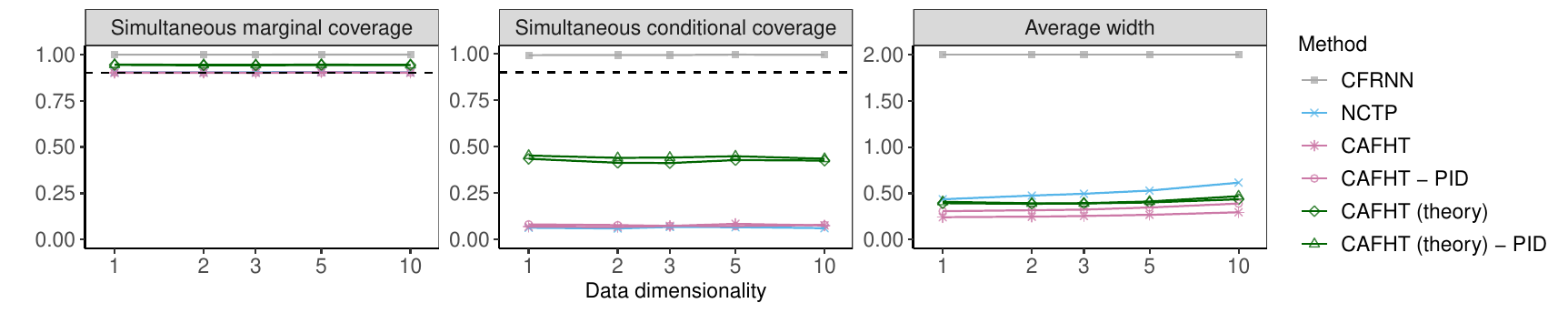}\vspace{-0.5cm}
    \caption{Performance on simulated heterogeneous trajectories of prediction bands constructed by different methods, as a function of the data dimensionality. See Table~\ref{tab:supp_exp_sim_static_ndim_fixed} for detailed results and standard errors.}
    \label{fig:supp_exp_sim_static_ndim_fixed}
\end{figure}

\begin{table}[!htb]
\centering
    \caption{Performance on simulated heterogeneous trajectories of prediction bands constructed by different methods, as a function of the data dimensionality. The red numbers indicate smaller prediction bands or higher conditional coverage. See the corresponding plot in Figure~\ref{fig:supp_exp_sim_static_ndim_fixed}.}
  \label{tab:supp_exp_sim_static_ndim_fixed}
  \centering
{\scriptsize
\begin{tabular}[t]{rlllll}
\toprule
\multicolumn{3}{c}{ } & \multicolumn{3}{c}{Simultaneous coverage} \\
\cmidrule(l{3pt}r{3pt}){4-6}
Data dimensionality & Method & Average width & Conditional-hard & Conditional-easy & Marginal\\
\midrule
\addlinespace[0.3em]
\multicolumn{6}{l}{\textbf{1}}\\
\hspace{1em}1 & CFRNN & 2.000 (0.000) & \textcolor{red}{0.993 (0.001)} & 1.000 (0.000) & 0.999 (0.000)\\
\hspace{1em}1 & NCTP & 0.435 (0.024) & 0.063 (0.008) & 0.996 (0.001) & 0.904 (0.002)\\
\hspace{1em}1 & CAFHT & \textcolor{red}{0.241 (0.008)} & 0.069 (0.010) & 0.993 (0.001) & 0.901 (0.002)\\
\hspace{1em}1 & CAFHT - PID & 0.306 (0.019) & \textcolor{red}{0.081 (0.012)} & 0.994 (0.001) & 0.904 (0.002)\\
\hspace{1em}1 & CAFHT (theory) & 0.393 (0.004) & \textcolor{red}{0.434 (0.011)} & 1.000 (0.000) & 0.944 (0.001)\\
\bottomrule
\end{tabular}}
\end{table}

\begin{figure}[!htb]
    \centering
    \includegraphics[width=\linewidth]{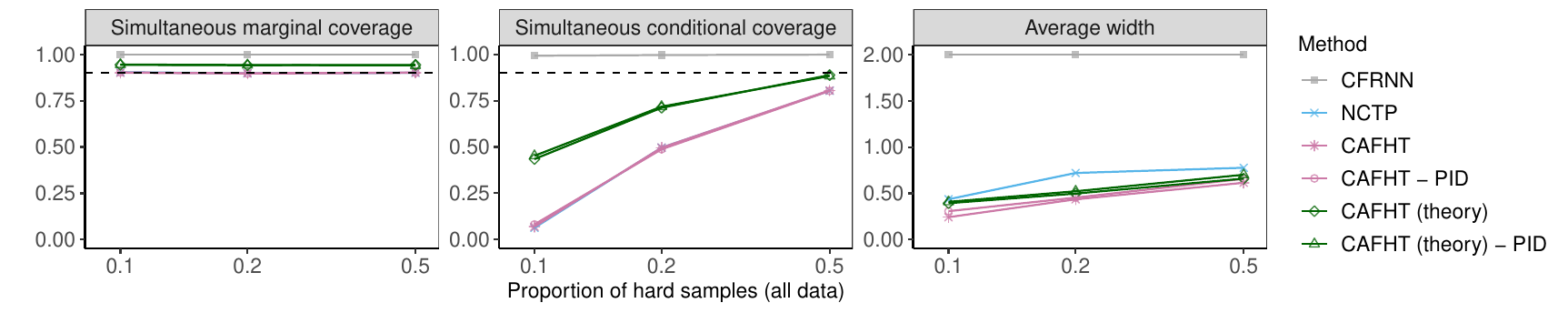}\vspace{-0.5cm}
    \caption{Performance on simulated heterogeneous trajectories of prediction bands constructed by different methods, as a function of the proportion of hard-to-predict trajectories. See Table~\ref{tab:supp_exp_sim_static_delta_fixed} for detailed results and standard errors.}
    \label{fig:supp_exp_sim_static_delta_fixed}
\end{figure}

\begin{table}[!htb]
\centering
    \caption{Performance on simulated heterogeneous trajectories of prediction bands constructed by different methods, as a function of the proportion of hard-to-predict trajectories. The red numbers indicate smaller prediction bands or higher conditional coverage. See the corresponding plot in Figure~\ref{fig:supp_exp_sim_static_delta_fixed}.}
  \label{tab:supp_exp_sim_static_delta_fixed}
  \centering
{\scriptsize
\begin{tabular}[t]{rlllll}
\toprule
\multicolumn{3}{c}{ } & \multicolumn{3}{c}{Simultaneous coverage} \\
\cmidrule(l{3pt}r{3pt}){4-6}
\makecell{Proportion of \\ hard samples \\(all data)} & Method & Average width & Conditional-hard & Conditional-easy & Marginal\\
\midrule
\addlinespace[0.3em]
\multicolumn{6}{l}{\textbf{0.1}}\\
\hspace{1em}0.1 & CFRNN & 2.000 (0.000) & \textcolor{red}{0.993 (0.001)} & 1.000 (0.000) & 0.999 (0.000)\\
\hspace{1em}0.1 & NCTP & 0.435 (0.024) & 0.063 (0.008) & 0.996 (0.001) & 0.904 (0.002)\\
\hspace{1em}0.1 & CAFHT & \textcolor{red}{0.241 (0.008)} & 0.069 (0.010) & 0.993 (0.001) & 0.901 (0.002)\\
\hspace{1em}0.1 & CAFHT - PID & 0.306 (0.019) & \textcolor{red}{0.081 (0.012)} & 0.994 (0.001) & 0.904 (0.002)\\
\hspace{1em}0.1 & CAFHT (theory) & 0.393 (0.004) & \textcolor{red}{0.434 (0.011)} & 1.000 (0.000) & 0.944 (0.001)\\
\bottomrule
\end{tabular}}
\end{table}

\begin{figure}[!htb]
    \centering
    \includegraphics[width=0.75\linewidth]{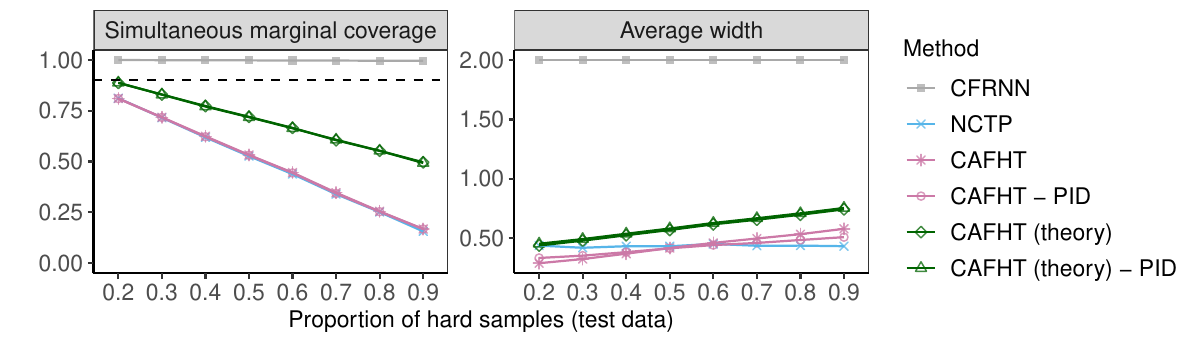}\vspace{-0.5cm}
    \caption{Performance on simulated heterogeneous trajectories of prediction bands constructed by different methods under distributional shift. The results are shown as a function of the proportion of hard-to-predict trajectories in the test data. See Table~\ref{tab:supp_exp_sim_static_delta_test_fixed} for detailed results and standard errors.}
    \label{fig:supp_exp_sim_static_delta_test_fixed}
\end{figure}

\begin{table}[!htb]
    \centering
    \caption{Performance on simulated heterogeneous trajectories of prediction bands constructed by different methods under distributional shift. The results are shown as a function of the proportion of hard-to-predict trajectories in the test data.  The red numbers indicate higher marginal coverage. See the corresponding plot in Figure~\ref{fig:supp_exp_sim_static_delta_test_fixed}.}
  \label{tab:supp_exp_sim_static_delta_test_fixed}
  \centering
{\scriptsize
\begin{tabular}[t]{rlll}
\toprule
Proportion of hard samples (test data) & Method & Length & Marginal coverage\\
\midrule
\addlinespace[0.3em]
\multicolumn{4}{l}{\textbf{0.2}}\\
\hspace{1em}0.2 & CFRNN & 2.000 (0.000) & \textcolor{red}{0.999 (0.000)}\\
\hspace{1em}0.2 & NCTP & 0.437 (0.028) & 0.811 (0.003)\\
\hspace{1em}0.2 & CAFHT & 0.289 (0.010) & 0.809 (0.003)\\
\hspace{1em}0.2 & CAFHT - PID & 0.333 (0.021) & 0.810 (0.004)\\
\hspace{1em}0.2 & CAFHT (theory) & 0.438 (0.005) & \textcolor{red}{0.886 (0.002)}\\
\hspace{1em}0.2 & CAFHT (theory) - PID & 0.453 (0.005) & \textcolor{red}{0.887 (0.003)}\\
\bottomrule
\end{tabular}}
\end{table}

\FloatBarrier

\subsection{Pedestrian Data}
In this subsection, we present the experimental results of the pedestrian data described in~\ref{sec:experiment}. Recall that we preprocess the dataset by adding heteroskedasticity such that $10\%$ of the data are designed to be hard-to-predict by adding a random noise follows $N(0,\sigma_t^2)$, where $\sigma_t^2 \propto t\cdot \text{noise level}$. The easy-to-predict data are added a random noise with $\sigma_t^2 \propto t$. By default, $10\%$ of the trajectories are set to be hard-to-predict.

Similar to the previous section, we first demonstrate the main result by using the CAFHT method with ACI and multiplicative scores as the main method to be compared with the benchmark methods CFRNN and NCTP. The results after adding the dynamic noise profile to the data are presented here for demonstrative purposes.

\subsubsection{Main Results | Comparing CFRNN, NCTP, and the Main Implementation of CAFHT}

Figure~\ref{fig:main_exp_real_noiselevel_dt01} and Table~\ref{tab:main_exp_real_noiselevel_dt01} show the average performance on pedestrian heterogeneous trajectories of prediction bands constructed by different methods, as a function of the noise level. The noise level is varied from 1.5 to 5. All methods achieve 90\% simultaneous marginal coverage. Our method (CAFHT) leads to more informative bands with lower average width and higher conditional coverage.

The results of another experiment in which $20\%$ of the trajectories are hard-to-predict are presented in Figure~\ref{fig:main_exp_real_noise_level} and Table~\ref{tab:main_exp_real_noise_level}. Again, we observe that even though a larger percentage of hard trajectories on the pedestrian data can increase the empirical conditional coverage of all methods, CAFHT maintains clear advantages relative to the baselines.

Additionally, Figure~\ref{fig:main_exp_real_dynamic_multi_ndata} and Table~\ref{tab:main_exp_real_dynamic_multi_ndata} present results for varying numbers of trajectories in the training and calibration sets, from 200 to 1000, with the noise level set at $3$ and the percentage of hard trajectories set to $10\%$. Again, CAFHT outperforms the other benchmarks.

\begin{table}[!htb]
\centering
    \caption{Performance on heterogeneous pedestrian trajectories of conformal prediction bands constructed by different methods, as a function of the noise level. The red numbers indicate smaller prediction bands or higher conditional coverage. $10\%$ of the trajectories are set to be hard-to-predict. See the corresponding plot in Figure~\ref{fig:main_exp_real_noiselevel_dt01}.}
  \label{tab:main_exp_real_noiselevel_dt01}
  \centering
{\scriptsize
\begin{tabular}[t]{rlllll}
\toprule
\multicolumn{3}{c}{ } & \multicolumn{3}{c}{Coverage} \\
\cmidrule(l{3pt}r{3pt}){4-6}
Noise level for hard samples & Method & Length & Marginal & Conditional-easy & Conditional-hard\\
\midrule
\addlinespace[0.3em]
\multicolumn{6}{l}{\textbf{1}}\\
\hspace{1em}1.0 & CFRNN & 2.000 (0.000) & 1.000 (0.000) & 1.000 (0.000) & \textcolor{red}{1.000 (0.000)}\\
\hspace{1em}1.0 & NCTP & \textcolor{red}{0.172 (0.001)} & 0.898 (0.003) & 0.898 (0.003) & 0.901 (0.006)\\
\hspace{1em}1.0 & CAFHT & 0.201 (0.001) & 0.902 (0.003) & 0.901 (0.003) & \textcolor{red}{0.912 (0.005)}\\
\addlinespace[0.3em]
\multicolumn{6}{l}{\textbf{1.5}}\\
\hspace{1em}1.5 & CFRNN & 2.000 (0.000) & 1.000 (0.000) & 1.000 (0.000) & \textcolor{red}{1.000 (0.000)}\\
\bottomrule
\end{tabular}}
\end{table}

\begin{figure}[!htb]
    \centering
    \includegraphics[width=\linewidth]{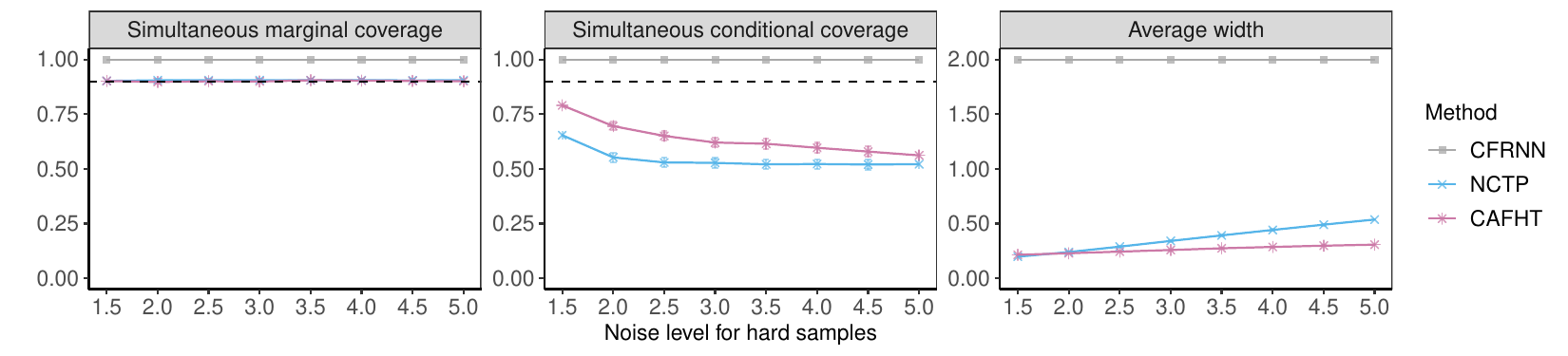}\vspace{-0.5cm}
    \caption{Performance on heterogeneous pedestrian trajectories of conformal prediction bands constructed by different methods, as a function of the noise level. $20\%$ of the trajectories are set to be hard-to-predict.}
    \label{fig:main_exp_real_noise_level}
\end{figure}

\begin{table}[!htb]
\centering
    \caption{Performance on heterogeneous pedestrian trajectories of conformal prediction bands constructed by different methods, as a function of the noise level. The red numbers indicate smaller prediction bands or higher conditional coverage. $20\%$ of the trajectories are set to be hard-to-predict. See the corresponding plot in Figure~\ref{fig:main_exp_real_noise_level}.}
  \label{tab:main_exp_real_noise_level}
  \centering
{\scriptsize
\begin{tabular}[t]{rlllll}
\toprule
\multicolumn{3}{c}{ } & \multicolumn{3}{c}{Coverage} \\
\cmidrule(l{3pt}r{3pt}){4-6}
Noise level for hard samples & Method & Length & Marginal & Conditional-easy & Conditional-hard\\
\midrule
\addlinespace[0.3em]
\multicolumn{6}{l}{\textbf{1.5}}\\
\hspace{1em}1.5 & CFRNN & 2.000 (0.000) & 1.000 (0.000) & 1.000 (0.000) & \textcolor{red}{1.000 (0.000)}\\
\hspace{1em}1.5 & NCTP & \textcolor{red}{0.197 (0.002)} & 0.902 (0.003) & 0.963 (0.002) & 0.654 (0.010)\\
\hspace{1em}1.5 & CAFHT & 0.214 (0.001) & 0.902 (0.003) & 0.930 (0.003) & \textcolor{red}{0.790 (0.007)}\\
\addlinespace[0.3em]
\multicolumn{6}{l}{\textbf{2}}\\
\hspace{1em}2.0 & CFRNN & 2.000 (0.000) & 1.000 (0.000) & 1.000 (0.000) & \textcolor{red}{1.000 (0.000)}\\
\bottomrule
\end{tabular}}
\end{table}

\begin{figure}[!htb]
    \centering
    \includegraphics[width=\linewidth]{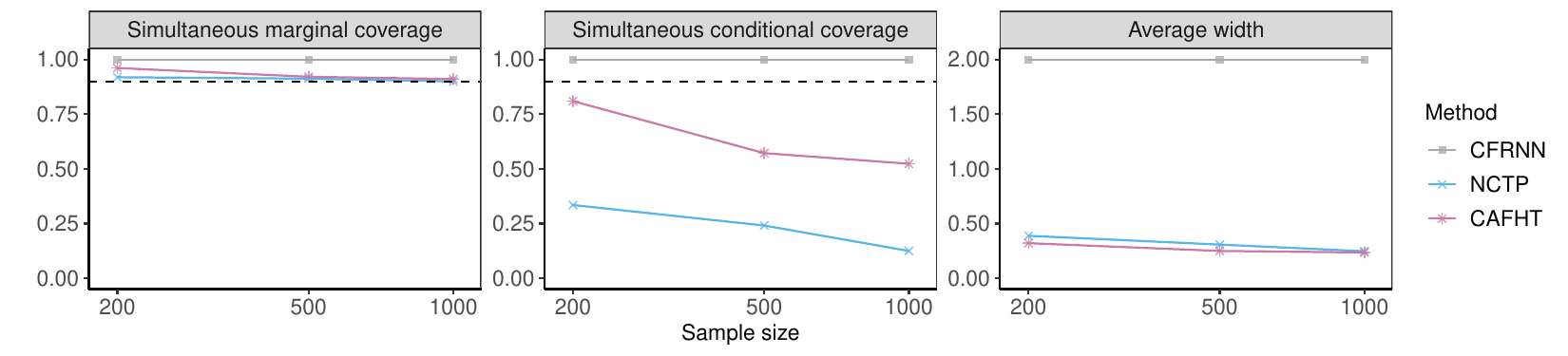}\vspace{-0.5cm}
    \caption{Performance on heterogeneous pedestrian trajectories of conformal prediction bands constructed by different methods, as a function of the total number of training and calibration trajectories.}
    \label{fig:main_exp_real_dynamic_multi_ndata}
\end{figure}

\begin{table}[!htb]
\centering
    \caption{Performance on heterogeneous pedestrian trajectories of conformal prediction bands constructed by different methods, as a function of the total number of training and calibration trajectories. The red numbers indicate smaller prediction bands or higher conditional coverage. See the corresponding plot in Figure~\ref{fig:main_exp_real_dynamic_multi_ndata}.}
  \label{tab:main_exp_real_dynamic_multi_ndata}
  \centering
{\scriptsize
\begin{tabular}[t]{rlllll}
\toprule
\multicolumn{3}{c}{ } & \multicolumn{3}{c}{Coverage} \\
\cmidrule(l{3pt}r{3pt}){4-6}
Sample size & Method & Length & Marginal & Conditional-easy & Conditional-hard\\
\midrule
\addlinespace[0.3em]
\multicolumn{6}{l}{\textbf{200}}\\
\hspace{1em}200 & CFRNN & 2.000 (0.000) & 1.000 (0.000) & 1.000 (0.000) & \textcolor{red}{1.000 (0.000)}\\
\hspace{1em}200 & NCTP & 0.388 (0.009) & 0.919 (0.004) & 0.986 (0.002) & 0.335 (0.023)\\
\hspace{1em}200 & CAFHT & \textcolor{red}{0.321 (0.006)} & 0.961 (0.004) & 0.979 (0.003) & \textcolor{red}{0.810 (0.018)}\\
\addlinespace[0.3em]
\multicolumn{6}{l}{\textbf{500}}\\
\hspace{1em}500 & CFRNN & 2.000 (0.000) & 1.000 (0.000) & 1.000 (0.000) & \textcolor{red}{1.000 (0.000)}\\
\hspace{1em}500 & NCTP & 0.308 (0.006) & 0.912 (0.003) & 0.989 (0.001) & 0.241 (0.017)\\
\hspace{1em}500 & CAFHT & \textcolor{red}{0.249 (0.002)} & 0.921 (0.004) & 0.961 (0.003) & \textcolor{red}{0.572 (0.016)}\\
\addlinespace[0.3em]
\multicolumn{6}{l}{\textbf{1000}}\\
\hspace{1em}1000 & CFRNN & 2.000 (0.000) & 1.000 (0.000) & 1.000 (0.000) & \textcolor{red}{1.000 (0.000)}\\
\hspace{1em}1000 & NCTP & 0.247 (0.004) & 0.900 (0.003) & 0.989 (0.001) & 0.125 (0.012)\\
\hspace{1em}1000 & CAFHT & \textcolor{red}{0.236 (0.002)} & 0.911 (0.003) & 0.955 (0.002) & \textcolor{red}{0.523 (0.014)}\\
\bottomrule
\end{tabular}}
\end{table}

\FloatBarrier

\subsubsection{Supplementary Results | Comparing Different CAFHT Implementations}
\subsubsection*{CAFHT - multiplicative scores}
The results of CAFHT with multiplicative conformity scores \eqref{eq:nonconf_scores_adap} are first presented in Figures~\ref{fig:supp_exp_real_ndata}--\ref{fig:supp_exp_real_noise_level} and Tables~\ref{tab:supp_exp_real_ndata}--\ref{tab:supp_exp_real_noise_level}.

\begin{figure}[!htb]
    \centering
    \includegraphics[width=\linewidth]{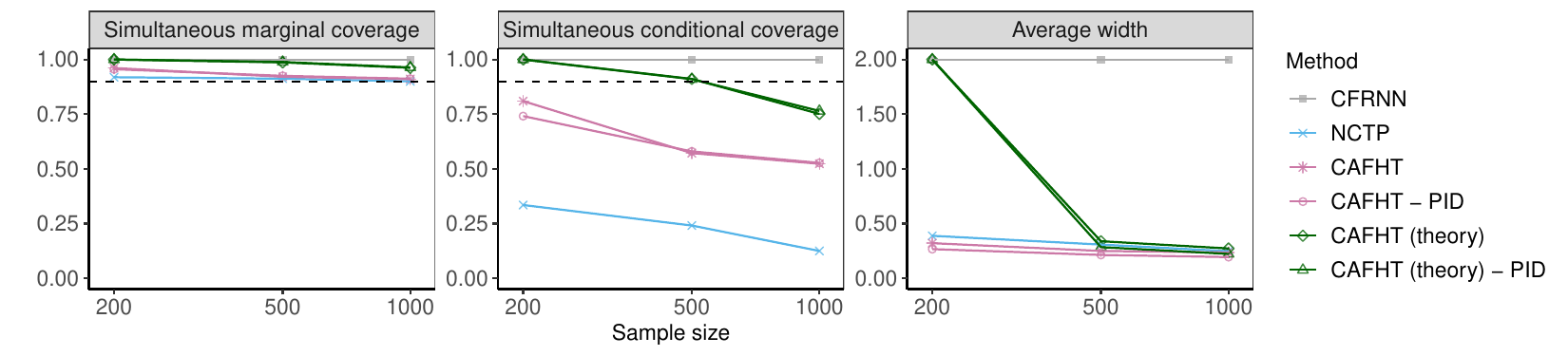}\vspace{-0.5cm}
    \caption{Performance on heterogeneous pedestrian trajectories of conformal prediction bands constructed by different methods, as a function of the total number of training and calibration trajectories.}
    \label{fig:supp_exp_real_ndata}
\end{figure}

\begin{table}[!htb]
\centering
    \caption{Performance on heterogeneous pedestrian trajectories of conformal prediction bands constructed by different methods, as a function of the total number of training and calibration trajectories. The red numbers indicate smaller prediction bands or higher conditional coverage. See the corresponding plot in Figure~\ref{fig:supp_exp_real_ndata}.}
  \label{tab:supp_exp_real_ndata}
  \centering
{\scriptsize
\begin{tabular}[t]{rlllll}
\toprule
\multicolumn{3}{c}{ } & \multicolumn{3}{c}{Coverage} \\
\cmidrule(l{3pt}r{3pt}){4-6}
Sample size & Method & Length & Marginal & Conditional-easy & Conditional-hard\\
\midrule
\addlinespace[0.3em]
\multicolumn{6}{l}{\textbf{200}}\\
\hspace{1em}200 & CFRNN & 2.000 (0.000) & 1.000 (0.000) & 1.000 (0.000) & \textcolor{red}{1.000 (0.000)}\\
\hspace{1em}200 & NCTP & 0.388 (0.009) & 0.919 (0.004) & 0.986 (0.002) & 0.335 (0.023)\\
\hspace{1em}200 & CAFHT & 0.321 (0.006) & 0.961 (0.004) & 0.979 (0.003) & \textcolor{red}{0.810 (0.018)}\\
\hspace{1em}200 & CAFHT - PID & \textcolor{red}{0.266 (0.005)} & 0.956 (0.005) & 0.980 (0.003) & \textcolor{red}{0.741 (0.022)}\\
\hspace{1em}200 & CAFHT (theory) & 2.000 (0.000) & 1.000 (0.000) & 1.000 (0.000) & \textcolor{red}{1.000 (0.000)}\\
\bottomrule
\end{tabular}}
\end{table}

\begin{figure}[!htb]
    \centering
    \includegraphics[width=\linewidth]{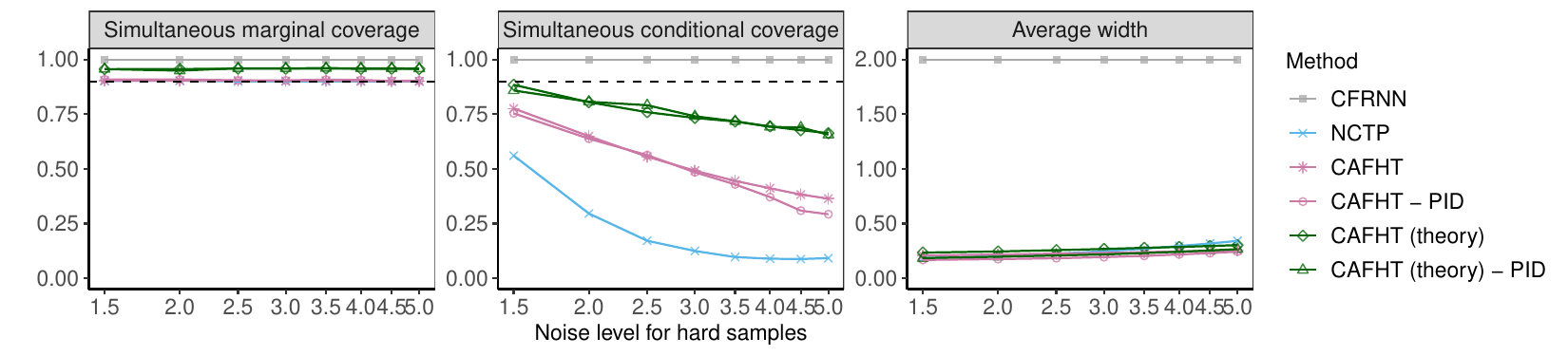}\vspace{-0.5cm}
    \caption{Performance on heterogeneous pedestrian trajectories of conformal prediction bands constructed by different methods, as a function of the noise level.}
    \label{fig:supp_exp_real_noise_level}
\end{figure}

\begin{table}[!htb]
\centering
    \caption{Performance on heterogeneous pedestrian trajectories of conformal prediction bands constructed by different methods, as a function of the noise level. The red numbers indicate smaller prediction bands or higher conditional coverage. See the corresponding plot in Figure~\ref{fig:supp_exp_real_noise_level}.}
  \label{tab:supp_exp_real_noise_level}
  \centering
{\scriptsize
\begin{tabular}[t]{rlllll}
\toprule
\multicolumn{3}{c}{ } & \multicolumn{3}{c}{Coverage} \\
\cmidrule(l{3pt}r{3pt}){4-6}
Noise level for hard samples & Method & Length & Marginal & Conditional-easy & Conditional-hard\\
\midrule
\addlinespace[0.3em]
\multicolumn{6}{l}{\textbf{1}}\\
\hspace{1em}1.0 & CFRNN & 2.000 (0.000) & 1.000 (0.000) & 1.000 (0.000) & \textcolor{red}{1.000 (0.000)}\\
\hspace{1em}1.0 & NCTP & 0.172 (0.001) & 0.898 (0.003) & 0.898 (0.003) & 0.901 (0.006)\\
\hspace{1em}1.0 & CAFHT & 0.201 (0.001) & 0.902 (0.003) & 0.901 (0.003) & \textcolor{red}{0.912 (0.005)}\\
\hspace{1em}1.0 & CAFHT - PID & \textcolor{red}{0.159 (0.001)} & 0.906 (0.003) & 0.906 (0.003) & \textcolor{red}{0.909 (0.005)}\\
\hspace{1em}1.0 & CAFHT (theory) & 0.226 (0.001) & 0.955 (0.002) & 0.955 (0.002) & \textcolor{red}{0.959 (0.004)}\\
\bottomrule
\end{tabular}}
\end{table}

\FloatBarrier

\subsubsection*{CAFHT | Additive Scores}
The results of CAFHT with additive conformity scores \eqref{eq:nonconf_scores} are presented in Figures~\ref{fig:supp_exp_real_ndata_fixed}--\ref{fig:supp_exp_real_noise_level_fixed} and Tables~\ref{tab:supp_exp_real_ndata_fixed}--\ref{tab:supp_exp_real_noise_level_fixed}.

\begin{figure}[!htb]
    \centering
    \includegraphics[width=\linewidth]{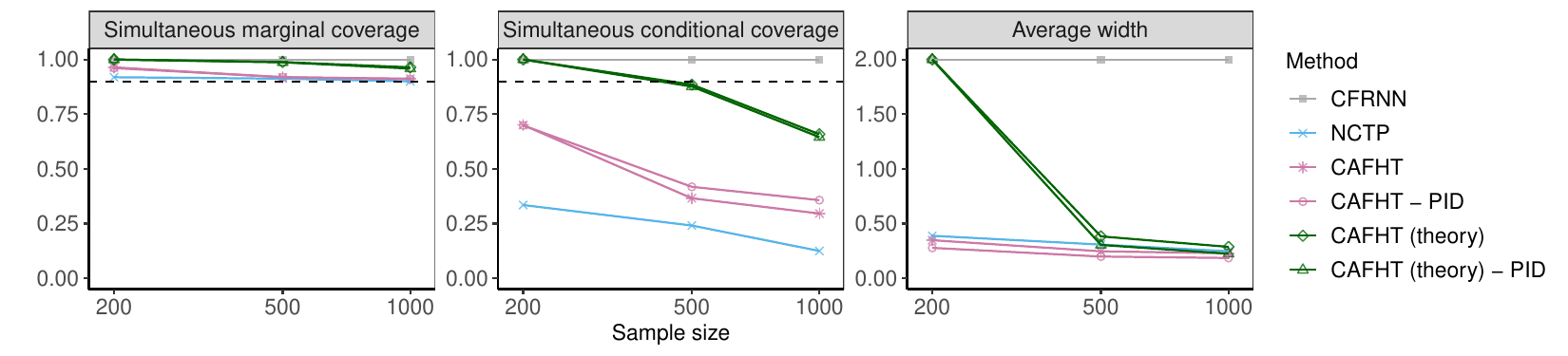}\vspace{-0.5cm}
    \caption{Performance on heterogeneous pedestrian trajectories of conformal prediction bands constructed by different methods, as a function of the total number of training and calibration trajectories.}
    \label{fig:supp_exp_real_ndata_fixed}
\end{figure}

\begin{table}[!htb]
\centering
    \caption{Performance on heterogeneous pedestrian trajectories of conformal prediction bands constructed by different methods, as a function of the total number of training and calibration trajectories. The red numbers indicate smaller prediction bands or higher conditional coverage. See the corresponding plot in Figure~\ref{fig:supp_exp_real_ndata_fixed}.}
  \label{tab:supp_exp_real_ndata_fixed}
  \centering
{\scriptsize
\begin{tabular}[t]{rlllll}
\toprule
\multicolumn{3}{c}{ } & \multicolumn{3}{c}{Coverage} \\
\cmidrule(l{3pt}r{3pt}){4-6}
Sample size & Method & Length & Marginal & Conditional-easy & Conditional-hard\\
\midrule
\addlinespace[0.3em]
\multicolumn{6}{l}{\textbf{200}}\\
\hspace{1em}200 & CFRNN & 2.000 (0.000) & 1.000 (0.000) & 1.000 (0.000) & \textcolor{red}{1.000 (0.000)}\\
\hspace{1em}200 & NCTP & 0.388 (0.009) & 0.919 (0.004) & 0.986 (0.002) & 0.335 (0.023)\\
\hspace{1em}200 & CAFHT & 0.348 (0.009) & 0.964 (0.003) & 0.994 (0.001) & \textcolor{red}{0.700 (0.024)}\\
\hspace{1em}200 & CAFHT - PID & \textcolor{red}{0.278 (0.008)} & 0.961 (0.004) & 0.991 (0.002) & \textcolor{red}{0.699 (0.025)}\\
\hspace{1em}200 & CAFHT (theory) & 2.000 (0.000) & 1.000 (0.000) & 1.000 (0.000) & \textcolor{red}{1.000 (0.000)}\\
\bottomrule
\end{tabular}}
\end{table}

\begin{figure}[!htb]
    \centering
    \includegraphics[width=\linewidth]{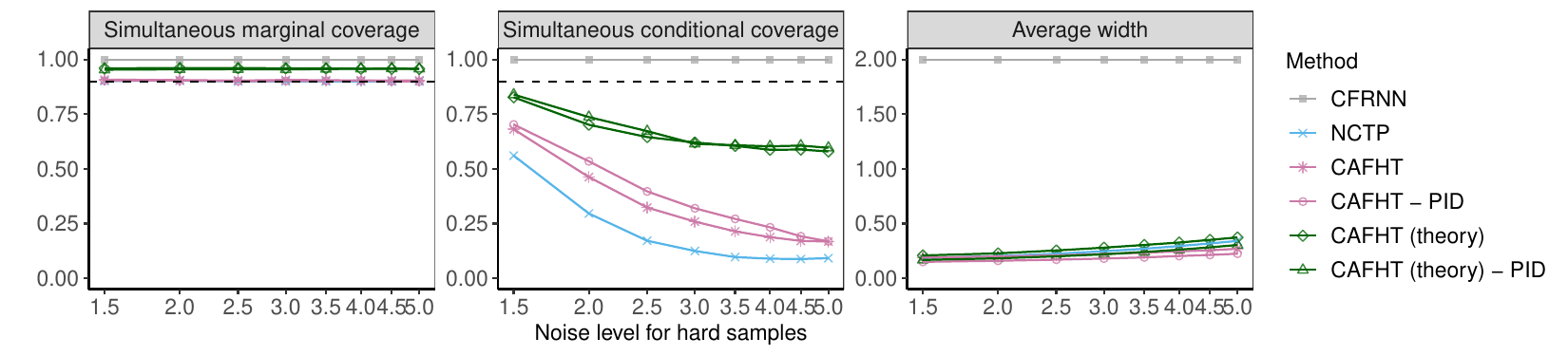}\vspace{-0.5cm}
    \caption{Performance on heterogeneous pedestrian trajectories of conformal prediction bands constructed by different methods, as a function of the noise level.}
    \label{fig:supp_exp_real_noise_level_fixed}
\end{figure}

\begin{table}[!htb]
\centering
    \caption{Performance on heterogeneous pedestrian trajectories of conformal prediction bands constructed by different methods, as a function of the noise level. The red numbers indicate smaller prediction bands or higher conditional coverage. See the corresponding plot in Figure~\ref{fig:supp_exp_real_noise_level_fixed}.}
  \label{tab:supp_exp_real_noise_level_fixed}
  \centering
{\scriptsize
\begin{tabular}[t]{rlllll}
\toprule
\multicolumn{3}{c}{ } & \multicolumn{3}{c}{Coverage} \\
\cmidrule(l{3pt}r{3pt}){4-6}
Noise level for hard samples & Method & Length & Marginal & Conditional-easy & Conditional-hard\\
\midrule
\addlinespace[0.3em]
\multicolumn{6}{l}{\textbf{1}}\\
\hspace{1em}1.0 & CFRNN & 2.000 (0.000) & 1.000 (0.000) & 1.000 (0.000) & \textcolor{red}{1.000 (0.000)}\\
\hspace{1em}1.0 & NCTP & 0.172 (0.001) & 0.898 (0.003) & 0.898 (0.003) & 0.901 (0.006)\\
\hspace{1em}1.0 & CAFHT & 0.180 (0.001) & 0.907 (0.003) & 0.907 (0.004) & \textcolor{red}{0.907 (0.005)}\\
\hspace{1em}1.0 & CAFHT - PID & \textcolor{red}{0.141 (0.001)} & 0.901 (0.003) & 0.901 (0.003) & 0.901 (0.006)\\
\hspace{1em}1.0 & CAFHT (theory) & 0.197 (0.001) & 0.958 (0.002) & 0.958 (0.002) & \textcolor{red}{0.961 (0.003)}\\
\bottomrule
\end{tabular}}
\end{table}

\FloatBarrier

\subsection{Comparing ACI and CAFHT}\label{app:more_experiments_aci_vs_cafht}

As previously explained, the objectives of CAFHT and ACI are very different. CAFHT leverages information from multiple exchangeable trajectories to construct prediction bands for a trajectories from the same population, ensuring simultaneous coverage as per Equation~\eqref{eq:simu_coverage}.
By contrast, ACI constructs an online prediction band for a single trajectory, aiming to achieve long-term average coverage.

Consider a motion planning scenario: CAFHT's objective is to maintain most vehicles within their predicted zones throughout a specified period, ensuring a high probability of reaching their destinations without incident. On the other hand, ACI aims for asymptotic average coverage, which tolerates frequent, albeit temporary, deviations from the predicted path for each vehicle. In practical terms, this means each vehicle might exit and re-enter the ACI-predicted region numerous times, spending about 90\% of the time within the prediction band on average. If exiting these regions could lead to severe accidents, CAFHT's approach would ensure that 90\% (or any pre-specified percentage) of vehicles safely arrive at their destinations, whereas ACI's approach could potentially result in none of the vehicles reaching their destinations safely.

This concept is demonstrated in Figure~\ref{fig:aci_vs_cafht}, which contrasts the prediction bands created using ACI and CAFHT for two pedestrian trajectories. The figure clearly shows that ACI does not fully encompass the trajectories, thus failing to meet our objective of achieving simultaneous coverage.

\begin{figure}[!htb]
    \centering
    \includegraphics[width=0.7\linewidth]{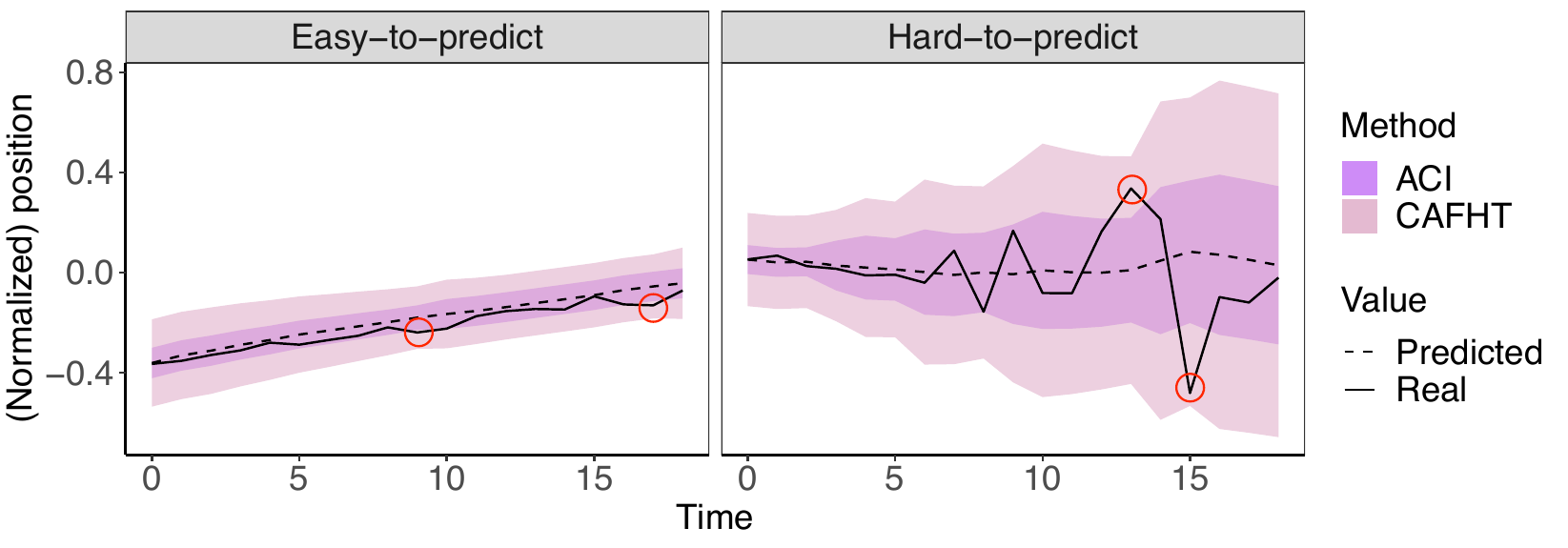}\vspace{-0.5cm}
    \caption{Forecasting bands constructed using ACI and CAFHT, for the heterogeneous pedestrian trajectories. Red circles indicate scenarios where the real values exceed ACI prediction bands.}
    \label{fig:aci_vs_cafht}
\end{figure}

Figure~\ref{fig:main_aci_vs_cafht} and Table~\ref{tab:main_aci_vs_cafht} provide additional insight, reporting on experiments that replicate the analysis from Figure~\ref{fig:main_exp_sim_ndata} but include results from ACI. Unlike CAFHT and the two other benchmark methods, ACI is unable to meet the simultaneous marginal coverage guarantee.

\begin{figure}[!htb]
    \centering
    \includegraphics[width=\linewidth]{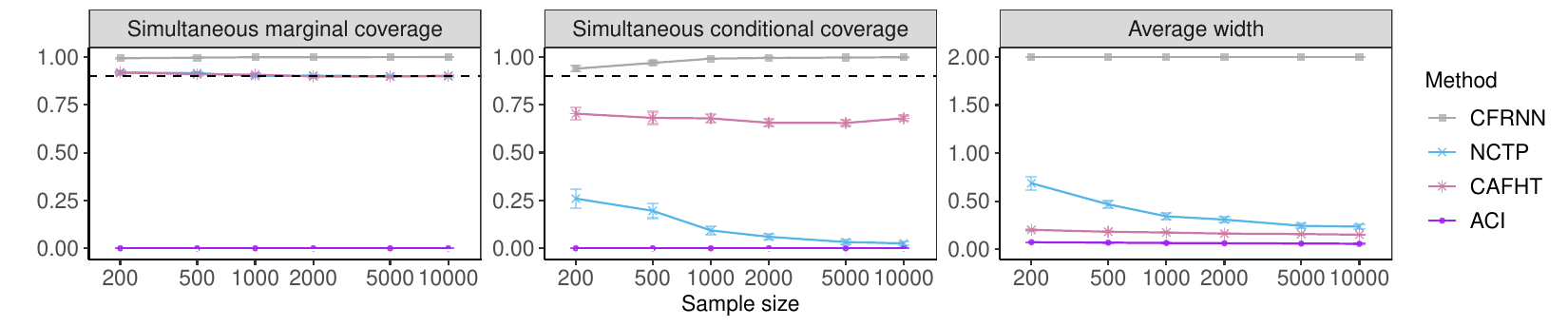}\vspace{-0.5cm}
    \caption{Performance on simulated heterogeneous trajectories of prediction bands constructed by different methods, as a function of the total number of training and calibration trajectories.}
    \label{fig:main_aci_vs_cafht}
\end{figure}

\begin{table}[!htb]
\centering
    \caption{Performance on heterogeneous pedestrian trajectories of conformal prediction bands constructed by different methods, as a function of the noise level. The red numbers indicate smaller prediction bands or higher conditional coverage. See the corresponding plot in Figure~\ref{fig:main_aci_vs_cafht}.}
  \label{tab:main_aci_vs_cafht}
  \centering
{\scriptsize
\begin{tabular}[t]{rlllll}
\toprule
\multicolumn{3}{c}{ } & \multicolumn{3}{c}{Simultaneous coverage} \\
\cmidrule(l{3pt}r{3pt}){4-6}
Sample size & Method & Average width & Conditional-hard & Conditional-easy & Marginal\\
\midrule
\addlinespace[0.3em]
\multicolumn{6}{l}{\textbf{200}}\\
\hspace{1em}200 & CFRNN & 2.000 (0.000) & \textcolor{red}{0.939 (0.007)} & 1.000 (0.000) & 0.994 (0.001)\\
\hspace{1em}200 & NCTP & 0.687 (0.035) & \textcolor{red}{0.260 (0.024)} & 0.992 (0.002) & 0.919 (0.004)\\
\hspace{1em}200 & CAFHT & 0.202 (0.003) & \textcolor{red}{0.704 (0.017)} & 0.944 (0.005) & 0.920 (0.005)\\
\hspace{1em}200 & ACI & \textcolor{red}{0.073 (0.001)} & \textcolor{red}{0.000 (0.000)} & 0.000 (0.000) & 0.000 (0.000)\\
\addlinespace[0.3em]
\multicolumn{6}{l}{\textbf{500}}\\
\bottomrule
\end{tabular}}
\end{table}

\FloatBarrier
\subsection{Comparing the Multiplicative Scores and the Additive Scores}\label{app:more_experiments_multi_vs_add}

The CAFHT prediction bands are constructed in two stages: initially, the underlying ACI bands are established, followed by adding a conformalized correction term. When corrections employ the additive nonconformity scores specified in Equation~\eqref{eq:nonconf_score_msteps_fixed}, the heterogeneity of the trajectories is managed exclusively via the ACI bands. In contrast, using the multiplicative scores from Equation~\eqref{eq:nonconf_score_msteps_multi} allows both components to adapt to heteroscedasticity, though the primary adjustment is through ACI.

More precisely, multiplicative scores impose proportionally wider margins of error on broader ACI intervals than on narrower ones. Hence, while adjusting for heteroscedasticity is chiefly the responsibility of ACI, the use of multiplicative scores arises from the recognition that ACI residuals might still display heteroscedastic traits. In such instances, multiplicative scores are better suited to capturing this variability than their additive counterparts.

As shown in Figure~\ref{fig:multi_vs_fixed}, additive scores impose a constant correction term (the empirical quantile $\hat{Q}$) on ACI intervals. In comparison, multiplicative scores adjust the ACI bands by a non-constant amount (the empirical quantile $\hat{Q}$ multiplied by the size of the ACI bands).

In line with established conformal inference methodologies, we prefer to delegate the more complex ``adaptability'' functions to the underlying machine learning model (in this case, the forecaster integrated with ACI). The next phase of conformalization simply involves a clear, straightforward adjustment to secure the simultaneous marginal coverage guarantee. Nonetheless, future developments might introduce more intricate scoring designs, potentially enhancing empirical performance but at the expense of simplicity in the methodology.

\begin{figure}[!htb]
    \centering
    \includegraphics[width=0.85\linewidth]{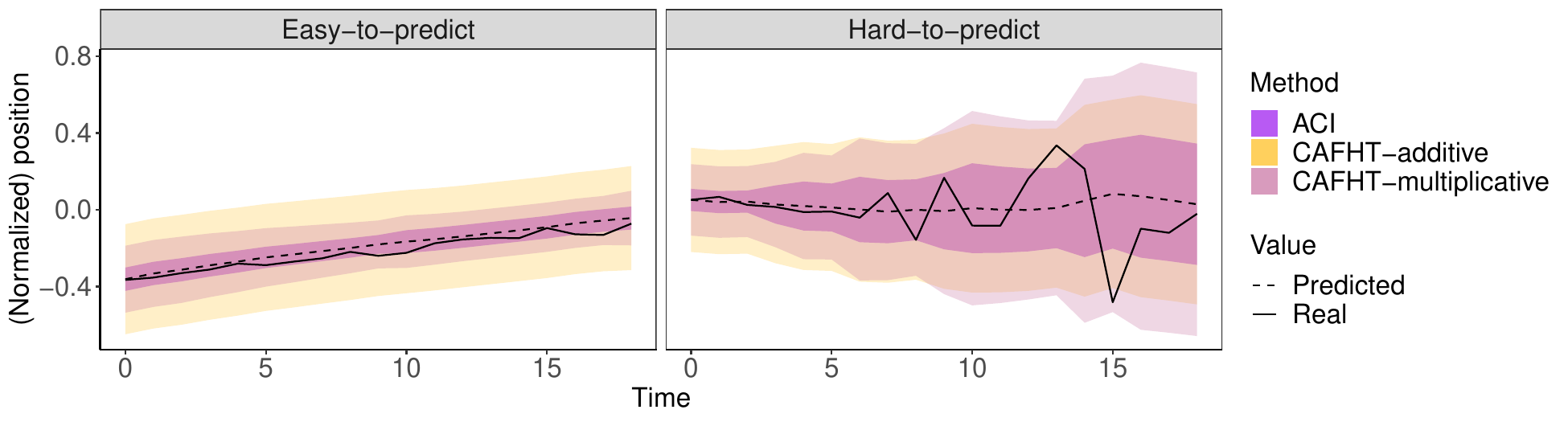}\vspace{-0.5cm}
    \caption{Forecasting bands constructed using ACI and CAFHT, for the heterogeneous pedestrian trajectories. Red circles indicate scenarios where the real values exceed ACI prediction bands.}
    \label{fig:multi_vs_fixed}
\end{figure}

Figure~\ref{fig:multi_vs_fixed_full} presents a side-by-side comparison of CAFHT using multiplicative scores, CAFHT using additive scores, NCTP, and CFRNN for two example heterogeneous pedestrian trajectories. The plot demonstrates CAFHT, with both scoring approaches, effectively manages heterogeneity, though the multiplicative scores offer superior adaptability. In contrast, NCTP and CFRNN do not adjust to heterogeneity. The empirical quantile $\hat{Q}$ for this experiment is recorded in Table~\ref{tab:multi_vs_fixed_full}.

\begin{figure}[!htb]
    \centering
    \includegraphics[width=0.85\linewidth]{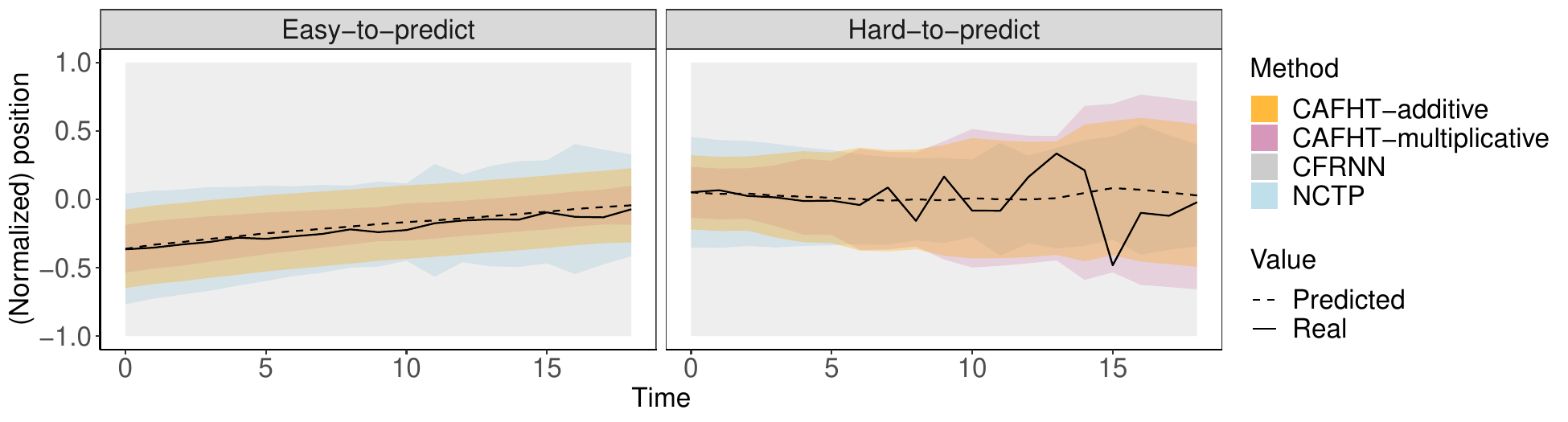}\vspace{-0.5cm}
    \caption{Forecasting bands constructed using different methods for the heterogeneous pedestrian trajectories.}
    \label{fig:multi_vs_fixed_full}
\end{figure}

\begin{table}[!htb]
\centering
    \caption{Empirical quantiles obtained from each method in Figure~\ref{fig:multi_vs_fixed_full}.}
  \label{tab:multi_vs_fixed_full}
  \centering
{\scriptsize
\begin{tabular}[t]{rll}
\toprule
Method & Empirical quantile $\hat{Q}$ & Remark\\
\midrule
\addlinespace[0.3em]
CFRNN & $\infty$ & $\hat{Q} = \infty$ for every time step.\\
NCTP & 0.5268 & \makecell[l]{$\hat{Q}$ is multiplied with the standard error at each time step \\ before adding or subtracting from the point prediction.} \\
CAFHT - multiplicative & 0.5883 & \makecell[l]{$\hat{Q}$ is multiplied with the width of the ACI bands at each  \\ time step before adding or subtracting from the ACI bands.} \\
CAFHT - additive & 0.2058 & \makecell[l]{$\hat{Q}$ is directly added or subtracted from the ACI bands.}  \\
\bottomrule
\end{tabular}}
\end{table}

\FloatBarrier

\subsection{Prediction Bands at Higher Coverage Levels}\label{app:more_experiments_higher_cov}

This section presents the results of additional experiments conducted using $\alpha = 0.05$ and $\alpha = 0.01$, seeking simultaneous coverage at the 95\% level and the 99\% level respectively. We continue to use the main implementation of the CAFHT method, which utilizes multiplicative scores based on the ACI algorithm and optimizes the learning rate through data splitting.

When higher coverage levels are employed, it is necessary to increase the number of samples in the calibration data to ensure that the adjusted empirical quantile level $(1-\alpha)(1-1/|\mathcal{D}_{\text{cal}}|)$ remains below 1. In our experiments, we cap the adjusted level at 1 whenever it exceeds this value.

\subsubsection*{Experiments with 95\% coverage level}

Figure~\ref{fig:main_exp_sim_ndata_alpha005} shows that all methods achieve 95\% simultaneous marginal coverage. Our method (CAFHT) leads to more informative bands with lower average width and higher conditional coverage.

\begin{figure}[!htb]
    \centering
    \includegraphics[width=\linewidth]{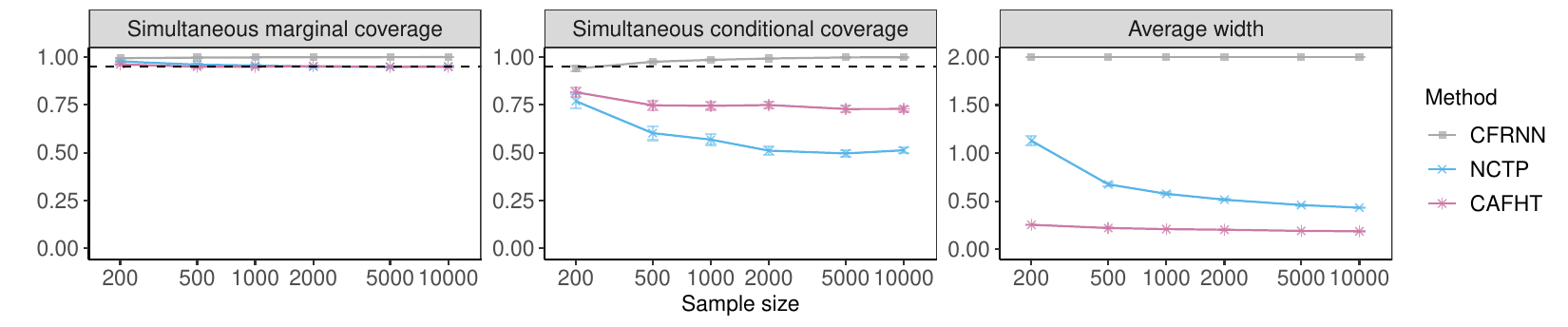}\vspace{-0.5cm}
    \caption{Performance on simulated heterogeneous trajectories of prediction bands constructed by different methods, as a function of the total number of training and calibration trajectories (25\% are randomly assigned to calibration set). The target simultaneous marginal coverage level is 95\%. See Table~\ref{tab:main_exp_sim_ndata_alpha005} for detailed results and standard errors.}
    \label{fig:main_exp_sim_ndata_alpha005}
\end{figure}

\begin{table}[!htb]
\centering
    \caption{Performance on simulated heterogeneous trajectories of prediction bands constructed by different methods, as a function of the total number of training and calibration trajectories. The red numbers indicate smaller prediction bands or higher conditional coverage. Target simultaneous marginal coverage level is 95\%. See corresponding plot in Figure~\ref{fig:main_exp_sim_ndata_alpha005}. }
  \label{tab:main_exp_sim_ndata_alpha005}
  \centering
{\scriptsize
\begin{tabular}[t]{rlllll}
\toprule
\multicolumn{3}{c}{ } & \multicolumn{3}{c}{Simultaneous coverage} \\
\cmidrule(l{3pt}r{3pt}){4-6}
Sample size & Method & Average width & Conditional-hard & Conditional-easy & Marginal\\
\midrule
\addlinespace[0.3em]
\multicolumn{6}{l}{\textbf{200}}\\
\hspace{1em}200 & CFRNN & 2.000 (0.000) & \textcolor{red}{0.940 (0.006)} & 1.000 (0.000) & 0.994 (0.001)\\
\hspace{1em}200 & NCTP & 1.129 (0.026) & 0.770 (0.018) & 1.000 (0.000) & 0.978 (0.002)\\
\hspace{1em}200 & CAFHT & \textcolor{red}{0.255 (0.003)} & \textcolor{red}{0.816 (0.013)} & 0.978 (0.003) & 0.962 (0.003)\\
\addlinespace[0.3em]
\multicolumn{6}{l}{\textbf{500}}\\
\hspace{1em}500 & CFRNN & 2.000 (0.000) & \textcolor{red}{0.976 (0.003)} & 1.000 (0.000) & 0.998 (0.000)\\
\bottomrule
\end{tabular}}
\end{table}

\FloatBarrier
\subsubsection*{Experiments with 99\% coverage level}

When seeking a 99\% coverage level, using a relatively small sample size will result in the adjusted level being very close to, or equal to, 1, mapping the empirical quantile $\hat{Q}$ to infinity. Consequently, as depicted in Figure~\ref{fig:main_exp_sim_ndata_alpha001}, NCTP and CAFHT generate regions that span the entire space $[-1,1]$ when the sample size is small. CAFHT requires slightly more calibration samples than NCTP to produce practically useful prediction regions when employing a data-splitting strategy. When the prediction bands are practically useful, CAFHT tends to produce narrower and thus more informative results compared to NCTP while maintaining similarly high conditional coverage.

\begin{figure}[!htb]
    \centering
    \includegraphics[width=\linewidth]{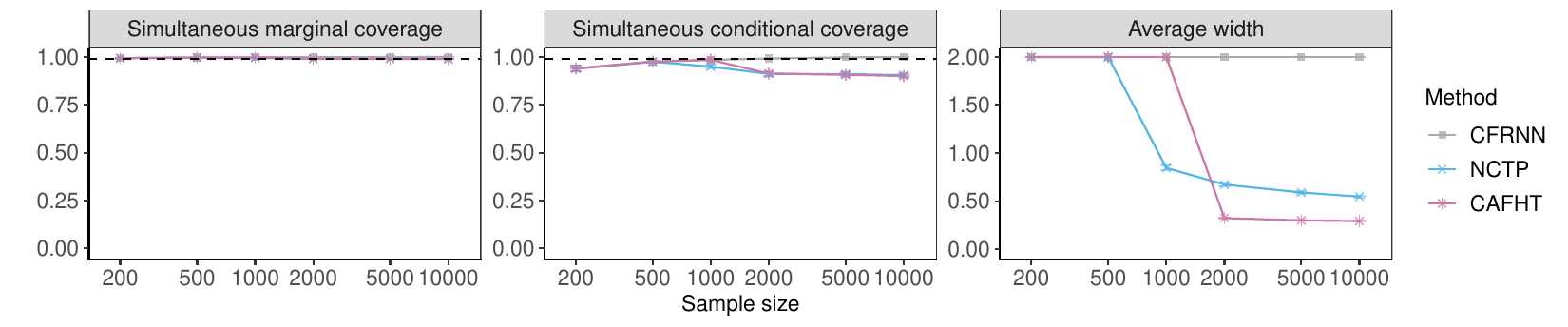}\vspace{-0.5cm}
    \caption{Performance on simulated heterogeneous trajectories of prediction bands constructed by different methods, as a function of the total number of training and calibration trajectories. Target simultaneous marginal coverage level is 99\%. See Table~\ref{tab:main_exp_sim_ndata_alpha001} for detailed results and standard errors.}
    \label{fig:main_exp_sim_ndata_alpha001}
\end{figure}

\begin{table}[!htb]
\centering
    \caption{Performance on simulated heterogeneous trajectories of prediction bands constructed by different methods, as a function of the total number of training and calibration trajectories. The red numbers indicate smaller prediction bands or higher conditional coverage. Target simultaneous marginal coverage level is 99\%. See corresponding plot in Figure~\ref{fig:main_exp_sim_ndata_alpha001}.}
  \label{tab:main_exp_sim_ndata_alpha001}
  \centering
{\scriptsize
\begin{tabular}[t]{rlllll}
\toprule
\multicolumn{3}{c}{ } & \multicolumn{3}{c}{Simultaneous coverage} \\
\cmidrule(l{3pt}r{3pt}){4-6}
Sample size & Method & Average width & Conditional-hard & Conditional-easy & Marginal\\
\midrule
\addlinespace[0.3em]
\multicolumn{6}{l}{\textbf{200}}\\
\hspace{1em}200 & CFRNN & 2.000 (0.000) & 0.940 (0.006) & 1.000 (0.000) & 0.994 (0.001)\\
\hspace{1em}200 & NCTP & 2.000 (0.000) & 0.940 (0.006) & 1.000 (0.000) & 0.994 (0.001)\\
\hspace{1em}200 & CAFHT & 2.000 (0.000) & 0.940 (0.006) & 1.000 (0.000) & 0.994 (0.001)\\
\addlinespace[0.3em]
\multicolumn{6}{l}{\textbf{500}}\\
\hspace{1em}500 & CFRNN & 2.000 (0.000) & 0.976 (0.003) & 1.000 (0.000) & 0.998 (0.000)\\
\bottomrule
\end{tabular}}
\end{table}

\FloatBarrier
\subsection{Comparisons with CopulaCPTS}\label{app:more_experiments_add_benchmark}

For completeness, this subsection presents empirical results that compare our CAFHT method with CopulaCPTS \citep{sun2023copula}, which uses the copula of prediction residuals across the entire horizon. Similar to NCTP, CopulaCPTS struggles with adaptability under heteroscedastic conditions and is thus expected to achieve conditional coverage akin to that of NCTP. We conducted these comparisons using synthetic AR data with dynamic profiles. The CopulaCPTS method is considered suitable only for situations with ample calibration data, as noted by \citet{sun2023copula}. Accordingly, we performed experiments with large datasets of 5,000 and 10,000 trajectories, designating 25\% randomly for calibration and the remainder for training. The findings were validated against an additional 100 independently generated test trajectories.

The results, displayed in Figures~\ref{fig:main_exp_sim_dynamic_horizon_copula}--\ref{fig:main_exp_sim_dynamic_delta_copula} and Tables~\ref{tab:main_exp_sim_dynamic_horizon_copula}--\ref{tab:main_exp_sim_dynamic_delta_copula}, confirm the anticipated outcomes. CopulaCPTS delivers results comparable to NCTP, while CAFHT surpasses the CopulaCPTS baseline by producing narrower prediction bands and achieving higher conditional coverage.

\begin{figure}[!htb]
    \centering
    \includegraphics[width=\linewidth]{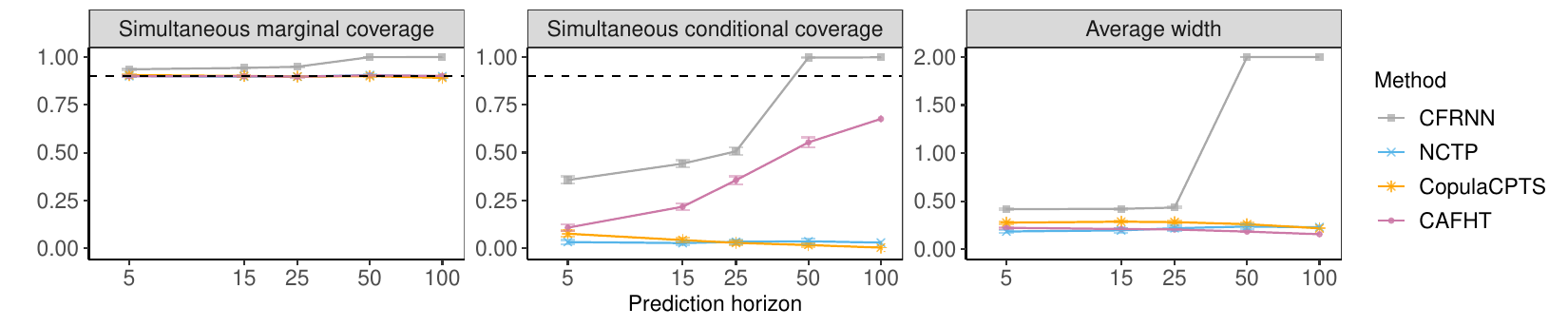}\vspace{-0.5cm}
    \caption{Performance on simulated heterogeneous trajectories of prediction bands constructed by different methods, as a function of the prediction horizon. See Table~\ref{tab:main_exp_sim_dynamic_horizon_copula} for detailed results and standard errors.}
    \label{fig:main_exp_sim_dynamic_horizon_copula}
\end{figure}

\begin{table}[!htb]
\centering
    \caption{Performance on simulated heterogeneous trajectories of prediction bands constructed by different methods, as a function of prediction horizon. The red numbers indicate smaller prediction bands or higher conditional coverage. See corresponding plot in Figure~\ref{fig:main_exp_sim_dynamic_horizon_copula}.}
  \label{tab:main_exp_sim_dynamic_horizon_copula}
  \centering
{\scriptsize
\begin{tabular}[t]{rlllll}
\toprule
\multicolumn{3}{c}{ } & \multicolumn{3}{c}{Simultaneous coverage} \\
\cmidrule(l{3pt}r{3pt}){4-6}
Prediction horizon & Method & Average width & Conditional-hard & Conditional-easy & Marginal\\
\midrule
\addlinespace[0.3em]
\multicolumn{6}{l}{\textbf{5}}\\
\hspace{1em}5 & CFRNN & 0.417 (0.006) & \textcolor{red}{0.357 (0.010)} & 1.000 (0.000) & 0.936 (0.001)\\
\hspace{1em}5 & NCTP & \textcolor{red}{0.185 (0.009)} & 0.033 (0.005) & 0.996 (0.001) & 0.901 (0.002)\\
\hspace{1em}5 & CopulaCPTS & 0.277 (0.006) & \textcolor{red}{0.076 (0.009)} & 0.997 (0.001) & 0.906 (0.002)\\
\hspace{1em}5 & CAFHT & 0.223 (0.004) & \textcolor{red}{0.108 (0.009)} & 0.987 (0.001) & 0.900 (0.002)\\
\addlinespace[0.3em]
\multicolumn{6}{l}{\textbf{15}}\\
\bottomrule
\end{tabular}}
\end{table}

\begin{figure}[!htb]
    \centering
    \includegraphics[width=\linewidth]{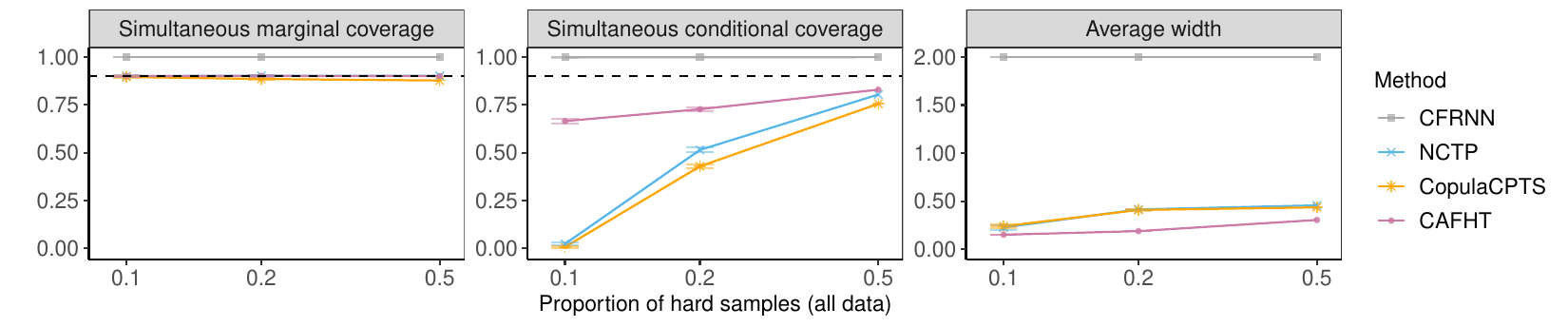}\vspace{-0.5cm}
    \caption{Performance on simulated heterogeneous trajectories of prediction bands constructed by different methods, as a function of the proportion of hard-to-predict trajectories. See Table~\ref{tab:main_exp_sim_dynamic_delta_copula} for detailed results and standard errors.}
    \label{fig:main_exp_sim_dynamic_delta_copula}
\end{figure}

\begin{table}[!htb]
\centering
    \caption{Performance on simulated heterogeneous trajectories of prediction bands constructed by different methods, as a function of the proportion of hard-to-predict trajectories. The red numbers indicate smaller prediction bands or higher conditional coverage. See corresponding plot in Figure~\ref{fig:main_exp_sim_dynamic_delta_copula}.}
  \label{tab:main_exp_sim_dynamic_delta_copula}
  \centering
{\scriptsize
\begin{tabular}[t]{rlllll}
\toprule
\multicolumn{3}{c}{ } & \multicolumn{3}{c}{Simultaneous coverage} \\
\cmidrule(l{3pt}r{3pt}){4-6}
Proportion of hard samples (all data) & Method & Average width & Conditional-hard & Conditional-easy & Marginal\\
\midrule
\addlinespace[0.3em]
\multicolumn{6}{l}{\textbf{0.1}}\\
\hspace{1em}0.1 & CFRNN & 2.000 (0.000) & \textcolor{red}{0.999 (0.000)} & 1.000 (0.000) & 1.000 (0.000)\\
\hspace{1em}0.1 & NCTP & 0.226 (0.011) & \textcolor{red}{0.023 (0.004)} & 0.999 (0.000) & 0.901 (0.001)\\
\hspace{1em}0.1 & CopulaCPTS & 0.239 (0.007) & 0.007 (0.002) & 0.995 (0.001) & 0.896 (0.002)\\
\hspace{1em}0.1 & CAFHT & \textcolor{red}{0.151 (0.001)} & \textcolor{red}{0.665 (0.007)} & 0.927 (0.001) & 0.900 (0.002)\\
\addlinespace[0.3em]
\multicolumn{6}{l}{\textbf{0.2}}\\
\hspace{1em}0.2 & CFRNN & 2.000 (0.000) & \textcolor{red}{0.999 (0.000)} & 1.000 (0.000) & 1.000 (0.000)\\
\hspace{1em}0.2 & NCTP & 0.415 (0.003) & \textcolor{red}{0.515 (0.006)} & 1.000 (0.000) & 0.903 (0.002)\\
\hspace{1em}0.2 & CopulaCPTS & 0.409 (0.003) & 0.429 (0.006) & 1.000 (0.000) & 0.885 (0.002)\\
\hspace{1em}0.2 & CAFHT & \textcolor{red}{0.189 (0.001)} & \textcolor{red}{0.727 (0.005)} & 0.945 (0.001) & 0.902 (0.002)\\
\addlinespace[0.3em]
\multicolumn{6}{l}{\textbf{0.5}}\\
\hspace{1em}0.5 & CFRNN & 2.000 (0.000) & \textcolor{red}{1.000 (0.000)} & 1.000 (0.000) & 1.000 (0.000)\\
\hspace{1em}0.5 & NCTP & 0.457 (0.003) & \textcolor{red}{0.804 (0.003)} & 1.000 (0.000) & 0.902 (0.002)\\
\hspace{1em}0.5 & CopulaCPTS & 0.436 (0.003) & 0.756 (0.003) & 1.000 (0.000) & 0.878 (0.002)\\
\hspace{1em}0.5 & CAFHT & \textcolor{red}{0.305 (0.002)} & \textcolor{red}{0.830 (0.003)} & 0.972 (0.001) & 0.901 (0.002)\\
\bottomrule
\end{tabular}}
\end{table}

\FloatBarrier


\section{Extension to Multi-Step Forecasting}\label{app:multi_step}

This section extends CAFHT to the multiple-step-ahead forecasting setting. Similar to section~\ref{sec:notations}, consider a data set containing $n$ observations of trajectories of length $T + 1$, namely $\mathcal{D}:= \{\bm{Y}^{(1)},\hdots,\bm{Y}^{(n)}\}$.
For $i \in [n] := \{1,\dots,n\}$, the array $\bm{Y}^{(i)} = (Y_0^{(i)}, \dots,Y_{T}^{(i)})$ represents $T+1$ observations of some $d$-dimensional vector $Y_{t}^{(i)} = (Y_{t,1}^{(i)}, \ldots, Y_{t,d}^{(i)})\in \mathbb{R}^d$, measured at distinct time steps $t \in \{0,\dots, T+1\}$. Let $g$ denote a trainable trajectory predictor that can make $H$-steps-ahead forecasts.

Consider a new trajectory $\bm{Y}^{(n+1)}$ sampled exchangeably with $\mathcal{D}$. Given the initial position $Y_0^{(n+1)}$, at every time $t$ for $t \in \{1, \dots, T\}$, the real value $Y_t^{(n+1)}$ is revealed, and we aim to construct prediction regions $(\hat{C}_{t}^1(\bm{Y}^{(n+1)}), \dots, \hat{C}_{t}^H(\bm{Y}^{(n+1)}))$ for $(Y_{t+1}^{(n+1)}, \dots, Y_{t+H}^{(n+1)})$ using the predictions $(\hat{Y}_{t+1}^{(n+1)}, \dots, \hat{Y}_{t+H}^{(n+1)})$ made by $\hat{g}$.

Let $\hat{C}_{t}^\tau(\bm{Y}^{(n+1)})$ represent the $\tau$-th-step-ahead prediction band for $Y^{(n+1)}_{t+\tau}$ output at time $t$ from the CAFHT method.
We aim to achieve the marginal simultaneous coverage, similar to Equation~\eqref{eq:simu_coverage}:
\begin{equation}\label{eq:simu_coverage_multi}
    \mathbb{P}\left[ Y_{t+\tau}^{(n+1)} \in \hat{C}_{t}^\tau(\bm{Y}^{(n+1)}), \; \forall t \in [T], \; \forall \tau \in[H] \right] \geq 1-\alpha.
\end{equation}

Similar to the one-step-ahead setting, we first initialize the adaptive prediction bands by extending the original ACI method to leverage the information of multi-step-ahead forecasting and to construct a multi-step-ahead prediction band. After that, we will calibrate the initialized adaptive bands and perform data-driven parameter selection.

\subsection{Multi-Step-Ahead ACI}

In this section, we explain how to extend the original one-step-ahead ACI to produce multi-steps-ahead prediction regions. 
Although this approach is intuitive, it may be possible to improve it in the future.

Consider a similar online setting as in \citet{gibbs2021adaptive}, where one observes covariate-response pairs $\{(X_t, Y_t)\}_{t\in \mathbb{N}} \subset \mathbb{R}^d \times \mathbb{R}$ in the sequential order. Denote the fitted model that can make $H$ steps ahead predictions as $\hat{g}$. At each time step $t$, assume that we observe pairs up until $\{(X_t, Y_t)\}$ and make $H$ steps ahead forecasts $(\hat{Y}_{t+1}, \dots, \hat{Y}_{t+H})$ for the future values $(Y_{t+1}, \dots, Y_{t+H})$ using $\hat{g}$. To construct the prediction regions for $(Y_{t+1}, \dots, Y_{t+H})$, consider running $H$ many ACI in parallel using the lagged nonconformity scores proposed by \citet{Dixit2023adaptive}.

First, to construct the prediction region for a single time step $Y_{t+\tau}$ in the future for any $\tau \in [H]$, we compute the lagged nonconformity score, defined as:
\begin{equation}\label{eq:lagged_nonconf_scores}
    S_t^\tau(X_t, y) = \| y - \hat{g}(X_t)\| = \| y - \hat{Y}^\tau_{t}\|.
\end{equation}
Intuitively, this measures the distance between $y$ and the prediction for $Y_{t+\tau}$ made at the current time. Then, the standard split conformal prediction approach to construct the prediction region for $Y_{t+\tau}$ at miscoverage level $\alpha$ would become $\hat{C}_t^\tau(\alpha) = \{y : S_t^\tau(X_t, y) \leq \hat{Q}(1-\alpha) \}$, where $\hat{Q}(1-\alpha) = \inf \{ s: (|\mathcal{D}_{\text{cal}}|^{-1} \sum_{(X_r, Y_r)\in \mathcal{D}_{\text{cal}}}\mathbbm{1}_{\{ S_{r-\tau}^\tau(X_{r-\tau}, Y_r)\leq s\} } ) \geq 1-\alpha \}$. To incorporate the core idea of ACI to continuously adapt the potential distribution changes within the time series, we run the following modified $\alpha$-update rule:
\begin{equation}\label{eq:alpha_multi}
    \alpha_{t+1}^\tau = \alpha^\tau_t + \gamma^\tau (\alpha - \text{err}_t^\tau),
\end{equation}
where
\begin{equation}\label{eq:err_multi}
    \text{err}_t^\tau= \begin{cases}
    1, & \text{ if } Y_t \notin \hat{C}^{\text{ACI},\tau}_{t-\tau}(\alpha_{t-1}^\tau), \\
    0, & \text{ otherwise}.
\end{cases}
\end{equation}
Above, $\gamma^\tau$ denotes the step size, which can be different for each $\tau$, and $\hat{C}^{\text{ACI}, \tau}_{t-\tau}(\alpha_{t-1}^\tau)$ is the prediction region constructed for $Y_t$ at $\tau$ steps ago as $\hat{C}^{\text{ACI}, \tau}_{t-\tau}(\alpha_{t-1}^\tau)= \{y : S_{t-\tau}^\tau(X_{t-\tau}, y) \leq \hat{Q}_{t-\tau}(1-\alpha_{t-1})\}$.
Equivalently,
$$\hat{C}^{\text{ACI},\tau}_{t-\tau}(\alpha_{t-1}) = [\hat{\ell}^{\text{ACI},\tau}_{t-\tau}, \hat{u}^{\text{ACI},\tau}_{t-\tau}] = [\hat{Y}_{t-\tau}^\tau - \hat{Q}_{t-\tau}(1-\alpha_{t-1}), \hat{Y}_{t-\tau}^\tau+\hat{Q}_{t-\tau}(1-\alpha_{t-1})].$$

The prediction region of $Y_{t+\tau}$ is then formed by:
\begin{equation}\label{eq:pred_region_multi}
    \hat{C}^{\text{ACI}, \tau}_{t}(\alpha_{t+1}^\tau) = [\hat{Y}^\tau_t - \hat{Q}_t(1-\alpha_{t+1}), \hat{Y}^\tau_t + \hat{Q}_t(1-\alpha_{t+1})].
\end{equation}

To construct multiple steps ahead prediction regions of $(Y_{t+1},\dots, Y_{t+H})$ at time $t$, we run the above procedure for every $\tau \in [H]$, and form the prediction regions $(\hat{C}^{\text{ACI}, 1}_{t}(\alpha_{t+1}^1), \dots, \hat{C}^{\text{ACI}, H}_{t}(\alpha_{t+1}^H))$; see Algorithm~\ref{alg:ACI_multi}.

\begin{algorithm}[!htb]
    \caption{Multi-step-ahead ACI}
    \label{alg:ACI_multi}
    \begin{algorithmic} [1]
        \STATE \textbf{Input}: A pre-trained forecaster $\hat{g}$ producing $H$-step-ahead predictions;
        current time $t$;
        time trajectory with observed past values $(Y_{1}, \dots, Y_{t-1})$.
        \STATE Observe the true value at current time $Y_t$.
        \STATE Make $H$-step-ahead predictions $(\hat{Y}_{t+1}, \dots, \hat{Y}_{t+H})$ for $(Y_{t+1},\dots, Y_{t+H})$.
        \FOR{$\tau \in [H]$}
        \STATE Evaluate $\text{err}_t^\tau$ using Equation~\eqref{eq:err_multi}.
        \STATE Update $\alpha^\tau_{t+1}$ using Equation~\eqref{eq:alpha_multi}.
        \STATE Construct prediction region $\hat{C}^{\text{ACI}, \tau}_{t}(\alpha_{t+1}^\tau)$ for $Y_{t+\tau}$ using Equation~\eqref{eq:pred_region_multi}.
        \ENDFOR
        \STATE \textbf{Output}: Online multi-steps-ahead prediction regions $(\hat{C}^{\text{ACI}, 1}_{t}(\alpha_{t+1}^1), \dots, \hat{C}^{\text{ACI}, H}_{t}(\alpha_{t+1}^H))$.
\end{algorithmic}
\end{algorithm}

\subsection{Calibrating the Adaptive Prediction Bands}

In the previous section, we discussed how to form multiple steps ahead prediction regions using ACI at every time $t$. We now proceed to calibrate these regions to achieve simultaneous coverage guarantee~\eqref{eq:simu_coverage_multi}. For simplicity, we start by taking the learning rate $\gamma^\tau$ as fixed and constant for all $\tau \in [H]$.

Different from the one-step-ahead setting, with multi-step-ahead ACI, at every time $t$ we can construct $H$ prediction regions for the following $H$ values. As we move on to observe the next trajectory value, we can update the future prediction regions using the more recent information. In fact, at every $t$, we will have $H-1$ different prediction regions, separately constructed from $H-1, H-2, \dots, 1$ steps ago, denoted as $\hat{C}^{\text{ACI}, H}_{t-H}, \dots, \hat{C}^{\text{ACI}, 1}_{t-1}$. To perform calibration, we need to summarize the information obtained from those into a single region, which we will explain next.

For any $\tau \in [H]$, let $\hat{C}^{\text{ACI}, \tau}_{t-\tau}(\bm{Y}^{(i)}, \gamma) = [\hat{\ell}^{\text{ACI}, \tau}_{t-\tau}(\bm{Y}^{(i)}, \gamma), \hat{u}^{\text{ACI}, \tau}_{t-\tau}(\bm{Y}^{(i)}, \gamma)]$ denote the prediction band for $Y_{t}$ constructed at $\tau$ steps ago with learning rate $\gamma$. For each calibration trajectory $i \in \mathcal{D}_{\text{cal}}$, CAFHT evaluates the nonconformity score $\hat{\epsilon}_i(\gamma)$ using the following equation:
\begin{align}\label{eq:nonconf_score_msteps_fixed}
    \begin{split}
      \hat{\epsilon}_i (\gamma) := \max_{t \in \{1,\dots, T\}} \Biggr\{ \max \Biggr\{
      & \left[  \max_{\tau \in [H]} \left\{ \hat{\ell}^{\text{ACI},\tau}_{t-\tau}(\bm{Y}^{(i)}, \gamma) \right\} - Y_t^{(i)}\right]_{+}  ,  \left[  Y_t^{(i)} -  \min_{\tau \in [H]} \left\{  \hat{u}^{\text{ACI},\tau}_{t-\tau}(\bm{Y}^{(i)}, \gamma) \right\} \right]_{+} \Biggr\} \Biggr\},
    \end{split}
\end{align}
Intuitively, $\hat{\epsilon}_i (\gamma)$ measures the maximum absolute distance of $Y_t$ from the prediction regions constructed at different historical time steps $\hat{C}^{\text{ACI}, H}_{t-H}, \dots, \hat{C}^{\text{ACI}, 1}_{t-1}$. 

The remaining components of our method then follow the same logic as the one-step-ahead CAFHT. Let $\hat{Q}(1-\alpha, \gamma)$ denote the $\lceil (1-\alpha)(1+|\mathcal{D}_{\text{cal}}|) \rceil$-th smallest value of $\hat{\epsilon}_i (\gamma)$ among $i \in \mathcal{D}_{\text{cal}}$.
At every time step $t\in[T]$, CAFHT constructs $H$-steps ahead prediction bands $\hat{C}^\tau_t(\bm{Y}^{(n+1)},\gamma),\forall\tau \in [H]$ using the following equation:
\begin{equation}\label{eq:predict_bands_msteps_fixed}
\begin{split}
    \hat{C}_t^\tau(\bm{Y}^{(n+1)} ,\gamma)
        & = \bigg[
        \hat{\ell}^{\text{ACI},\tau}_{t}(\bm{Y}^{(n+1)}, \gamma)  - \hat{Q}(1-\alpha, \gamma),
        \hat{u}^{\text{ACI},\tau}_{t}(\bm{Y}^{(n+1)}, \gamma) + \hat{Q}(1-\alpha, \gamma) \bigg].
\end{split}
\end{equation}
The next result establishes finite-sample simultaneous coverage guarantees for this method.

\begin{theorem} \label{theorem:coverage_msteps}
Assume that the calibration trajectories in $\mathcal{D}_{\text{cal}}$ are exchangeable with $\bm{Y}^{(n+1)}$.
Then, for any $\alpha \in (0,1)$, the prediction band output by the multi-step-ahead CAFHT, applied with fixed parameters $\alpha$, $\alpha_{\mathrm{ACI}}$, and $\gamma$, satisfies~\eqref{eq:simu_coverage_multi}.
\end{theorem}
\begin{proof}
The proof is very similar to the proof of Theorem~\ref{theorem:coverage}, and it follows directly from the exchangeability of the conformity scores.
Denote $\hat{\epsilon}_{n+1}(\gamma)$ the conformity score of the test trajectory $\bm{Y}^{(t+1)}$ evaluated using Equation~\eqref{eq:nonconf_score_msteps_fixed}. For any fixed $\alpha$ and $\gamma>0$, we have that  $Y_{t+\tau}^{(n+1)} \in \hat{C}^{\tau}_{t}(\bm{Y}^{(n+1)},\gamma) \; \forall \tau \in [H] \; \forall t \in [T]$ if and only if $\hat{\epsilon}_{n+1}(\gamma) \leq \hat{Q}(1-\alpha, \gamma)$, where $\hat{Q}(1-\alpha, \gamma)$ is the $\lceil (1-\alpha)(1+|\mathcal{D}_{\text{cal}}|)\rceil$-th smallest value of $\hat{\epsilon}_i (\gamma)$ for all $i \in \mathcal{D}_{\text{cal}}$. Since the test trajectory is exchangeable with $\mathcal{D}_{\text{cal}}$, its score $\hat{\epsilon}_{n+1}(\gamma)$ is also exchangeable with $\{\hat{\epsilon}_{i}(\gamma), i\in\mathcal{D}_{\text{cal}}\}$. Then by Lemma 1 in \citet{romano2019conformalized}, it follows that $\mathbb{P}(Y_{t+\tau}^{(n+1)} \in \hat{C}^{\tau}_{t}(\bm{Y}^{(n+1)},\gamma) \; \forall \tau \in [H] \; \forall t \in [T] ) = \mathbb{P}(  \hat{\epsilon}_{n+1}(\gamma) \leq \hat{Q}(1-\alpha, \gamma)  )\geq 1-\alpha$.
\end{proof}

\subsection{Data-Driven Parameter Selection}
Similar to the one-step-ahead CAFHT, we can choose the step size parameter $\gamma$ in a data-driven way. For simplicity, we start by selecting among a grid of candidate $\{ \gamma_1, \dots, \gamma_L\}$, but assuming that the step size stays the same for every time step $\tau \in [H]$. Later in the experiments, we discuss using alternative options, such as setting $\gamma$ decaying as $\tau$ increases, which is more intuitive in practice as predictions made longer steps ahead are usually less reliable than the predictions made more recently.

\begin{algorithm}[!htb]
    \caption{Model selection component of multi-steps-ahead CAFHT}
    \label{alg:fixed_msteps_CAFHT_ds-model-selection}
    \begin{algorithmic} [1]
        \STATE \textbf{Input}: A pre-trained forecaster $\hat{g}$ producing H-step-ahead predictions;
        calibration trajectories $\mathcal{D}_{\text{cal}}^1$;
        a grid of candidate learning rates $\{\gamma_1, \dots, \gamma_L\}$.
        \FOR{$\ell \in [L]$}
            \STATE Construct $\hat{C}^{\text{ACI},\tau}_t(\bm{Y}^{(i)}, \gamma_\ell) \;\forall t \in [T], \forall \tau \in [H]$ using Algorithm~\ref{alg:ACI_multi}, for $i \in \mathcal{D}_{\text{cal}}^1$.
            \STATE Evaluate $\hat{\epsilon}_i (\gamma_\ell)$ using~\eqref{eq:nonconf_score_msteps_fixed}, for $i \in \mathcal{D}_{\text{cal}}^1$.
            \STATE Compute $\hat{Q}(1 - \alpha, \gamma_\ell)$, the $(1-\alpha)(1+1/|\mathcal{D}_{\text{cal}}^1|)$-th quantile of $\{ \hat{\epsilon}_i (\gamma_\ell), i \in \mathcal{D}_{\text{cal}}^1\}$.
            \STATE Construct $\hat{C}^\tau_t(\bm{Y}^{(i)}, \gamma_\ell) \; \forall \tau \in [H] \forall t \in [H]$ using~\eqref{eq:predict_bands_msteps_fixed} for $i \in \mathcal{D}_{\text{cal}}^1$.
        \ENDFOR
        \STATE Pick $\hat{\gamma}$ such that,
        \begin{equation}
            \hat{\gamma} := \argmin_{\ell \in [L]} \text{AvgWidth}(\{ C^\tau_t(\bm{Y}^{(i)}, \gamma_\ell)\}_{t\in [T], \tau \in [H]} ).
        \end{equation}
        \STATE \textbf{Output}: Selected learning rate parameter $\hat{\gamma}$.
\end{algorithmic}
\end{algorithm}

\begin{algorithm}[!htb]
    \caption{Multi-step-ahead CAFHT}
    \label{alg:fixed_msteps_CAFHT_ds}
    \begin{algorithmic} [1]
        \STATE \textbf{Input}: A pre-trained forecaster $\hat{g}$ producing multi-step-ahead predictions;
        calibration trajectories $\mathcal{D}_{\text{cal}}$; the initial position $Y_0^{(n+1)}$ of a test trajectory $\bm{Y}^{(n+1)}$;
        the desired nominal level $\alpha \in (0,1)$;
        a grid of candidate learning rates $\{\gamma_1, \dots, \gamma_L\}$.
        \STATE Randomly split $\mathcal{D}_{\text{cal}}$ into $\mathcal{D}_{\text{cal}}^1$ and $\mathcal{D}_{\text{cal}}^2$.
        \STATE Select a learning rate $\hat{\gamma} \in \{\gamma_1, \ldots, \gamma_L\}$, applying Algorithm~\ref{alg:fixed_CAFHT_ds-model-selection} using the trajectory data in $\mathcal{D}_{\text{cal}}^1$.
        \STATE Construct $\hat{C}^{\text{ACI}}(\bm{Y}^{(i)}, \hat{\gamma})$ using ACI, for $i \in \mathcal{D}_{\text{cal}}^2$.
        \STATE Evaluate $\hat{\epsilon}_i(\hat{\gamma})$ using~\eqref{eq:nonconf_score_msteps_fixed}, for $i \in \mathcal{D}_{\text{cal}}^2$.
        \STATE Compute the empirical quantile $\hat{Q}(1-\alpha, \hat{\gamma})$.
        \FOR{$t \in [T]$}
        \STATE Observe the current step $Y_t^{(n+1)}$.
        \STATE Compute $\hat{C}^{\text{ACI},\tau}_{t}(\bm{Y}^{(n+1)}, \hat{\gamma}) \; \forall \tau \in [H]$ with the multi-step-ahead ACI stated in Algorithm~\ref{alg:ACI_multi}, using the past of the test trajectory $(Y_{1}^{(n+1)},\ldots,Y_{t}^{(n+1)})$.
        \STATE Compute prediction bands $\hat{C}_{t}^\tau (\bm{Y}^{(n+1)}, \hat{\gamma} ), \forall \tau \in [H]$ for the next $H$ steps, using~\eqref{eq:predict_bands_msteps_fixed}.
        \ENDFOR
        \STATE \textbf{Output}: Online prediction bands $\hat{C}(\bm{Y}^{(n+1)})$.
\end{algorithmic}
\end{algorithm}

\subsection{Multi-step-ahead CAFHT using Multiplicative Scores}
Similar to the one-step-ahead cases, we can utilize a multiplicative score for the multi-step-ahead settings. This can be simply accomplished by replacing the nonconformity scores defined in~\eqref{eq:nonconf_score_msteps_fixed} with these:
\begin{align}\label{eq:nonconf_score_msteps_multi}
    \begin{split}
      \tilde{\epsilon}_i (\gamma) := \max_{t \in \{1,\dots, T\}} \Biggr\{ \max \Biggr\{
      & \max_{\tau \in [H]} \left\{ \frac{\left[  \hat{\ell}^{\text{ACI},\tau}_{t-\tau}(\bm{Y}^{(i)}, \gamma) - Y_t^{(i)}\right]_{+} }{| \hat{C}_{t-\tau}^{\text{ACI},\tau}(\bm{Y}^{(i)}, \gamma)|} \right \}  ,  \max_{\tau \in [H]} \left\{ \frac{\left[ Y_t^{(i)}- \hat{u}^{\text{ACI},\tau}_{t-\tau}(\bm{Y}^{(i)}, \gamma)\right]_{+} }{| \hat{C}_{t-\tau}^{\text{ACI},\tau}(\bm{Y}^{(i)}, \gamma)|} \right \}  \Biggr\} \Biggr\},
    \end{split}
\end{align}
and the counterpart of Equation~\eqref{eq:predict_bands_msteps_fixed} becomes
\begin{equation}\label{eq:predict_bands_msteps_multi}
\begin{split}
    \tilde{C}_t^\tau(\bm{Y}^{(n+1)} ,\gamma)
         = \bigg[
        & \hat{\ell}^{\text{ACI},\tau}_{t}(\bm{Y}^{(n+1)}, \gamma)  - \hat{Q}(1-\alpha, \gamma) \cdot | \hat{C}_{t}^{\text{ACI},\tau}(\bm{Y}^{(i)}, \gamma)|, \\
        & \hat{u}^{\text{ACI},\tau}_{t}(\bm{Y}^{(n+1)}, \gamma) + \hat{Q}(1-\alpha, \gamma) \cdot | \hat{C}_{t}^{\text{ACI},\tau}(\bm{Y}^{(i)}, \gamma)| \bigg].
\end{split}
\end{equation}

\subsection{Numerical Experiments}
We utilize the same synthetic settings as in Section~\ref{sec:experiment}, but modify the LSTM models so that they can make multiple steps ahead of predictions. Again, we choose the ACI-based multiplicative scores as the main CAFHT method.

\begin{figure*}[!t]
    \centering
    \includegraphics[width=\linewidth]{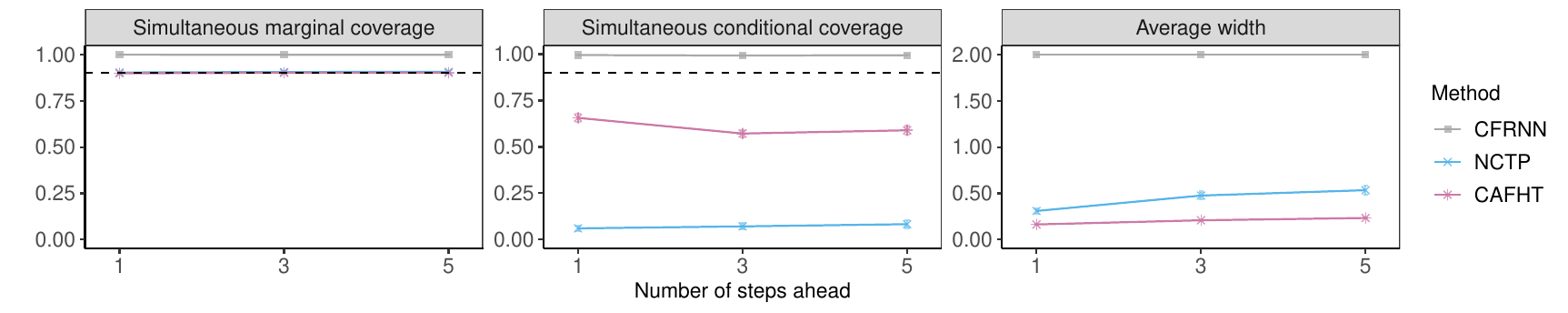}\vspace{-0.5cm}
    \caption{Performance on simulated heterogeneous trajectories of prediction bands constructed by different methods, as a function of the steps-ahead parameter $H$ utilized by the forecaster. Other details are as in Table~\ref{tab:main_exp_sim_msteps_ahead}.}
    \label{fig:main_exp_sim_msteps_ahead}
\end{figure*}

Figure~\ref{fig:main_exp_sim_msteps_ahead} summarizes the performance of the three methods as a function of the number of steps ahead predictions made by the forecaster, which is varied from $1$ to $5$. When number of steps is equal to $1$, we recover the one-step-ahead CAFHT results. In each case, $75\%$ of the trajectories are used for training and the remaining $25\%$ for calibration. Our method utilizes 50\% of the calibration trajectories to select the ACI learning rate $\gamma$. The results are averaged over 500 test trajectories and 100 independent experiments.

As we can see, all methods attain $90\%$ simultaneous coverage as defined in~\eqref{eq:simu_coverage_multi}. However, our method yields the most efficient results in terms of obtaining the smallest size of the prediction band and higher conditional coverage than the NCTP benchmark. See Table~\ref{tab:main_exp_sim_msteps_ahead} for standard errors.

\begin{table}[!htb]
\centering
    \caption{Performance on simulated heterogeneous trajectories of prediction bands constructed by different methods, as a function of the total number of training and calibration trajectories. The red numbers indicate smaller prediction bands or higher conditional coverage. See the corresponding plot in Figure~\ref{fig:main_exp_sim_msteps_ahead}.}
  \label{tab:main_exp_sim_msteps_ahead}
\centering
{\scriptsize
\begin{tabular}[t]{rlllll}
\toprule
\multicolumn{3}{c}{ } & \multicolumn{3}{c}{Simultaneous coverage} \\
\cmidrule(l{3pt}r{3pt}){4-6}
Number of steps ahead & Method & Average width & Conditional-hard & Conditional-easy & Marginal\\
\midrule
\addlinespace[0.3em]
\multicolumn{6}{l}{\textbf{1}}\\
\hspace{1em}1 & CFRNN & 2.000 (0.000) & \textcolor{red}{0.995 (0.001)} & 1.000 (0.000) & 1.000 (0.000)\\
\hspace{1em}1 & NCTP & 0.308 (0.014) & 0.060 (0.008) & 0.996 (0.001) & 0.903 (0.002)\\
\hspace{1em}1 & CAFHT & \textcolor{red}{0.163 (0.001)} & \textcolor{red}{0.656 (0.010)} & 0.926 (0.002) & 0.899 (0.003)\\
\addlinespace[0.3em]
\multicolumn{6}{l}{\textbf{3}}\\
\hspace{1em}3 & CFRNN & 2.000 (0.000) & \textcolor{red}{0.993 (0.001)} & 1.000 (0.000) & 0.999 (0.000)\\
\hspace{1em}3 & NCTP & 0.475 (0.022) & 0.071 (0.008) & 0.996 (0.001) & 0.906 (0.002)\\
\hspace{1em}3 & CAFHT & \textcolor{red}{0.208 (0.002)} & \textcolor{red}{0.572 (0.010)} & 0.938 (0.002) & 0.902 (0.002)\\
\addlinespace[0.3em]
\multicolumn{6}{l}{\textbf{5}}\\
\hspace{1em}5 & CFRNN & 2.000 (0.000) & \textcolor{red}{0.995 (0.001)} & 1.000 (0.000) & 0.999 (0.000)\\
\hspace{1em}5 & NCTP & 0.534 (0.025) & 0.083 (0.010) & 0.997 (0.001) & 0.907 (0.002)\\
\hspace{1em}5 & CAFHT & \textcolor{red}{0.233 (0.002)} & \textcolor{red}{0.589 (0.012)} & 0.936 (0.002) & 0.902 (0.002)\\
\bottomrule
\end{tabular}}
\end{table}

In another experiment, the steps-ahead parameter is fixed as $H=3$, and the total number of trajectories in the training and calibration sets are varied from 200 to 2000. Again, our method yields the most informative bands.

\begin{figure*}[!t]
    \centering
    \includegraphics[width=\linewidth]{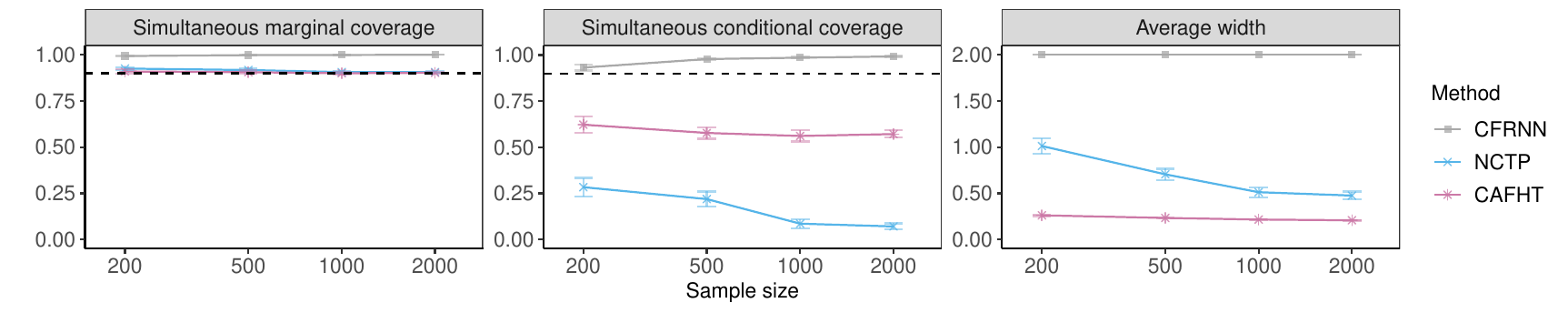}\vspace{-0.5cm}
    \caption{Performance on simulated heterogeneous trajectories of prediction bands constructed by different methods, as a function of the number of trajectories in the training and calibration sets, made by the 3-steps-ahead forecaster. Other details are as in Table~\ref{tab:main_exp_sim_msteps_ndata}.}
    \label{fig:main_exp_sim_msteps_ndata}
\end{figure*}

\begin{table}[!htb]
\centering
    \caption{Performance on simulated heterogeneous trajectories of prediction bands constructed by different methods, as a function of the total number of training and calibration trajectories. The red numbers indicate smaller prediction bands or higher conditional coverage. See the corresponding plot in Figure~\ref{fig:main_exp_sim_msteps_ndata}.}
  \label{tab:main_exp_sim_msteps_ndata}
\centering
{\scriptsize 
\begin{tabular}[t]{rlllll}
\toprule
\multicolumn{3}{c}{ } & \multicolumn{3}{c}{Simultaneous coverage} \\
\cmidrule(l{3pt}r{3pt}){4-6}
Sample size & Method & Average width & Conditional-hard & Conditional-easy & Marginal\\
\midrule
\addlinespace[0.3em]
\multicolumn{6}{l}{\textbf{200}}\\
\hspace{1em}200 & CFRNN & 2.000 (0.000) & \textcolor{red}{0.932 (0.008)} & 1.000 (0.000) & 0.993 (0.001)\\
\hspace{1em}200 & NCTP & 1.012 (0.040) & 0.284 (0.025) & 0.997 (0.001) & 0.925 (0.003)\\
\hspace{1em}200 & CAFHT & \textcolor{red}{0.261 (0.004)} & \textcolor{red}{0.623 (0.022)} & 0.942 (0.006) & 0.909 (0.007)\\
\addlinespace[0.3em]
\multicolumn{6}{l}{\textbf{500}}\\
\hspace{1em}500 & CFRNN & 2.000 (0.000) & \textcolor{red}{0.979 (0.003)} & 1.000 (0.000) & 0.998 (0.000)\\
\hspace{1em}500 & NCTP & 0.705 (0.030) & 0.219 (0.021) & 0.995 (0.002) & 0.917 (0.003)\\
\hspace{1em}500 & CAFHT & \textcolor{red}{0.233 (0.002)} & \textcolor{red}{0.578 (0.016)} & 0.942 (0.003) & 0.906 (0.004)\\
\addlinespace[0.3em]
\multicolumn{6}{l}{\textbf{1000}}\\
\hspace{1em}1000 & CFRNN & 2.000 (0.000) & \textcolor{red}{0.986 (0.002)} & 1.000 (0.000) & 0.999 (0.000)\\
\hspace{1em}1000 & NCTP & 0.512 (0.027) & 0.086 (0.012) & 0.995 (0.001) & 0.906 (0.002)\\
\hspace{1em}1000 & CAFHT & \textcolor{red}{0.216 (0.002)} & \textcolor{red}{0.561 (0.015)} & 0.937 (0.003) & 0.900 (0.003)\\
\addlinespace[0.3em]
\multicolumn{6}{l}{\textbf{2000}}\\
\hspace{1em}2000 & CFRNN & 2.000 (0.000) & \textcolor{red}{0.993 (0.001)} & 1.000 (0.000) & 0.999 (0.000)\\
\hspace{1em}2000 & NCTP & 0.475 (0.022) & 0.071 (0.008) & 0.996 (0.001) & 0.906 (0.002)\\
\hspace{1em}2000 & CAFHT & \textcolor{red}{0.208 (0.002)} & \textcolor{red}{0.572 (0.010)} & 0.938 (0.002) & 0.902 (0.002)\\
\bottomrule
\end{tabular}}
\end{table}






\end{document}